%% file: main.tex
\documentclass{article}
\pdfoutput=1

\usepackage[accepted]{icml2023}

\input{header.tex}

\input{symbols.tex}

\icmltitlerunning{Target-based Surrogates for Stochastic Optimization}

\begin{document}

\twocolumn[
\icmltitle{Target-based Surrogates for Stochastic Optimization}

\icmlsetsymbol{equal}{*}

\begin{icmlauthorlist}
\icmlauthor{Jonathan Wilder Lavington}{equal,ubc}
\icmlauthor{Sharan Vaswani}{equal,sfu}
\icmlauthor{Reza Babanezhad}{samsung}
\icmlauthor{Mark Schmidt}{ubc,cifaramii}
\icmlauthor{Nicolas Le Roux}{msr,cifarmila}

\end{icmlauthorlist}

\icmlaffiliation{ubc}{University of British Columbia}
\icmlaffiliation{sfu}{Simon Fraser University}
\icmlaffiliation{samsung}{Samsung - SAIT AI Lab, Montreal}
\icmlaffiliation{msr}{Microsoft Research}
\icmlaffiliation{cifaramii}{ Canada CIFAR AI Chair (Amii)}
\icmlaffiliation{cifarmila}{ Canada CIFAR AI Chair (MILA)}

\icmlcorrespondingauthor{Jonathan Wilder Lavington}{wilderlavington@gmail.com}
\icmlcorrespondingauthor{Sharan Vaswani}{vaswani.sharan@gmail.com}

\icmlkeywords{Stochastic optimization, Imitation learning, Surrogate optimization, Composition structure}

\vskip 0.3in
]

\printAffiliationsAndNotice{\icmlEqualContribution} 

\input{Sections/Abstract}
\input{Sections/Introduction}
\input{Sections/Problem}
\input{Sections/Deterministic}
\input{Sections/Stochastic}
\input{Sections/Experiments}
\input{Sections/Discussion}

\input{Acknowledgements}
\bibliographystyle{apalike}
\bibliography{ref}

\clearpage
\appendix 
\onecolumn 

\input{Appendix/Setup}

\input{Appendix/Algorithms}

\input{Appendix/Proofs}
\input{Appendix/Experiments}

\end{document}

%% file: header.tex
\usepackage[utf8]{inputenc} 
\usepackage[T1]{fontenc}    
\usepackage{url}            
\usepackage{booktabs}       
\usepackage{amsfonts}       
\usepackage{nicefrac}       
\usepackage{microtype}      
\usepackage{wrapfig}
\usepackage{subcaption}

\usepackage{amssymb,amsmath}
\usepackage{amsthm}
\usepackage{mathrsfs}
\usepackage{thmtools, thm-restate}
\usepackage[mathscr]{eucal} 
\usepackage{natbib}
\usepackage{bbm}
\usepackage{arydshln}
\usepackage{enumitem}
\usepackage{cancel}
\usepackage{graphicx}
\usepackage{mdframed}
\usepackage{pgfplots}

\theoremstyle{plain}
\newtheorem{theorem}{Theorem}[section]

\newtheorem{lemma}[theorem]{Lemma}

\theoremstyle{definition}

\theoremstyle{remark}

\pgfplotsset{compat=1.18}
\definecolor{mydarkblue}{rgb}{0,0.08,0.45}
\definecolor{mydarkgreen}{rgb}{0,0.45,0.08}

\usepackage[colorlinks=true,
    linkcolor=mydarkblue,
    citecolor=mydarkblue,
    filecolor=mydarkblue,
    urlcolor=mydarkblue,
    pagebackref=true]{hyperref} 
\usepackage[capitalize]{cleveref}
\renewcommand*\backref[1]{\ifx#1\relax \else (cited on #1) \fi}

\newcommand{\blue}[1]{\textcolor{mydarkblue}{#1}}
\newcommand{\darkgreen}[1]{\textcolor{mydarkgreen}{#1}}

\mdfdefinestyle{mdframedthmbox}{%
	leftmargin=.0\textwidth,
	rightmargin=.0\textwidth,%
	innertopmargin=0.75em,
	innerleftmargin=.5em,
	innerrightmargin=.5em,
}

\usepackage{float}

\crefname{lemmanew}{Lemma}{Lemmas}

%% file: symbols.tex
\newcommand{\grad}[1]{\nabla \ell(#1)}
\newcommand{\gradh}[1]{\nabla h(#1)}

\newcommand{\inner}[2]{\langle #1,\, #2 \rangle}
\newcommand{\normsq}[1]{\left\|#1 \right\|_{2}^{2}}
\newcommand{\norm}[1]{\left\|#1 \right\|}

\newcommand{\indnormsq}[2]{\left\|#1\right\|_{#2}^{2}}

\newcommand{\tht}{\theta_t}
\newcommand{\thtt}{\theta_{t+1}}
\newcommand{\bthtt}{\bar{\theta}_{t+1}}
\newcommand{\tthtt}{\tilde{\theta}_{t+1}}

\newcommand{\pthtt}{\theta^{\prime}_{t+1}}

\newcommand{\z}{z}
\newcommand{\zt}{z_t}

\newcommand{\ztt}{z_{t+1}}
\newcommand{\tztt}{\tilde{z}_{t+1}}
\newcommand{\cZ}{\mathcal{Z}}
\newcommand{\cP}{\mathcal{P}}

\newcommand{\fopt}{{z^*}}
\newcommand{\ft}{{z_{t}}}

\newcommand{\ftt}{{z_{t+1}}}
\newcommand{\tftt}{\tilde{z}_{t+1}}
\newcommand{\bftt}{{\bar{z}_{t+1}}}
\newcommand{\fth}{{z_{t + \nicefrac{1}{2}}}}

\def\1{\bm{1}}

\newcommand{\transpose}{^\mathsf{\scriptscriptstyle T}}
\newcommand\inv{\raisebox{1.15ex}{$\scriptscriptstyle-\!1$}}

\def\epst{{\epsilon_t}}
\def\epstsq{{\epsilon^2_{t}}}
\def\epstt{{\epsilon_{t+1}}}
\def\epsttsq{{\epsilon^2_{t+1}}}
\def\alphat{{\alpha_t}}
\def\alphatsq{{\alpha^{2t}}}

\newcommand{\g}{\textsl{g}^{p}}
\newcommand{\gz}{g}

\newcommand{\gradt}[1]{{\nabla \ell_t(#1)}}

\newcommand{\etat}{{\eta_t}}

\newcommand{\petat}{{\eta'_t}}
\newcommand{\petatsq}{{\eta'^{2}_t}}

\DeclareMathAlphabet{\mathsfit}{\encodingdefault}{\sfdefault}{m}{sl}
\SetMathAlphabet{\mathsfit}{bold}{\encodingdefault}{\sfdefault}{bx}{n}

\newcommand{\E}{\mathbb{E}}

\newcommand{\R}{\mathbb{R}}

\DeclareMathOperator*{\argmax}{arg\,max}
\DeclareMathOperator*{\argmin}{arg\,min}

\def\FMAPG/{{FMA-PG}}
\def\DFMAPG/{{DFMA-PG}}
\def\SFMAPG/{{SFMA-PG}}

\def\SSO/{\texttt{SSO}}

%% file: Sections/Abstract.tex
\begin{abstract}
We consider minimizing functions for which it is expensive to compute the (possibly stochastic) gradient. Such functions are prevalent in reinforcement learning,  imitation learning and adversarial training. Our target optimization framework uses the (expensive) gradient computation to construct surrogate functions in a \emph{target space} (e.g. the logits output by a linear model for classification) that can be minimized efficiently. This allows for multiple parameter updates to the model, amortizing the cost of gradient computation. In the full-batch setting, we prove that our surrogate is a global upper-bound on the loss, and can be (locally) minimized using a black-box optimization algorithm. We prove that the resulting majorization-minimization algorithm ensures convergence to a stationary point of the loss. Next, we instantiate our framework in the stochastic setting and propose the $\SSO/$ algorithm, which can be viewed as projected stochastic gradient descent in the target space. This connection enables us to prove theoretical guarantees for $\SSO/$ when minimizing convex functions. Our framework allows the use of standard stochastic optimization algorithms to construct surrogates which can be minimized by any deterministic optimization method. To evaluate our framework, we consider a suite of supervised learning and imitation learning problems. Our experiments indicate the benefits of target optimization and the effectiveness of $\SSO/$.
\end{abstract}

%% file: Sections/Introduction.tex
\section{Introduction}
\label{sec:introduction} 

Stochastic gradient descent (SGD)~\citep{robbins1951stochastic} and its variants~\citep{duchi2011adaptive,kingma2014adam} are ubiquitous optimization methods in machine learning (ML). For supervised learning, iterative first-order methods require computing the gradient over individual mini-batches of examples, and that cost of computing the gradient often dominates the total computational cost of these algorithms. For example, in reinforcement learning (RL)~\citep{williams1992simple,sutton2000policy} or online imitation learning (IL)~\citep{ross2011reduction}, policy optimization requires gathering data via potentially expensive interactions with a real or simulated environment. 
 
We focus on algorithms that access the expensive gradient oracle to construct a sequence of surrogate functions. Typically, these surrogates are chosen to be global upper-bounds on the underlying function and hence minimizing the surrogate allows for iterative minimization of the original function. Algorithmically, these surrogate functions can be minimized efficiently \emph{without additional accesses to the gradient oracle}, making this technique advantageous for the applications of interest. The technique of incrementally constructing and minimizing surrogate functions is commonly referred to as \emph{majorization-minimization} and includes the Expectation-Maximization (EM) algorithm~\citep{dempster1977maximum} as an example. In RL, common algorithms~\citep{schulman2015trust, schulman2017proximal} also rely on minimizing surrogates. 

Typically, surrogate functions are constructed by using the convexity and/or smoothness properties of the underlying function. Such surrogates have been used in the stochastic setting~\citep{mairal2013stochastic, mairal2015incremental}. The prox-linear algorithm~\citep{drusvyatskiy2017proximal} for instance, uses the composition structure of the loss function and constructs surrogate functions in the parametric space. Unlike these existing works, we construct surrogate functions over a well-chosen \emph{target space} rather than the parametric space, leveraging the \emph{composition structure} of the loss functions prevalent in ML to build better surrogates. For example, in supervised learning, typical loss functions are of the form $h(\theta) = \ell(f(\theta))$, where $\ell$ is (usually) a convex loss (e.g. the squared loss for regression or the logistic loss for classification), while $f$ corresponds to a transformation (e.g. linear or high-dimensional, non-convex as in the case of neural networks) of the inputs. Similarly, in IL, $\ell$ measures the divergence between the policy being learned and the ground-truth expert policy, whereas $f$ corresponds to a specific parameterization of the policy being learned. More formally, if $\Theta$ is the feasible set of parameters, $f: \Theta \rightarrow \cZ$ is a potentially non-convex mapping from the \emph{parametric space} $\Theta \subseteq \R^d$ to the \emph{target space} $\cZ \subseteq \R^{p}$ and $\ell: \cZ \rightarrow \R$ is a convex loss function. For example, for linear regression, $h(\theta) = \frac{1}{2} \, \normsq{X \theta - y}$, $\z = f(\theta) = X \theta$ and $\ell(\z) = \frac{1}{2} \normsq{\z - y}$. In our applications of interest, computing $\nabla_{z} \ell(\z)$ requires accessing the expensive gradient oracle, but $\nabla_{\theta} f(\theta)$ can be computed efficiently. Unlike~\citet{nguyen2022finite} who exploit this composition structure to prove global convergence, we will use it to construct surrogate functions in the target space.~\citet{johnson2020guided} also construct surrogate functions using the target space, but require access to the (stochastic) gradient oracle for each model update, making the algorithm proposed inefficient in our setting. Moreover, unlike our work,~\citet{johnson2020guided} do not have theoretical guarantees in the stochastic setting. Concurrently with our work,~\citet{woodworth2023two} also consider minimizing expensive-to-evaluate functions, but do so by designing ``proxy'' loss function which is similar to the original function. We make the following contributions. 

\textbf{Target smoothness surrogate}: In~\cref{sec:method-deterministic}, we use the smoothness of $\ell$ with respect to $\z$ in order to define the \emph{target smoothness surrogate} and prove that it is a global upper-bound on the underlying function $h$. In particular, these surrogates are constructed using tighter bounds on the original function. This  ensures that additional progress can be made towards the minimizer using multiple model updates before recomputing the expensive gradient.


\textbf{Target optimization in the deterministic setting}: In~\cref{sec:method-deterministic}, we devise a majorization-minimization algorithm where we iteratively form the target smoothness surrogate and then (locally) minimize it using any black-box algorithm. Although forming the target smoothness surrogate requires access to the expensive gradient oracle, it can be minimized without additional oracle calls resulting in multiple, computationally efficient updates to the model. We refer to this framework as \emph{target optimization}. This idea of constructing surrogates in the target space has been recently explored in the context of designing efficient off-policy algorithms for reinforcement learning~\citep{vaswani2021general}. However, unlike our work,~\citet{vaswani2021general} do not consider the stochastic setting, or provide theoretical convergence guarantees. In~\cref{alg:generic}, we instantiate the target optimization framework and prove that it converges to a stationary point at a sublinear rate (\cref{lemma:deterministic-algo} in~\cref{app:proofs-deterministic}). 

\textbf{Stochastic target smoothness surrogate}: In~\cref{sec:method-stochastic}, we consider the setting where we have access to an expensive stochastic gradient oracle that returns a noisy, but unbiased estimate of the true gradient. Similar to the deterministic setting, we access the gradient oracle to form a \emph{stochastic} target smoothness surrogate. Though the surrogate is constructed by using a stochastic gradient in the target space, it is a deterministic function with respect to the parameters and can be minimized using any standard optimization algorithm. In this way, our framework disentangles the stochasticity in $\nabla_{z} \ell(\z)$ (in the target space) from the potential non-convexity in $f$ (in the parametric space). 

\textbf{Target optimization in the stochastic setting}: Similar to the deterministic setting, we use $m$ steps of GD to minimize the stochastic target smoothness surrogate and refer to the resulting algorithm as stochastic surrogate optimization ($\SSO/$). We then interpret $\SSO/$ as inexact projected SGD in the target space. This interpretation of $\SSO/$ allows the use of standard stochastic optimization algorithms to construct surrogates which can be minimized by any deterministic optimization method. Minimizing surrogate functions in the target space is also advantageous since it allows us to choose the space in which to constrain the size of the updates. Specifically, for overparameterized  models such as deep neural networks, there is only a loose connection between the updates in the parameter and target space. In order to directly constrain the updates in the target space, methods such as natural gradient~\citep{Amari1998,kakade2001natural} involve computationally expensive operations. In comparison, $\SSO/$ has direct control over the updates in the target space and can also be implemented efficiently.

\textbf{Theoretical results for $\SSO/$}: Assuming $h(\theta)$ to be smooth, strongly-convex, in~\cref{sec:theory}, we prove that $\SSO/$ (with a constant step-size in the target space) converges linearly to a neighbourhood of the true minimizer (\cref{prop:gd_sc_bnd}). In~\cref{prop:proj-err}, we prove that the size of this neighbourhood depends on the noise ($\sigma^2_{z}$) in the stochastic gradients in the target space and an error term $\zeta^2$ that depends on the dissimilarity in the stochastic gradients at the optimal solution. In~\cref{prop:sso-lb}, we provide a quadratic example that shows the necessity of this error term in general. However, as the size of the mini-batch increases or the model is over-parameterized enough to interpolate the data~\citep{schmidt2013fast, vaswani2019fast}, $\zeta^2_t$ becomes smaller. In the special case when interpolation is exactly satisfied, we prove that $\SSO/$ with $O\left(\log(\nicefrac{1}{\epsilon})\right)$ iterations is sufficient to guarantee convergence to an $\epsilon$-neighbourhood of the true minimizer. Finally, we argue that $\SSO/$ can be more efficient than using parametric SGD for expensive gradient oracles and common loss functions (\cref{sec:theory-benefits}).   

\textbf{Experimental evaluation}: To evaluate our target optimization framework, we consider online imitation learning (OIL) as our primary example. For policy optimization in OIL, computing $\nabla_{z} \ell$ involves gathering data through interaction with a computationally expensive simulated environment. Using the Mujoco benchmark suite~\citep{todorov2012mujoco} we demonstrate that $\SSO/$ results in superior empirical performance  (\cref{sec:experiments}). We then consider standard supervised learning problems where we compare $\SSO/$ with different choices of the target surrogate to standard optimization methods. These empirical results indicate the practical benefits of target optimization and $\SSO/$.

%% file: Sections/Problem.tex
\section{Problem Formulation}
\label{sec:problem}
We focus on minimizing functions that have a composition structure and for which the gradient is expensive to compute. Formally, our objective is to solve the following problem: $\min_{\theta \in \Theta} h(\theta) := \ell(f(\theta))$ where $\Theta \subseteq \R^d$, $\cZ \subseteq \R^p$, $f: \Theta \rightarrow \cZ$ and $\ell: \cZ \rightarrow \R$. Throughout this paper, we will assume that $h$ is $L_{\theta}$-smooth in the parameters $\theta$ and that $\ell(\z)$ is $L$-smooth in the targets $\z$. For all generalized linear models including linear and logistic regression, $f = X^\top\theta$ is a linear map in $\theta$ and $\ell$ is convex in $\z$. 

\textit{Example}: For logistic regression with features $X \in \R^{n \times d}$ and labels $y \in \{-1,+1\}^{n}$, $h(\theta)=$ $ \sum_{i = 1}^{n} \log\left(1 + \exp(-y_i \langle X_i, \theta \rangle) \right)$. If we ``target'' the logits\footnote{The target space is not unique. For example, we could directly target the classification probabilities for logistic regression resulting in different $f$ and $\ell$.}, then $\cZ = \{\z \vert z = X \theta \} \subseteq \R^n$ and $\ell(\z) =$ $ \sum_{i = 1}^{n} \log\left(1 + \exp(-y_i z_i) \right)$. In this case, $L_\theta$ is the maximum eigenvalue of $X\transpose X$, whereas $L = \frac{1}{4}$. A similar example follows for linear regression.
 
In settings that use neural networks, it is typical for the function $f$ mapping $X$ to $y$ to be non-convex, while for the loss $\ell$ to be convex. For example in OIL, the target space is the space of parameterized policies, where $\ell$ is the cumulative loss when using a policy distribution $\pi := f(\theta)$ who's density is parameterized by $\theta$. Though our algorithmic framework can handle non-convex $\ell$ and $f$, depending on the specific setting, our theoretical results will assume that $\ell$ (or $h$) is (strongly)-convex in $\theta$ and $f$ is an affine map.

For our applications of interest, computing $\nabla_{\z} \ell(\z)$ is computationally expensive, whereas $f(\theta)$ (and its gradient) can be computed efficiently. For example, in OIL, computing the cumulative loss $\ell$ (and the corresponding gradient $\nabla_{z} \ell(\z)$) for a policy involves evaluating it in the environment. Since this operation involves interactions with the environment or a simulator, it is computationally expensive. On the other hand, the cost of computing $\nabla_{\theta} f(\theta)$ only depends on the policy parameterization and does not involve additional interactions with the environment. In some cases, it is more natural to consider access to a \emph{stochastic} gradient oracle that returns a noisy, unbiased gradient $\nabla \tilde{\ell}(\z)$ such that $\E[\nabla \tilde{\ell}(\z)] = \nabla \ell(\z)$. We consider the effect of stochasticity in~\cref{sec:method-stochastic}. 

If we do not take advantage of the composition structure nor explicitly consider the cost of the gradient oracle, iterative first-order methods such as GD or SGD can be directly used to minimize $h(\theta)$. At iteration $t \in [T]$, the \emph{parametric GD} update is: $\thtt = \tht - \eta \nabla h_t(\tht)$ where $\eta$ is the step-size to be selected or tuned according to the properties of $h$. Since $h$ is $L_\theta$-smooth, each iteration of parametric GD can be viewed as exactly minimizing the quadratic surrogate function derived from the smoothness condition with respect to the parameters. Specifically, $\thtt := \argmin \g_t(\theta)$ where $\g_t$ is the \emph{parametric smoothness surrogate}: $\blue{\g_t(\theta) := h(\tht) + \langle \gradh{\tht}, \theta - \tht \rangle + \frac{1}{2 \eta} \normsq{\theta - \tht}}$.

The quadratic surrogate is tight at $\tht$, i.e $h(\tht) = \g_t(\tht)$ and becomes looser as we move away from $\tht$. For $\eta \leq \frac{1}{L_\theta}$, the surrogate is a global (for all $\theta$) upper-bound on $h$, i.e. $\g_t(\theta) \geq h(\theta)$. Minimizing the global upper-bound results in descent on $h$ since $h(\thtt) \leq \g_t(\thtt) \leq \g_t(\tht) = h_t(\tht)$. Similarly, the \emph{parametric SGD} update consists of accessing the stochastic gradient oracle to obtain $\big(\tilde{h}(\theta), \nabla \tilde{h}(\theta)\big)$ such that $\E[\tilde{h}(\theta)] = h(\theta)$ and $\E[\nabla \tilde{h}(\theta)] = \nabla h(\theta)$, and iteratively constructing the \emph{stochastic parametric smoothness surrogate} $\tilde{\g}_t(\theta)$. Specifically, $\thtt = \argmin \tilde{\g}_t(\theta)$, where \blue{$\tilde{\g}_t(\theta) := \tilde{h}(\tht) + \langle \nabla \tilde{h}(\tht), \theta - \tht \rangle + \frac{1}{2 \etat} \normsq{\theta - \tht}$}. Here, $\etat$ is the iteration dependent step-size, decayed according to properties of $h$~\citep{robbins1951stochastic}. In contrast to these methods, in the next section, we exploit the smoothness of the losses with respect to the target space and propose a majorization-minimization algorithm in the deterministic setting. 

%% file: Sections/Deterministic.tex
\section{Deterministic Setting}
\label{sec:method-deterministic}
We consider minimizing $\ell(f)$ in the deterministic setting where we can exactly evaluate the gradient $\nabla_{z} \ell(\z)$. Similar to the parametric case in~\cref{sec:problem}, we use the smoothness of $\ell(\z)$ w.r.t the target space and define the \emph{target smoothness surrogate} around $\zt$ as: $\ell(\z_t) + \langle \nabla_{z} \ell(\zt), \z - \zt \rangle + \frac{1}{2 \eta} \normsq{\z - \zt}$, where $\eta$ is the step-size in the target space and will be determined theoretically. Since $\z = f(\theta)$, the surrogate can be expressed as a function of $\theta$ as: \blue{$\gz_t(\theta) := \left[\ell(\zt) + \langle \nabla_{z} \ell(\z_t), f(\theta) -\zt \rangle + \tfrac{1}{2 \eta} \normsq{f(\theta) - \zt} \right]$}, which in general is not quadratic in $\theta$. 

\textit{Example}: For linear regression, $f(\theta) = X^\top \theta$ and $\gz_t(\theta) = \frac{1}{2} \normsq{X \tht - y} + \big \langle [X \tht - y], X (\theta - \tht) \big \rangle + \frac{1}{2 \eta} \normsq{X (\theta - \theta_t)}$. 


Similar to the parametric smoothness surrogate, we see that $h(\tht) = \ell(f(\tht)) = \gz_t(\tht)$. If $\ell$ is $L$-smooth w.r.t the target space $\cZ$, then for $\eta \leq \frac{1}{L}$, we have $\gz_t(\theta) \geq \ell(f(\theta)) = h(\theta)$ for all $\theta$, that is the surrogate is a global upper-bound on $h$. Since $\gz_t$ is a global upper-bound on $h$, similar to GD, we can minimize $h$ by minimizing the surrogate at each iteration i.e. $\thtt = \argmin_{\theta} \gz_t(\theta)$. However, unlike GD, in general, there is no closed form solution for the minimizer of $\gz_t$, and we will consider minimizing it approximately. 
\setlength{\textfloatsep}{0pt}
\begin{algorithm}[!ht]
\caption{\darkgreen{(Stochastic)} Surrogate optimization}
\label{alg:generic}
\textbf{Input}: $\theta_0$ (initialization), $T$ (number of iterations), $m_t$ (number of inner-loops), $\eta$ (step-size for the target space), $\alpha$ (step-size for the parametric space) 
\begin{algorithmic}
\FOR{$t = 0$  to  $T-1$}
    \STATE Access the \darkgreen{(stochastic)} gradient oracle to construct $\darkgreen{\tilde{\textcolor{black}{\gz_{t}\,}}}(\theta)$
    \STATE Initialize inner-loop: $\omega_0 = \tht$ 
    \FOR{$k \leftarrow 0$ to $m_{t-1}$}
    \STATE $\omega_{k+1} = \omega_{k} - \alpha \nabla_{\omega} \darkgreen{\tilde{\textcolor{black}{\gz_{t}\,}}}(\omega_k)$
    \ENDFOR
    \STATE $\theta_{t+1} = \omega_{m}$ \quad \text{;} \quad $\ztt = f(\thtt)$
\ENDFOR
\STATE Return $\theta_{T}$
\end{algorithmic}
\end{algorithm}

This results in the following meta-algorithm: for each $t \in [T]$, at iterate $\tht$, form the surrogate $\gz_t$ and compute $\thtt$ by (approximately) minimizing $\gz_t(\theta)$. This meta-algorithm enables the use of any black-box algorithm to minimize the surrogate at each iteration. In~\cref{alg:generic}, we instantiate this meta-algorithm by minimizing $\gz_t(\theta)$ using $m \geq 1$ steps of gradient descent. For $m = 1$,~\cref{alg:generic} results in the following update: $\thtt = \tht - \alpha \nabla \gz_t(\tht) = \tht - \alpha \nabla h(\tht)$, and is thus equivalent to parametric GD with step-size $\alpha$. 

\textit{Example}: For linear regression, instantiating~\cref{alg:generic} with $m = 1$ recovers parametric GD on the least squares objective. On the other hand, minimizing the surrogate exactly (corresponding to $m = \infty$) to compute $\thtt$ results in the following update: $\thtt = \theta_t - \eta \, (X\transpose X)^{-1} \left[X\transpose (X \tht - y) \right]$ and recovers the Newton update in the parameter space. In this case, approximately minimizing $\gz_t$ using $m \in (1,\infty)$ steps of GD interpolates between a first and second-order method in the parameter space. In this case, our framework is similar to quasi-Newton methods~\citep{nocedal1999numerical} that attempt to model the curvature in the loss without explicitly modelling the Hessian. Unlike these methods, our framework does not have an additional memory overhead.

In~\cref{lemma:deterministic-algo} in~\cref{app:proofs-deterministic}, we prove that Algorithm~\ref{alg:generic} with any value of $m \geq 1$ and appropriate choices of $\alpha$ and $\eta$ results in an $O(\nicefrac{1}{T})$ convergence to a stationary point of $h$. Importantly, this result only relies on the smoothness of $\ell$ and $g_t^\z$, and does not require either $\ell(\z)$ or $f(\theta)$ to be convex. Hence, this result holds when using a non-convex model such as a deep neural networks, or for problems with non-convex loss functions such as in reinforcement learning. In the next section, we extend this framework to consider the stochastic setting where we can only obtain a noisy (though unbiased) estimate of the gradient. 

%% file: Sections/Stochastic.tex
\section{Stochastic Setting}
\label{sec:method-stochastic}
In the stochastic setting, we use the noisy but unbiased estimates $(\tilde{\ell}(\z), \nabla \tilde{\ell}(\z))$ from the gradient oracle to construct the \emph{stochastic target surrogate}. To simplify the theoretical analysis, we will focus on the special case where $\ell$ is a finite-sum of losses i.e. $\ell(\z) = \frac{1}{n} \sum_{i = 1}^{n} \ell_{i}(\z)$. In this case, querying the stochastic gradient oracle at iteration $t$ returns the individual loss and gradient corresponding to the loss index $i_t$ i.e. $\big(\tilde{\ell}(\z), \nabla \tilde{\ell}(\z)\big) = \left(\ell_{i_t}(\z), \nabla \ell_{i_t}(\z) \right)$. This structure is present in the use-cases of interest, for example, in supervised learning when using a dataset of $n$ training points or in online imitation learning where multiple trajectories are collected using the policy at iteration $t$. In this setting, the deterministic surrogate is \blue{$\gz_{t}(\theta) := \frac{1}{n} \left[\sum \left[\ell_i(\z) + \langle \nabla \ell_i(\z_t), f(\theta) -\zt \rangle \right] + \frac{1}{2 \etat} \normsq{f(\theta) - \zt} \right]$}.

In order to admit an efficient implementation of the stochastic surrogate and the resulting algorithms, we only consider loss functions that are separable w.r.t the target space, i.e. for $\z \in \cZ$, if $\z^i \in \R$ denotes coordinate $i$ of $\z$, then $\ell(\z) = \frac{1}{n} \sum_{i} \ell_i(\z^i)$. For example, this structure is present in the loss functions for all supervised learning problems where $\cZ \subseteq \R^n$ and $z^i = f_i(\theta) := f(X_i, \theta)$. In this setting, $\frac{\partial \ell_{i}}{\partial \z^j} = 0$ for all $i$ and $j \neq i$ and the \emph{stochastic target surrogate} is defined as
\blue{$\tilde{\gz}_t(\theta) := \ell_{i_t}(\zt) + \frac{\partial \ell_{i_t}(\zt)}{\partial z^{i_t}} \left[f_{i_t}(\theta) - \zt^{i_t} \right] + \frac{1}{2 \etat} \left[f_{i_t}(\theta) - \zt^{i_t}\right]^2$}, where $\etat$ is the target space step-size at iteration $t$. Note that $\tilde{\gz}_t(\theta)$ only depends on $i_t$ and only requires access to $\partial \ell_{i_t}(\zt)$ (meaning it can be constructed efficiently). We make two observations about the stochastic surrogate: (i) unlike the parametric stochastic smoothness surrogate $\tilde{g}^{p}$ that uses $i_t$ to form the stochastic gradient, $\tilde{\gz}_t(\theta)$ uses $i_t$ (the same random sample) for both the stochastic gradient $\partial \ell_{i_t}(\zt)$ and the regularization $\left[f_{i_t}(\theta) - \zt^{i_t}\right]^2$; (ii) while $\E_{i_t}[\tilde{\gz}_t(\theta)] = \gz_t(\theta)$, $\E_{i_t}[\argmin \tilde{\gz}_t(\theta)] \neq \argmin \gz_t(\theta)$, in contrast to the parametric case where $\E[\argmin \tilde{g}^{p}_t(\theta)] = \argmin \g_t(\theta)$.

\textit{Example}: For linear regression, $z^i = f_i(\theta) = X_i \theta$ and $\ell_i(\z^i) = \frac{1}{2} (\z^i - y_i)^2$. In this case, the stochastic target surrogate is equal to $\tilde{\gz}_t(\theta) = \frac{1}{2} \, (X_{i_t} \tht - y_{i_t})^2 + [X_{i_t} \tht - y_{i_t}] \cdot X_{i_t} (\theta - \tht) + \frac{1}{2 \etat} \normsq{X_{i_t} (\theta - \theta_t)}$. 

Similar to the deterministic setting, the next iterate can be obtained by (approximately) minimizing $\tilde{\gz}_t(\theta)$. Algorithmically, we can form the surrogate $\tilde{\gz}_t$ at iteration $t$ and minimize it approximately by using any black-box algorithm. We refer to the resulting framework as \emph{stochastic surrogate optimization} ($\SSO/$). For example, we can minimize $\tilde{\gz}_t$ using $m$ steps of GD. The resulting algorithm is the same as~\cref{alg:generic} but uses $\tilde{\gz}_t$ (the changes to the algorithm are highlighted in \darkgreen{green}). Note that the surrogate depends on the randomly sampled $i_t$ and is therefore random. However, once the surrogate is formed, it can be minimized using any deterministic algorithm, i.e. there is no additional randomness in the inner-loop in~\cref{alg:generic}. Moreover, for the special case of $m = 1$,~\cref{alg:generic} has the same update as parametric stochastic gradient descent. Previous work like the retrospective optimization framework in~\citet{newton2021retrospective} also considers multiple updates on the same batch of examples. However, unlike~\cref{alg:generic}, which forces proximity between consecutive iterates in the target space,~\citet{newton2021retrospective} use a specific stopping criterion in every iteration and consider a growing batch-size.

In order to prove theoretical guarantees for $\SSO/$, we interpret it as projected SGD in the target space. In particular, we prove the following equivalence in~\cref{app:smd-surrogate}.
\begin{restatable}{lemma}{restatesmdsurrogate}
\label{lemma:smd-surrogate}
The following updates are equivalent:
\begin{align}
(1)& \quad \tthtt = \argmin_{\theta} \tilde{\gz}_t(\theta) \text{;} \tag{$\SSO/$} \\ 
& \quad \tilde{z}^{(1)}_{t+1} = f(\tthtt) \nonumber  \\
(2)& \quad  \fth = \ft - \etat \nabla_{z} \ell_{i_t}(\ft)  \tag{Target-space SGD} \, \text{;} \\
& \quad \tztt^{(2)} = \argmin_{\z \in \cZ} \frac{1}{2} \indnormsq{\fth - \z}{\cP_t} 
\nonumber
\end{align}
where, $\cP_t \in \R^{p \times p}$ is a random diagonal matrix such that $\cP_t(i_t, i_t) = 1$ and $\cP_t(j,j) = 0$ for all $j \neq i_t$. That is, $\SSO/$ (1) and target space SGD (2), result in the same iterate in each step i.e. if $z_t = f(\tht)$, then $\tftt := \tftt^{(1)} = \tftt^{(2)}$.
\end{restatable}
The second step in target-space SGD corresponds to the projection (using randomly sampled index $i_t$) onto $\cZ$\footnote{For linear parameterization, $z^{i_t} = \langle X_{i_t} , \theta \rangle$ and the set $\cZ$ is convex. For non-convex $f$, the set $\cZ$ can be non-convex and the projection is not well-defined. However, $\tilde{\gz}_t$ can still be minimized, albeit without any guarantees on the convergence.}. 

Using this equivalence enables us to interpret the inexact minimization of $\tilde{\gz}_t$ (for example, using $m$ steps of GD in~\cref{alg:generic}) as an inexact projection onto $\cZ$ and will be helpful to prove convergence guarantees for $\SSO/$. The above interpretation also enables us to use the existing literature on SGD~\citep{robbins1951stochastic, li2021second, vaswani2019painless} to specify the step-size sequence $\{\etat\}_{t = 1}^{T}$ in the target space, completing the instantiation of the stochastic surrogate. Moreover, alternative stochastic optimization algorithms such as follow the regularized leader~\citep{abernethy2009competing} and adaptive gradient methods like AdaGrad~\citep{duchi2011adaptive}, online Newton method~\citep{hazan2007logarithmic}, and stochastic mirror descent~\citep[Chapter 6]{bubeck2015convex} in the target space result in different stochastic surrogates (refer to~\cref{app:algorithm}) that can then be optimized using a black-box deterministic algorithm. Next, we prove convergence guarantees for $\SSO/$ when $\ell(\z)$ is a smooth, strongly-convex function. 

\subsection{Theoretical Results}
\label{sec:theory}
For the setting where $\ell(\z)$ is a smooth and strongly-convex function, we first analyze the convergence of inexact projected SGD in the target space (\cref{sec:convergence}), and then bound the projection errors in in~\cref{sec:proj-error-control}. 
\subsubsection{Convergence Analysis}
\label{sec:convergence}
For the theoretical analysis, we assume that the choice of $f$ ensures that the projection is well-defined (e.g. for linear parameterization where $f = X\transpose \theta$). We define \blue{$\bftt := f(\bthtt)$, where $\bthtt := \argmin \tilde{q}_t(\theta)$ and $\tilde{q}_t := \ell_{i_t}(\zt) + \frac{\partial \ell_{i_t}(\zt)}{\partial z^{i_t}} \left[f_{i_t}(\theta) - \zt^{i_t} \right] + \frac{1}{2 \petat} \normsq{f(\theta) - \zt}$}. In order to ensure that $\E[\tilde{q}_t(\theta)] = \gz_{t}(\theta)$, we will set $\petat = \etat \, n$. Note that $\tilde{q}_t$ is similar to $\tilde{g}_t$, but there is no randomness in the regularization term. Analogous to $\tilde{z}_{t+1}$, $\bar{z}_{t+1}$ can be interpreted as a result of projected (where the projection is w.r.t $\ell_2$-norm) SGD in the target space (\cref{app:lemma:smdsurrogate2}). Note that $\tilde{q}_t$ is only defined for the analysis of $\SSO/$. 

For the theoretical analysis, it is convenient to define $\epstt := \norm{\ftt - \bftt}_{2}$ as the projection error at iteration $t$. Here, $\ftt = f(\thtt)$ where $\thtt$ is obtained by  (approximately) minimizing $\tilde{\gz}_{t}$. Note that the projection error incorporates the effect of both the random projection (that depends on $i_t$) as well as the inexact minimization of $\tilde{\gz}_t$, and will be bounded in~\cref{sec:proj-error-control}.  In the following lemma (proved in~\cref{app:scpro}), we use the equivalence in~\cref{lemma:smd-surrogate} to derive the following guarantee for two choices of $\etat$: (a) constant and (b) exponential step-size~\citep{li2021second, vaswani2022towards}. 
\begin{restatable}{theorem}{restatescpro}
\label{prop:gd_sc_bnd}
Assuming that (i) $\ell_{i_t}$ is $L$-smooth and convex, (ii) $\ell$ is $\mu$-strongly convex, (iii) $\fopt=\arg\min_{z \in \cZ} \ell(z)$ (iv) and that for all $t$, $\epst \leq \epsilon$, $T$ iterations of $\SSO/$ result in the following bound for $\z_T = f(\theta_{T})$,
\setlength{\abovedisplayskip}{2pt}
\setlength{\belowdisplayskip}{2pt}
\begin{align*}
    \E \norm{z_{T+1} - \fopt} & \leq \bigg ( \left(\prod\nolimits_{i=1}^{T}\rho_i\right) \normsq{z_1 - \fopt}\\ &\quad+ 2\sigma^2 \sum\nolimits_{t=1}^T \prod\nolimits_{i=t+1}^{T}\rho_i \petat^2\bigg)^{1/2}\\  &+2\epsilon \sum\nolimits_{t=1}^T \prod\nolimits_{i=t+1}^{T} \rho_i  \, ,  
\end{align*} 
where $\rho_t = (1-\mu \petat)$, $\sigma^2 := \E\left[\normsq{\grad{\fopt} - \gradt{\fopt}} \right]$,\\
(a) \textbf{Constant step-size}: When $\etat= \frac{1}{2L \, n}$ and for $\rho = 1 - \frac{\mu}{2L}$,
\begin{align*}
\E \norm{z_{T+1} - \fopt} & \le \norm{z_{1} - \fopt} \left(1 - \frac{1}{2\kappa}\right)^{\frac{T}{2}} \\ & + \frac{\sigma}{\sqrt{\mu L}} + \frac{2\epsilon}{1 - \sqrt{1 - \frac{1}{2\kappa}}}\,.
\end{align*}
(b) \textbf{Exponential step-size}: When $\etat = \frac{1}{2L \, n} \alpha^{t}$ for $\alpha = (\frac{\beta}{T})^{\frac{1}{T}}$, 
\begin{align*}
    \E \norm{\z_{T+1} -\fopt} & \leq c_1 \, \exp\left( - \frac{T}{4\kappa} \frac{\alpha}{\ln(\nicefrac{T}{\beta})}\right)\norm{\z_1 -\fopt} \\ & + \frac{4 \kappa c_1 (\ln(\nicefrac{T}{\beta}))}{L e  \alpha \sqrt{T}} \sigma + 2 \epsilon \, c_2 \,.
\end{align*}
where $c_1 = \exp\left( \frac{1}{4\kappa} \, \frac{2\beta}{\ln(\nicefrac{T}{\beta})}\right)$, and $c_2 =\exp \left(\frac{\beta \ln (T) }{2\kappa\ln(\nicefrac{T}{\beta}) } \right)$
\end{restatable}
In the above result, $\sigma^2$ is the natural analog of the noise in the unconstrained case. In the unconstrained case, $\nabla\ell(z^*) = 0$ and we recover the standard notion of noise used in SGD analyses~\citep{bottou2018optimization,gower2019sgd}. Unlike parametric SGD, both $\sigma^2$ and $\kappa$ do not depend on the specific model parameterization, and only depend on the properties of $\ell$ in the target space. For both the constant and exponential step-size, the above lemma generalizes the SGD proofs in~\citet{bottou2018optimization} and~\citet{li2021second,vaswani2022towards} to projected SGD and can handle inexact projection errors similar to~\citet{schmidt2011convergence}. For constant step-size, $\SSO/$ results in convergence to a neighbourhood of the minimizer where the neighbourhood depends on $\sigma^2$ and the projection errors $\epstsq$. For the exponential step-sizes, $\SSO/$ results in a noise-adaptive~\citep{vaswani2022towards} $O\big(\exp\left(\frac{-T}{\kappa}\right) + \frac{\sigma^2}{T}\big)$ convergence to a neighbourhood of the solution which only depends on the projection errors. In the next section, we bound the projection errors.  

\subsubsection{Controlling the Projection Error}
\label{sec:proj-error-control}
In order to complete the theoretical analysis of $\SSO/$, we need to control the projection error $\epstt = \norm{\bftt - \ftt}_{2}$ that can be decomposed into two components as follows: $\epstt \leq \norm{\bftt - \tilde{\z}_{t+1}}_{2} + \norm{\tilde{\z}_{t+1} - \z_{t+1}}_{2}$, where $\tilde{z}_{t+1} = f(\tthtt)$ and $\tthtt$ is the minimizer of $\tilde{\gz}_{t}$. The first part of this decomposition arises because of using a stochastic projection that depends on $i_t$ (in the definition of $\tilde{\z}_{t+1}$) versus a deterministic projection (in the definition of $\bftt$). The second part of the decomposition arises because of the inexact minimization of the stochastic surrogate, and can be controlled by minimizing $\tilde{\gz}_t$ to the desired tolerance. In order to bound $\epstt$, we will make the additional assumption that $f$ is $L_f$-Lipschitz in $\theta$.\footnote{For example, when using a linear parameterization, $L_f = \norm{X}$ This property is also satisfied for certain neural networks.} The following proposition (proved in~\cref{app:proj-err-cnt}) bounds $\epstt$.
\begin{restatable}{proposition}{restateprojerr}
\label{prop:proj-err} 
Assuming that (i) $f$ is $L_{f} $-Lipschitz continuous, (ii) $\tilde{\gz}_{t}$ is $\mu_g$ strongly-convex and $L_g$ smooth with $\kappa_g := L_g/\mu_g$\footnote{We consider strongly-convex functions for simplicity, and it should be possible to extend these results to when $\tilde{\gz}$ is non-convex but satisfies the PL inequality, or is weakly convex.}, (iii) $\tilde{q}_{t}$ is $\mu_q$ strongly-convex (iv) $\zeta_t^2 := \frac{8}{\min\{\mu_g, \mu_q\}} \, \left(\left[\min \left\{\E_{i_t} \left[\tilde{g}_t \right] \right\} - \E_{i_t} \left[\min \left\{\tilde{g}_t \right\} \right] \right]\right)$ $+ \left(\frac{8}{\min\{\mu_g, \mu_q\}} \, \left[\min \left\{\E_{i_t} \left[\tilde{q}_t \right] \right\} - \E_{i_t} \left[\min \left\{\tilde{q}_t \right\} \right] \right] \right)$, (v) $\sigma^2_{\z} := \min_{\z \in \cZ} \left\{\E_{i_t} \left[\ell_{i_t}(\z) \right] \right\} - \E_{i_t} \left[\min_{\z \in \cZ} \left\{\ell_{i_t}(\z) \right\} \right]$, if $\tilde{\gz}_t$ is minimized using $m_t$ inner-loops of GD with the appropriate step-size, then, $\E[\epsttsq]  \leq L_f^2 \, \zeta_t^2 + \frac{4 L_f^2}{\mu_g} \left[\exp\left(\nicefrac{-m_t}{\kappa_g} \right) \, \left[\E[\ell(\z_t) - \ell(\z^*)] + \sigma^2_{\z} \right]  \right]$. 
\end{restatable}
Note that the definition of both $\zeta^2_t$ and $\sigma^2_{z}$ is similar to that used in the SGD analysis in~\citet{vaswani2022towards}, and can be interpreted as variance terms. In particular, using strong-convexity, $\min \left\{\E_{i_t} \left[\tilde{g}_t \right] \right\} - \E_{i_t} \left[\min \left\{\tilde{g}_t \right\} \right] \leq \frac{2}{\mu_g} \, \E[\normsq{\nabla \tilde{g}_t(\pthtt) - \E [\nabla \tilde{g}_t(\pthtt)]}]$ which is the variance in $\nabla \tilde{g}_t(\pthtt)$. Similarly, we can interpret the other term in the definition of $\zeta_t^2$. The first term $L_f^2 \, \zeta_t^2$ because of the stochastic versus deterministic projection, whereas the second term $\nicefrac{4 L_f^2}{\mu_g} \, \exp\left(\nicefrac{-m_t}{\kappa_g} \right) \, \left[\E[\ell(\z_t) - \ell(\z^*)] + \sigma^2_{\z} \right]$ arises because of the inexact minimization of the stochastic surrogate. Setting $m_t = O(\log(1/\epsilon))$ can reduce the second term to $O(\epsilon)$. When sampling a mini-batch of examples in each iteration of $\SSO/$, both $\zeta_t^2$ and $\sigma^2_{z}$ will decrease as the batch-size increases (because of the standard sampling-with-replacement bounds~\citep{lohr2019sampling}) becoming zero for the full-batch. Another regime of interest is when using over-parameterized models that can \emph{interpolate} the data~\citep{schmidt2013fast, ma2018power, vaswani2019fast}. The interpolation condition implies that the stochastic gradients become zero at the optimal solution, and is satisfied for models such as non-parametric regression~\cite{liang2018just,belkin2019does} and over-parametrized deep neural networks~\cite{zhang2016understanding}. Under this condition, $\zeta_t^2 = \sigma^2_{z} = 0$~\citep{vaswani2020adaptive,loizou2021stochastic}. From an algorithmic perspective, the above result implies that in cases where the noise dominates, using large $m$ for $\SSO/$ might not result in substantial improvements. However, when using a large batch-size or over-parameterized models, using large $m$ can result in the superior performance.

Given the above result, a natural question is whether the dependence on $\zeta_t^2$ is necessary in the general (non-interpolation) setting. To demonstrate this, we construct an example (details in~\cref{app:counter-example}) and show that even for a sum of two one-dimensional quadratics, when minimizing the surrogate exactly i.e. $m = \infty$ and for any sequence of convergent step-sizes, $\SSO/$ will converge to a neighborhood of the solution. 
\begin{restatable}{proposition}{restatecounter}
\label{prop:sso-lb}   
Consider minimizing the sum $h(\theta) := \frac{h_1(\theta) + h_2(\theta)}{2}$ of two one-dimensional quadratics, $h_1(\theta) := \frac{1}{2}(\theta - 1)^2$ and $h_2(\theta)= \frac{1}{2}\left(2\theta + \nicefrac{1}{2} \right)^2$, using $\SSO/$ with $m_t = \infty$ and $\etat = c \, \alpha_{t}$ for any sequence of $\alpha_t$ and any constant $c \in (0,1]$. $\SSO/$ results in convergence to a neighbourhood of the solution, specifically, if $\theta^*$ is the minimizer of $h$ and $\theta_1 > 0$, then, $\E (\theta_T - \theta^*) \geq \min\left(\theta_1, \frac{3}{8}\right)$. 
\end{restatable}
In order to show the above result, we use the fact that for quadratics, $\SSO/$ (with $m = \infty$) is equivalent to the sub-sampled Newton method and in the one-dimensional case, we can recover the example in~\citet{vaswani2022towards}. Since the above example holds for $m = \infty$, we conclude that this bias is not because of the inexact surrogate minimization. Moreover, since the example holds for all step-sizes including \emph{any} decreasing step-size, we can conclude that the optimization error is not a side-effect of the stochasticity. In order to avoid such a bias term, sub-sampled Newton methods use different batches for computing the sub-sampled gradient and Hessian, and either use an increasing batch-size~\citep{bollapragada2019exact} or consider using over-parameterized models~\citep{meng2020fast}.

\citet{agarwal2020stochastic} also prove theoretical guarantees when doing multiple SGD steps on the same batch. In contrast to our work, their motivation is to analyze the performance of data-echoing~\citep{choi2019faster}. From a technical perspective, they consider (i) updates in the parameteric space, and (ii) their inner-loop step-size decreases as $m$ increases. Finally, we note our framework and subsequent theoretical guarantees would also apply to this setting.

\subsection{Benefits of Target Optimization }
\label{sec:theory-benefits}
In order to gain intuition about the possible benefits of target optimization, let us consider the simple case where each $h_i$ is $\mu_{\theta}$-strongly convex, $L_{\theta}$-smooth and $\kappa_{\theta} = \nicefrac{L_{\theta}}{\mu_{\theta}}$. For convenience, we define $\zeta^2 := \max_{t \in [T]} \zeta^2_t$. Below, we show that under certain regimes, for ill-conditioned least squares problems (where $\kappa_\theta >> 1$), target optimization has a provable advantage over parametric SGD.

\textit{Example}: For the least squares setting, $\kappa_{\theta}$ is the condition number of the $X\transpose X$ matrix. In order to achieve an $\epsilon$ sub-optimality, assuming complete knowledge of $\sigma^2_{\theta} := \E \normsq{\nabla h_i(\theta^*)}$ and all problem-dependent constants, parametric SGD requires \blue{$T_{\text{param}} = O \big(\max \big\{\kappa_{\theta} \log \big(\nicefrac{1}{\epsilon}\big),\nicefrac{\sigma_{\theta}^2}{\mu^2 \epsilon} \big\} \big)$} 
iterations~\citep[Theorem 3.1]{gower2019sgd} where the first term is the bias term and the second term is the effect of the noise. In order to achieve an $\epsilon$ sub-optimality for $\SSO/$, we require that $\zeta^2 \leq \epsilon$ (for example, by using a large enough batch-size) and $O(\kappa_\theta \, \log(1/\epsilon))$ inner iterations. Similar to the parametric case, the number of outer iterations for $\SSO/$ is $T_{\text{target}} = O \big(\max \big\{\log \big(\nicefrac{1}{\epsilon}\big), \nicefrac{\sigma^2}{\epsilon} \big\} \big)$, since the condition number w.r.t the targets is equal to $1$.

In order to model the effect of an expensive gradient oracle, let us denote the cost of computing the gradient of $\ell$ as $\tau$ and the cost of computing the gradient of the surrogate as equal to $1$. When the noise in the gradient is small and the bias dominates the number of iterations for both parametric and target optimization, the cost of parametric SGD is dominated by $\tau T_{\text{param}} = O(\kappa_{\theta} \tau)$, whereas the cost of target optimization is given by $T_{\text{target}} \times [\tau + \kappa_{\theta} \log(1/\epsilon)] = O(\tau + \kappa_{\theta})$. Alternatively, when the noise dominates, we need to compare the $O\big(\tau \nicefrac{\sigma^2_{\theta}}{\mu^2 \epsilon}\big)$ cost for the parametric case against the $O\big(\frac{\sigma^2 }{\epsilon} \, [\tau + \kappa_{\theta} \log(1/\epsilon)]\big)$ cost for target optimization. Using the definition of the noise, $\sigma^2_{\theta} = \E_i \normsq{X_i\transpose (X_i \theta^* - y_i)} \leq   L \sigma^2$. By replacing $\sigma^2_\theta$ by $L \sigma^2$, we again see that the complexity of parametric SGD depends on $O(\kappa_{\theta} \tau)$, whereas that of $\SSO/$ depends on $O(\kappa_{\theta} + \tau)$. 

A similar property can also be shown for the logistic regression. Similarly, we could use other stochastic optimization algorithms to construct surrogates for target optimization. See \cref{app:algorithm} and \cref{app:experiments} for additional discussion.


%% file: Sections/Experiments.tex
\begin{figure*}[!ht]
\centering 
    \includegraphics[width=\textwidth]{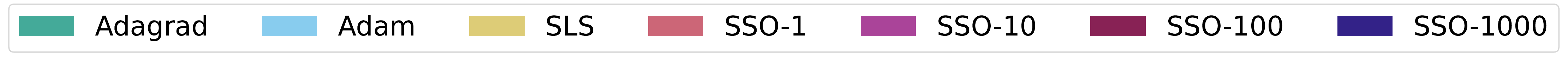}
    \begin{subfigure}{.495 \linewidth}
    \centering
        \includegraphics[width=\textwidth,]{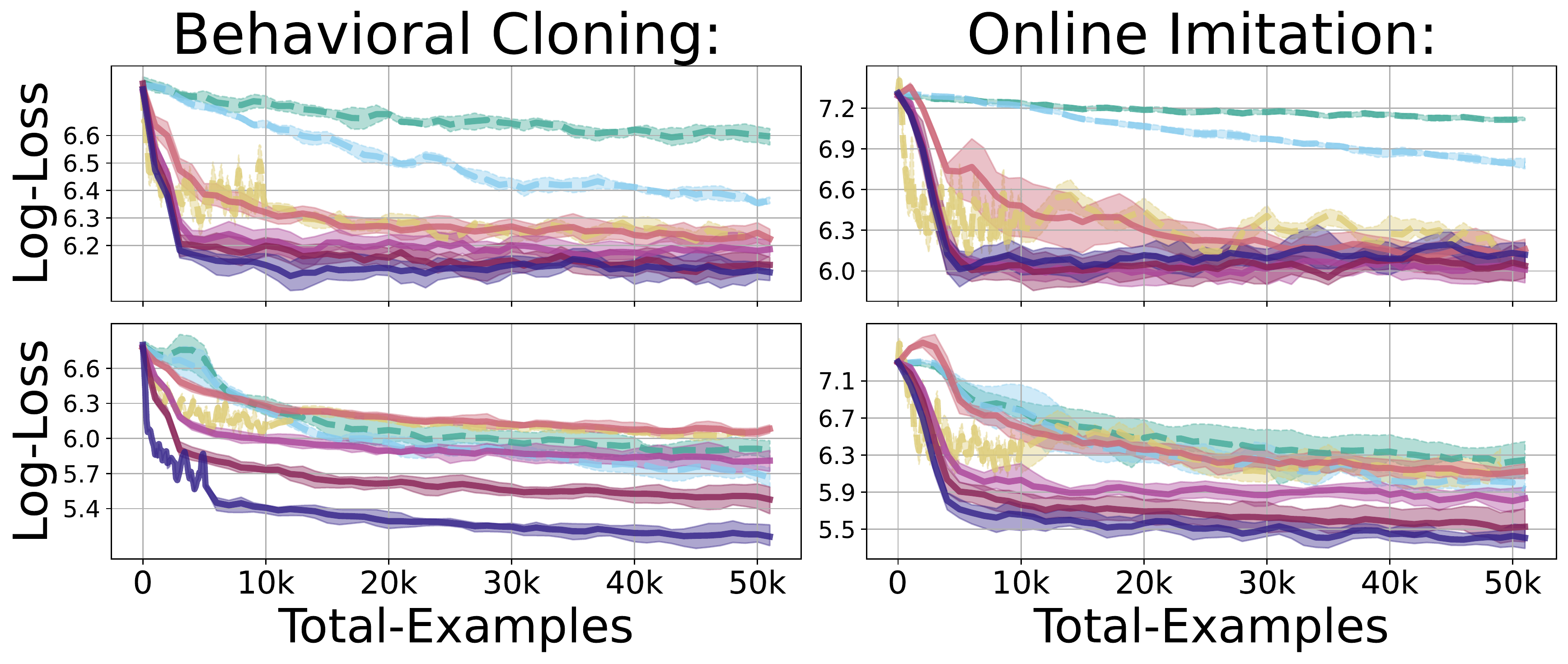}
        \caption{Hopper-v2 Environment}
    \end{subfigure}
    \begin{subfigure}{.495 \linewidth}
    \centering
        \includegraphics[width=\textwidth]{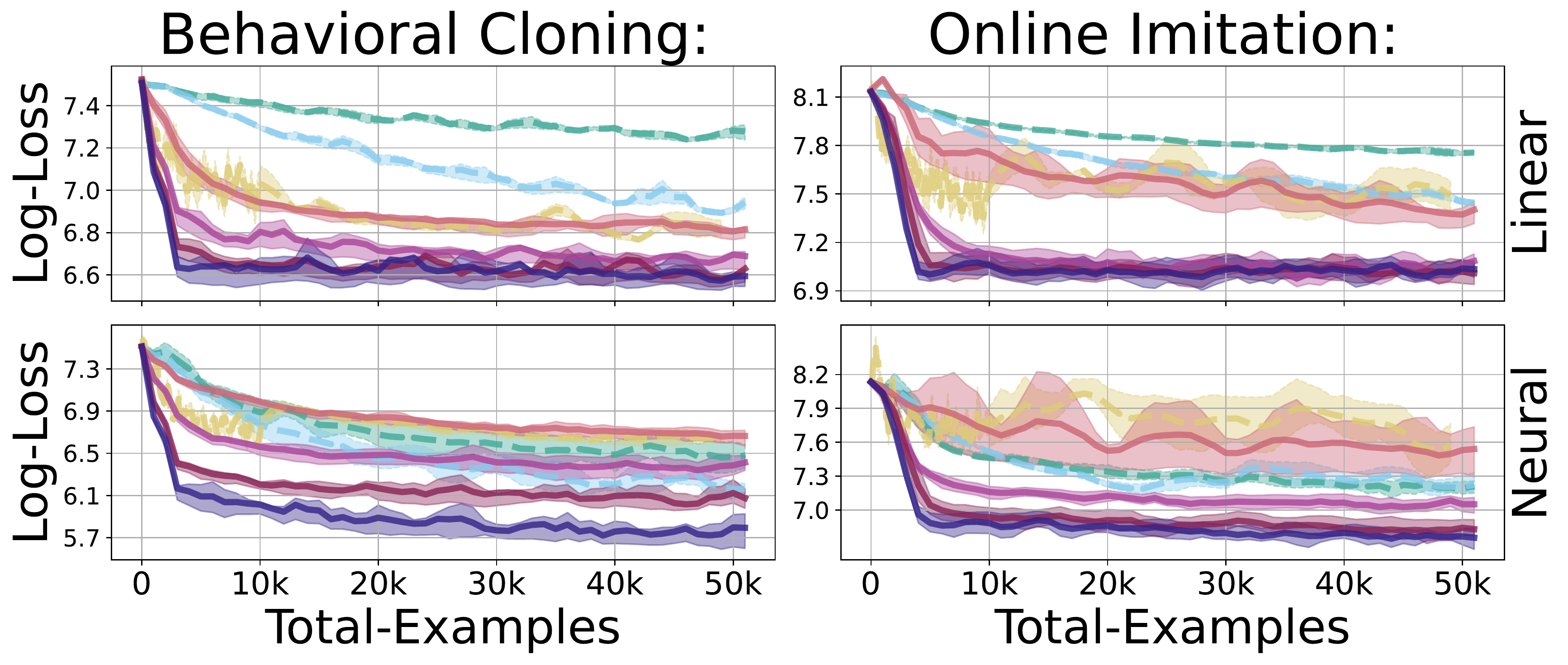}
        \caption{Walker2d-v2 Environment}
     \end{subfigure} 
    \caption{Comparison of log policy loss (mean-squared error between the expert labels and the mean action produced by the policy model) incurred by \texttt{SGD}, \texttt{SLS}, \texttt{Adam}, \texttt{Adagrad}, and \texttt{SSO} as a function of the total interactions (equal to $t$ in~\cref{alg:generic}). \texttt{SSO-m} in the legend indicates that the surrogate has been minimized for $m$ GD steps. The bottom row shows experiments where the policy is parameterized by a neural network, while the top row displays an example where the policy is parameterized by a linear model. Across all environments, behavioral policies, and model types, \texttt{SSO} outperforms all other online-optimization algorithms. Additionally, as m increases, so does the performance of \texttt{SSO}. }
    \label{fig:oil_fig}
\end{figure*}
\begin{figure*}[!ht]
    \includegraphics[width=\textwidth]{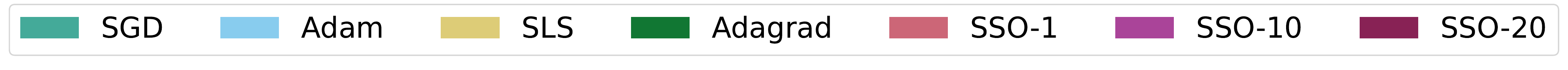} 
    \includegraphics[width=\textwidth]{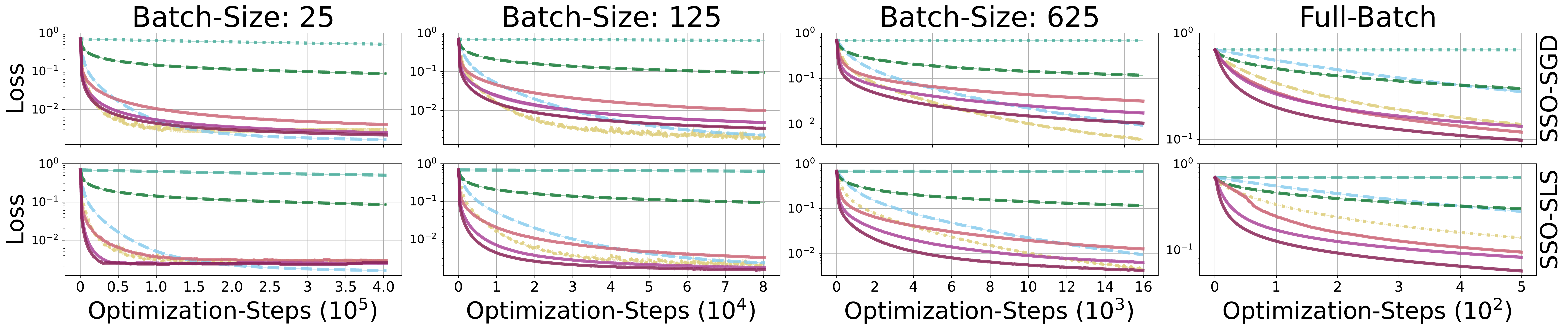}  
    \caption{Comparison of \texttt{SGD} and its SSO variant (top row), \texttt{SLS} and it SSO variant (bottom row) over the rcv1 dataset~\cite{chang2011libsvm} using a logistic loss. \texttt{Adam} and \texttt{Adagrad} are included as baselines. All plots are in log space, where the x-axis defines optimization steps (equal to $t$ in~\cref{alg:generic}). We note that \texttt{SGD} with its theoretical step-size is outperformed by more sophisticated algorithms like \texttt{SLS} or \texttt{Adam}. In contrast, \texttt{SSO} with the theoretical step-size is competitive with both \texttt{SLS} and \texttt{Adam} with default hyper-parameters. Notably, the SSO variant of both SLS and SGD outperforms its parametric counterpart across both m and batch-size.}
    \label{fig:sso_variants}
\end{figure*}
\begin{figure}[!ht]
    \centering 
    \includegraphics[width=0.49\textwidth]{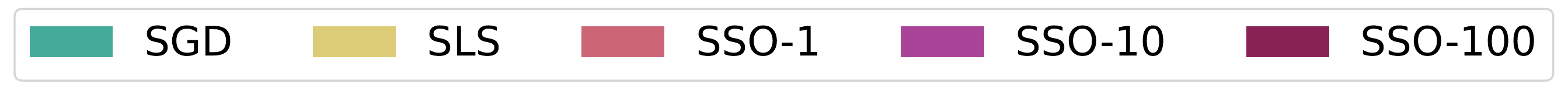}
    \includegraphics[width=0.49\textwidth]{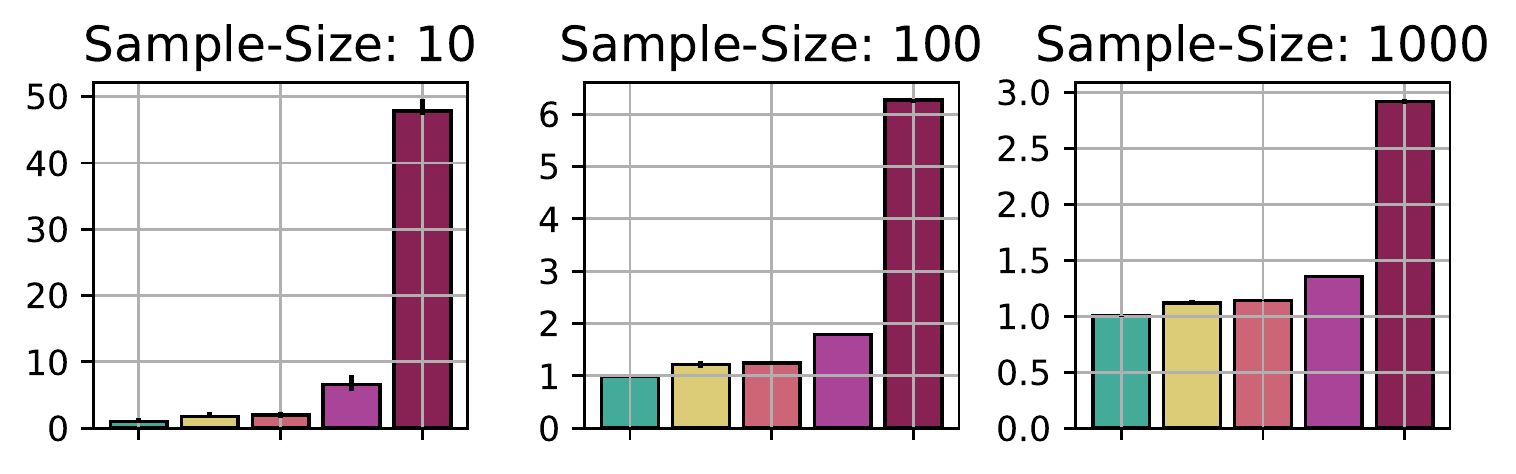}
    \caption{Comparison of run-times normalized with respect to the cost of a single SGD update) between SSO and relevant baselines. These plots illustrate that when data collection is expensive, SSO will be as fast as SGD, even for relatively large m. For \cref{fig:oil_fig}, $\tau\approx1000$, and the ratio of the time required to compute $\ell$ vs $f$ is approximately $0.001$.} 
    \label{fig:rcv1_runtime}
\end{figure}

\section{Experimental Evaluation}
\label{sec:experiments}

We evaluate the target optimization framework for online imitation learning and supervised learning\footnote{The code is available at 
\href{https://github.com/WilderLavington/Target-Based-Surrogates-For-Stochastic-Optimization}{\tt http://github.com/\newline WilderLavington/Target-Based-Surrogates-For-\newline Stochastic-Optimization}.
}. In the subsequent experiments, we use either the theoretically chosen step-size when available, or the default step-size provided by~\citet{paszke2017automatic}. We do not include a decay schedule for any experiments in the main text, but consider both the standard $\nicefrac{1}{\sqrt{t}}$ schedule~\citep{bubeck2015convex}, as well as the exponential step-size-schedule~\citep{vaswani2022towards,orabona2019modern} in~\cref{app:experiments}. For \texttt{SSO}, since optimization of the surrogate is a deterministic problem, we use the standard back-tracking Armijo line-search~\citep{armijo1966minimization} with the same hyper-parameters across all experiments. For each experiment, we plot the average loss against the number of calls to the (stochastic) gradient oracle. The mean and the relevant quantiles are reported using three random seeds. 

\textbf{Online Imitation Learning:}  
We consider a setting in which the losses are generated through interaction with a simulator. In this setting, a behavioral policy gathers examples by observing a state of the simulated environment and taking an action at that state. For each state gathered through the interaction, an expert  policy provides the action that it would have taken. The goal in imitation learning is to produce a policy which imitates the expert. The loss measures the discrepancy between the learned policy $\pi$ and the expert policy. In this case, $z = \pi$ where $\pi$ is a distribution over actions given states, $\ell_{i_t}(\z) = \E_{s_t} \left[\mathbf{KL}(\pi(\cdot|s_t)||\pi_{\text{expert}}(\cdot|s_t)) \right]$ where the expectation is over the states visited by the behavioral policy and $f(\theta)$ is the policy parameterization. Since computing $\ell_{i_t}$ requires the behavioural policy to interact with the environment, it is  expensive. Furthermore, the KL divergence is $1$-strongly convex in the $\ell_1$-norm, hence OIL satisfies all our assumptions. When the behavioral policy is the expert itself, we refer to the problem as \emph{behavioral cloning}. When the learned policy is used to interact with the environment~\citep{pmlr-v164-florence22a}, we refer to the problem as \emph{online imitation learning} (OIL)~\citep{lavington2022improved, ross2011reduction}. 

In \cref{fig:oil_fig}, we consider continuous control environments from the Mujoco benchmark suite~\citep{todorov2012mujoco}. The policy corresponds to a standard normal distribution whose mean is parameterized by either a linear function or a neural network. For gathering states, we sample from the stochastic (multivariate normal) policy in order to take actions. At each round of environment interaction 1000 states are gathered, and used to update the policy. The expert policy, defined by a normal distribution and parameterized by a two-layer MLP is trained using the Soft-Actor-Critic Algorithm~\citep{haarnoja2018soft}.~\cref{fig:oil_fig} shows that \texttt{SSO} with the theoretical step-size drastically outperforms standard optimization algorithms in terms the log-loss as a function of environment interactions (calls to the gradient oracle). Further, we see that for both the linear and neural network parameterization, the performance of the learned policy consistently improves as $m$ increases.

\textit{Runtime Comparison:} In~\cref{fig:rcv1_runtime}, we demonstrate the relative run-time between algorithms. Each column represents the average run-time required to take a single optimization step (sample states and update the model parameters) normalized by the time for SGD. We vary the number of states gathered (referred to as sample-size) per step and consider sample-sizes of $10$, $100$ and $1000$. The comparison is performed on the \texttt{Hopper-v2} environment using a two layer MLP. We observe that for small sample-sizes, the time it takes to gather states does not dominate the time it takes to update the model (for example, in column 1, \texttt{SSO-100} takes almost 50 times longer than SGD). On the other hand, for large sample-sizes, the multiple model updates made by $\SSO/$ are no longer a dominating factor (for example, see the right-most column where \texttt{SSO-100} only takes about three times as long as SGD but results in much better empirical performance). This experiment shows that in cases where data-access is the major bottleneck in computing the stochastic gradients, target optimization can be beneficial, matching the theoretical intuition developed in~\cref{sec:theory-benefits}. For applications with more expensive simulators such as those for autonomous vehicle~\cite{Dosovitskiy17}, $\SSO/$ can result in further improvements. 

\textbf{Supervised Learning}: In order to explore using other optimization methods in the target optimization framework, we consider a simple supervised learning setup. In particular, we use the  the \texttt{rcv1} dataset from libsvm~\citep{chang2011libsvm} across four different batch sizes under a logistic-loss. We include additional experiments over other data-sets, and optimization algorithms in~\cref{app:experiments}.\footnote{ We also compared \SSO/ against SVRG~\citep{johnson2013accelerating}, a variance reduced method, and found that \SSO/ consistently outperformed it across the batch-sizes and datasets we consider. For an example of this behavior see \cref{app:fig:svrg} in \cref{app:experiments}.}

\textit{Extensions to the Stochastic Surrogate}: We consider using a different optimization algorithm -- stochastic line-search~\citep{vaswani2019painless} (\texttt{SLS}) in the target space, to construct a different surrogate. We refer to the resulting algorithm as \texttt{SSO-SLS}. For \texttt{SSO-SLS}, at every iteration, we perform a backtracking line-search in the target space to set $\etat$ that satisfies the Armijo condition: $\ell_{i_t}(z_t -\etat \nabla_z \ell_{i_t}(z_t)) \leq \ell_{i_t}(z_t) - \frac{\etat}{2} \normsq{\nabla_z \ell_{i_t}(z_t)}$. We use the chosen $\etat$ to instantiate the surrogate $\tilde{\gz}_t$ and follow~\cref{alg:generic}. In~\cref{fig:sso_variants}, we compare both \texttt{SSO-SGD} (top row) and \texttt{SSO-SLS} (bottom row) along with its parametric variant. We observe that the $\SSO/$ variant of both SGD and SLS (i) does as well as its parametric counterpart across all m, and (ii), improves it for m sufficiently large. 

%% file: Sections/Discussion.tex
\section{Discussion}
\label{sec:discussion}
In the future, we aim to extend our theoretical results to a broader class of functions, and empirically evaluate target optimization for more complex models. Since our framework allows using any optimizer in the target space, we will explore other optimization algorithms in order to construct better surrogates. Another future direction is to construct better surrogate functions that take advantage of the additional structure in the model. For example,~\citet{pmlr-v48-taylor16,pmlr-v151-amid22a} exploit the composition structure in deep neural network models, and construct local layer-wise surrogates enabling massive parallelization. Finally, we also aim to extend our framework to applications such as RL where $\ell$ is non-convex.

%% file: Acknowledgements.tex
\section{Acknowledgements}
We would like to thank Frederik Kunstner and Reza Asad for pointing out mistakes in the earlier version of this paper. This research was partially supported by the Canada CIFAR AI Program, the Natural Sciences and Engineering Research Council of Canada (NSERC) Discovery Grants RGPIN-2022-03669 and RGPIN-2022-04816. 

%% file: Appendix/Setup.tex
\section*{Organization of the Appendix}
\begin{itemize}

   \item[\ref{app:definitions}] \hyperref[app:definitions]{Definitions}
   
   \item[\ref{app:algorithm}] \hyperref[app:algorithm]{Algorithms}
   
   \item[\ref{app:proofs-deterministic}] \hyperref[app:proofs-deterministic]{Proofs in the Deterministic Setting}
   
   \item[\ref{app:proofs-stochastic}] \hyperref[app:proofs-stochastic]{Proofs in the Stochastic Setting}
   
   \item[\ref{app:experiments}] \hyperref[app:experiments]{Additional Experimental Results}
   
\end{itemize}

\section{Definitions}
\label{app:definitions}
Our main assumptions are that each individual function $f_i$ is differentiable, has a finite minimum $f_i^*$, and is $L_i$-smooth, meaning that for all $v$ and $w$, 
\begin{align}
    f_i(v) & \leq f_i(w) + \inner{\nabla f_i(w)}{v - w} + \frac{L_i}{2} \normsq{v - w},
    \tag{Individual Smoothness}
    \label{eq:individual-smoothness}
\end{align}
which also implies that $f$ is $L$-smooth, where $L$ is the maximum smoothness constant of the individual functions. A consequence of smoothness is the following bound on the norm of the stochastic gradients,
\begin{align}
    \norm{\nabla f_i(w)-\nabla f_i^*}^2 
    \leq
    2 L (f_i(w) - f_i^* - \langle \nabla f_i^*, w-w_i^*\rangle).
\end{align}
We also assume that each $f_i$ is convex, meaning that for all $v$ and $w$,
\begin{align}
    f_i(v) &\geq f_i(w) + \inner{\nabla f_i(w)} {v - w},
    \tag{Convexity}
    \label{eq:individual-convexity}
\end{align}
Depending on the setting, we will also assume that $f$ is $\mu$ strongly-convex, meaning that for all $v$ and $w$,
\begin{align}
f(v) & \geq f(w) + \inner{\nabla f(w)}{v - w} + \frac{\mu}{2} \normsq{v - w},
\tag{Strong Convexity}
\label{eq:strong-convexity}
\end{align}

%% file: Appendix/Algorithms.tex
\section{Algorithms}
\label{app:algorithm}
In this section, we will formulate the algorithms beyond the standard SGD update in the target space. We will do so in two ways -- (i) extending SGD to the online Newton step that uses second-order information in~\cref{app:newton} and (ii) extend SGD to the more general stochastic mirror descent algorithm in~\cref{app:smd}. For both (i) and (ii), we will instantiate the resulting algorithms for the squared and logistic losses. 

\subsection{Online Newton Step}
\label{app:newton}
Let us consider the online Newton step w.r.t to the targets. The corresponding update is: 
\begin{align}
\fth & = \ft - \etat [\nabla_{\z}^2 \ell_t(\zt)]\inv \nabla_{\z} \ell_t(\ft) \quad \text{;} \quad \bftt = \argmin_{\z \in \cZ} \frac{1}{2} \indnormsq{\z - \fth}{\cP_t} 
\label{eq:online-newton-1} \\
\ztt & = f(\thtt) \quad \text{;} \quad \thtt = \argmin_{\theta} \left[\langle \nabla_{\z} \ell_t(\zt), f(\theta) - \zt \rangle + \frac{1}{2 \etat} \indnormsq{f(\theta) - \zt}{\nabla^2 \ell_t(\zt)} \right] \label{eq:online-newton-2} 
\end{align}
where $\nabla_{\z}^2 \ell_t(\zt)$ is the Hessian of example of the loss corresponding to sample $i_t$ w.r.t $\z$. Let us instantiate this update for the squared-loss. In this case, $\ell_t(\z) = \frac{1}{2} \normsq{\z - y_t}$, and hence, $\nabla \ell_t(\z) = \z - y_t$, $[\nabla^2 \ell_t(\z)]_{i_t, i_t} = 1$ and $[\nabla^2 \ell_t(\z)]_{j, j} = 0$ for all $j \neq i_t$. Hence, for the squared loss,~\cref{eq:online-newton-1} is the same as GD in the target space. 

For the logistic loss, $\ell_t(\z) = \log\left(1 + \exp\left(-y_t \z \right)\right)$. If $i_t$ is the loss index sampled at iteration $t$, then, $[\nabla \ell_t(\z)]_{j} = 0$ for all $j \neq i_t$. Similarly, all entries of $\nabla^2 \ell_t(\z)$ except the $[i_t,i_t]$ are zero. 
\begin{align*}
[\nabla \ell_t(\z)]_{i_t} = \frac{-y_t}{1 + \exp(y_t \, \zt)} \quad \text{;} \quad [\nabla^2 \ell(\z)]_{i_t,i_t},  = \frac{1}{1 + \exp(y_t \zt)} \, \frac{1}{1 + \exp(-y_t \zt)} = (1 - p_t) \,p_t \, , 
\end{align*}
where, $p_t = \frac{1}{1 + \exp(-y_t \zt)}$ is the probability of classifying the example $i_t$ to have the $+1$ label. In this case, the surrogate can be written as:
\begin{align}
\tilde{g}_t^\z(\theta) = \frac{-y_t}{1 + \exp(y_t \, \zt)} \, \left(f_{i_t}(\theta) - \zt^{i_t} \right) + \frac{(1 - p_t) \, p_t}{2 \etat} \left(f_{i_t}(\theta) - \zt^{i_t}\right)^2
\label{eq:logistic-newton-surrogate}
\end{align}    
As before, the above surrogate can be implemented efficiently.

\subsection{Stochastic Mirror Descent}
\label{app:smd}
If $\phi$ is a differentiable, strictly-convex mirror map, it induces a Bregman divergence between $x$ and $y$: $D_\phi(y,x) := \phi(y) - \phi(x) - \langle \nabla \phi(x), y - x \rangle$. For an efficient implementation of stochastic mirror descent, we require the Bregman divergence to be separable, i.e. $D_\phi(y,x)=\sum_{j=1}^p D_{\phi_j}(y^j,x^j)= \sum_{j=1}^p \phi_j(y^j) - \phi_j(x^j) - \frac {\partial \phi_j(x)}{\partial x^j}[ y^j - x^j]$. Such a separable structure is satisfied when $\phi$ is the Euclidean norm or negative entropy. We define stochastic mirror descent update in the target space as follows,  
\begin{align}
\nabla \phi(\fth) & = \nabla \phi(\ft) - \etat \nabla_{\z} \ell_t(\ft) \quad \text{;} \quad \bftt = \argmin_{\z \in \cZ} \sum_{j=1}^p \mathbb I(j=i_t)D_{\phi_j}(\z^j, \z^j_{t+1/2})) \label{eq:fsmd-1} \\
\implies \bftt & = \argmin_{\z \in \cZ} \left[\langle \nabla_{z} \ell_t(\zt), \z - \zt \rangle + \frac{1}{ \etat} \sum_{j=1}^p \mathbb I(j=i_t)D_{\phi_j}(\z^j, \z^j_{t+1/2}) \right] \label{eq:fsmd-2}
\end{align}
where $\mathbb I$ is an indicator function and $i_t$ corresponds to the index of the sample chosen in iteration $t$. For the Euclidean mirror map, $\phi(\z) = \frac{1}{2} \normsq{z}$, $D_{\phi}(\z, \zt) = \frac{1}{2} \normsq{\z - \zt}$ and we recover the SGD update. 

Another common choice of the mirror map is the negative entropy function: $\phi(x) = \sum_{i = 1}^{K} x^i \, \log(x^i)$ where $x^i$ is coordinate $i$ of the $x \in \R^K$. This induces the (generalized) KL divergence as the Bregman divergence, 
\[
D_{\phi}(x, y) = \sum_{k = 1}^{K} x^k \, \log \left(\frac{x^k}{y^k}\right) + \sum_{k = 1}^{K} x^k - \sum_{k = 1}^{K} y^{k} \,.
\]
If both $x$ and $y$ correspond to probability distributions i.e. $\sum_{k = 1}^{K} x^k = \sum_{k = 1}^{K} y^k = 1$, then the induced Bregman divergence corresponds to the standard KL-divergence between the two distributions. For multi-class classification, $\cZ \subseteq \mathbb R^{p\times K}$ and each $z^i \in \Delta_K$ where $\Delta_K$ is $K$-dimensional simplex. We will refer to coordinate $j$ of $\z^{i}$ as $[\z^i]_{j}$. Since $\z^i \in \Delta_{K}$, $[\z^i]_k \geq 0$ and $\sum_{k = 1}^{K} [\z^i]_k = 1$.   

Let us instantiate the general SMD updates in~\cref{eq:fsmd-1} when using the negative entropy mirror map. In this case, for $\z \in \Delta_{K}$,  $[\nabla \phi(\z)]_{k} = 1 + \log([\z]_k)$. Denoting $\nabla_t := \nabla_{z} \ell_t(\zt)$ and using $[\nabla_t]_k$ to refer to coordinate $k$ of the $K$-dimensional vector $\nabla_t$. Hence,~\cref{eq:fsmd-1} can be written as:
\begin{align}
[z^{i_t}_{t + \nicefrac{1}{2}}]_{k} = [\ft^{i_t}]_k \, \exp\left(-\etat [\nabla_t]_k \right)  
\label{eq:exp-weights}  
\end{align}
For multi-class classification, $y^i \in \{0,1\}^{K}$ are one-hot vectors. If $[y^i]_k$ refers to coordinate $k$ of vector $y^i$, then corresponding log-likelihood for $n$ observations can be written as:
\[
\ell(\z) = \sum_{i = 1}^{n} \sum_{k = 1}^{K} [y^i]_k \, \log([\z^i]_k) \,.
\]
In our target optimization framework, the targets correspond to the probabilities of classifying the points into one of the classes. We use a parameterization to model the vector-valued function $f_i(\theta): \R^d \rightarrow \R^{K}$. This ensures that for all $i$, $\sum_{k = 1}^{K} [f_i(\theta)]_{k} = 1$. Hence, the projection step in~\cref{eq:fsmd-1} can be rewritten as:
\begin{align*}
\min_{\z \in \cZ} D_{\phi}(\z, \fth) = \sum_{k = 1}^{K} [f_{i_t}(\theta)]_k \, \log\left(\frac{[f_{i_t}(\theta)]_k}{[z^{i_t}_{t + \nicefrac{1}{2}}]_{k}}\right)
\end{align*}
where $[z^{i_t}_{t + \nicefrac{1}{2}}]_{k}$ is computed according to~\cref{eq:exp-weights}. Since the computation of $[z^{i_t}_{t + \nicefrac{1}{2}}]_{k}$ and the resulting projection only depends on sample $i_t$, it can be implemented efficiently.

%% file: Appendix/Proofs.tex
\section{Proofs in the Deterministic Setting}
\label{app:proofs-deterministic}

\begin{restatable}{lemma}{restatedeter}
\label{lemma:deterministic-algo}
Assuming that $g^\z_t(\theta)$ is $\beta$-smooth w.r.t. the Euclidean norm and $\eta \leq \frac{1}{L}$, then, for  $\alpha = \nicefrac{1}{\beta}$, iteration $t$ of~\cref{alg:generic} guarantees that $h(\thtt) \geq h(\tht)$ for any number $m \geq 1$ of surrogate steps. In this setting, under the additional assumption that $h$ is lower-bounded by $h^*$, then~\cref{alg:generic}  results in the following guarantee, 
\[
\min_{t \in \{0, \ldots, T-1\}} \normsq{\nabla h(\tht)} \leq \frac{2 \beta \, [h(\theta_0) - h^*]}{T}.
\]
\end{restatable}
\begin{proof}
Using the update in~\cref{alg:generic} with $\alpha = \frac{1}{\beta}$ and the $\beta$-smoothness of $g^\z_t(\theta)$,  for all $k \in [m-1]$, 
\begin{align*}
g^\z_t(\omega_{k+1}) & \leq g^\z_t(\omega_{k}) - \frac{1}{2 \beta} \normsq{\nabla g^\z_t(\omega_k)} \\
\intertext{After $m$ steps,}
g^\z_t(\omega_{m}) & \leq g^\z_t(\omega_{0}) - \frac{1}{2 \beta} \sum_{k = 0}^{m-1} \normsq{\nabla g^\z_t(\omega_k)} \\
\intertext{Since $\theta_{t+1} = \omega_{m}$ and $\omega_{0} = \theta_{t}$ in~\cref{alg:generic},}
\implies g^\z_t(\theta_{t+1}) & \leq g^\z_t(\theta_t) - \frac{1}{2 \beta} \normsq{\nabla g^\z_t(\theta_t)} - \sum_{k = 1}^{m-1} \normsq{\nabla g^\z_t(\omega_k)} \\
\end{align*}
Note that $h(\tht) = g^\z_t(\tht)$ and if $\eta \leq \frac{1}{L}$, then $h(\thtt) \leq g^\z_t(\thtt)$. Using these relations, 
\begin{align*}
h(\thtt) \leq h(\tht) - \left[\underbrace{\frac{1}{2 \beta} \normsq{\nabla g^\z_t(\tht)} + \sum_{k = 1}^{m-1} \normsq{\nabla g^\z_t(\omega_k)}}_{\geq 0} \right] \implies h(\thtt) \leq h(\tht).
\end{align*}
This proves the first part of the Lemma. Since $\sum_{k = 1}^{m-1} \normsq{\nabla g^\z_t(\omega_k)} \geq 0$,
\begin{align*}
h(\thtt) \leq h(\tht) - \frac{1}{2 \beta} \normsq{\nabla g^\z_t(\tht)} \implies \normsq{\nabla h(\tht)} \leq 2 \beta \, [h(\tht) - h(\thtt)] \tag{Since $\nabla h(\tht) = \nabla g^\z_t(\tht)$} 
\end{align*}
Summing from $k = 0$ to $T - 1$, and dividing by $T$,
\begin{align*}
\frac{\normsq{\nabla h(\tht)}}{T} & \leq \frac{2 \beta \, [h(\tht) - h^*]}{T} \implies \min_{t \in \{0, \ldots, T-1\}} \normsq{\nabla h(\tht)} \leq \frac{2 \beta \, [h(\tht) - h^*]}{T} 
\end{align*}
\end{proof}

\newpage

\section{Proofs in the Stochastic Setting}
\label{app:proofs-stochastic}

\subsection{Equivalence of \SSO/ and SGD in Target Space}
\label{app:smd-surrogate}
\restatesmdsurrogate*
\begin{proof}
Since $\ell$ is separable, if $i_t$ is the coordinate sampled from $z$, we can rewrite the target-space update as follows: 
\begin{align*}
\fth^{i_t} & = \ft^{i_t} - \etat \frac{\partial \ell_{i_t}(\ft)}{\partial z^{i_t}} \, \, \, \nonumber \\
\fth^j & = \ft^j \, \, \, \textrm{ when } j \neq i_t\nonumber \; .
\end{align*}
Putting the above update in the projection step we have, 
\begin{align*}
\tftt &= \argmin_{\z \in \cZ} \frac{1}{2}\{ \normsq{\fth^{i_t} - \z^{i_t}}\}\\
    &= \argmin_{\z \in \cZ} \frac{1}{2} \{ \normsq{\ft^{i_t} - \etat \frac{\partial \ell_{i_t}(\ft)}{\partial z^{i_t}} - \z^{i_t}} \}\\
    &= \argmin_{\z \in \cZ} \left\{ \left [ \frac{\partial \ell_{i_t}(\ft)}{\partial \z^{i_t}} \, [\z^{i_t} - \zt^{i_t}] + \frac{1}{2 \etat} \normsq{\z^{i_t} - \zt^{i_t}}  \right ] \right\} \tag{Due to separability of $\ell$} \\       
    \intertext{Since for all $\z \in \cZ$, $\z = f(\theta)$ and $\z^{i} = f_i(\theta)$ for all $i$. Hence $\tftt = f(\tthtt)$ such that,}
    \tthtt &= \argmin_{\theta \in \Theta} \left\{\frac{\partial \ell_{i_t}(\ft)}{\partial \z^{i_t}} \, [f_{i_t}(\theta) - f_{i_t}(\tht)] + \frac{1}{2 \etat} \normsq{f_{i_t}(\theta) - f_{i_t}(\tht)} \right\} = \argmin_{\theta \in \Theta} \tilde{g}^\z_t(\theta) 
\end{align*}
\end{proof}

\begin{lemma}
\label{app:lemma:smdsurrogate2}
Consider the following updates:
\begin{align}
\bthtt &= \argmin_{\theta} \tilde{q}_t(\theta) \quad \text{;} \quad \bftt^{(1)} = f(\bthtt) \tag{$\SSO/$}  \\
\fth & = \ft - \petat \, \nabla_{z} \ell_{i_t}(\ft)  \tag{Target-space SGD} \, \text{;} \\
\bftt^{(2)} & = \argmin_{\z \in \cZ} \frac{1}{2} \indnormsq{\fth - \z}{2} 
\nonumber
\end{align}
$\SSO/$ and target space SGD result in the same iterate in each step i.e. if $z_t = f(\tht)$, then $\bftt := \bftt^{(1)} = \bftt^{(2)}$.
\end{lemma}
\begin{proof}
\begin{align*}
\bar{\z}_{t+1} &= \argmin_{\z \in \cZ} \frac{1}{2}\{ \normsq{\fth - \z}\}\\
    &= \argmin_{\z \in \cZ} \frac{1}{2} \{ \normsq{\ft - \petat \nabla_{z} \ell_{i_t}(\ft) - \z} \}\\
    &= \argmin_{\z \in \cZ} \left\{ \left [ \frac{\partial \ell_{i_t}(\ft)}{\partial \z^{i_t}} \, [\z^{i_t} - \zt^{i_t}] + \frac{1}{2 \petat} \normsq{\z - \zt}  \right ] \right\} \tag{Due to separability of $\ell$} \\       
    \intertext{Since for all $\z \in \cZ$, $\z = f(\theta)$. Hence $\bar{\z}_{t+1} = f(\bthtt)$ such that,}
\bthtt &= \argmin_{\theta \in \Theta} \left\{\frac{\partial \ell_{i_t}(\ft)}{\partial \z^{i_t}} \, [f_{i_t}(\theta) - f_{i_t}(\tht)] + \frac{1}{2 \petat} \normsq{f(\theta) - f(\tht)} \right\} \\
    & = \argmin_{\theta \in \Theta} \tilde{q}_t(\theta) 
\end{align*}
\end{proof}

\subsection{Proof for Strongly-convex Functions}
\label{app:scpro}
We consider the case where $\ell(z)$ is strongly-convex and the set $\cZ$ is convex. We will focus on SGD in the target space, and consider the following updates:
\begin{align*}
\fth &= \ft - \petat \gradt{\ft}  \\
\bftt &= \Pi_{\cZ}[\fth] := \argmin_{\z \in \cZ} \frac{1}{2} \indnormsq{\z - \fth}{2} \\
\norm{\ftt - \bftt} & \leq \epstt
\end{align*}

\begin{lemma}
Bounding the suboptimality (to $\fopt$) of $\ftt$ based on the sub-optimality of $\bftt$ and $\epstt$, we get that 
\begin{align}
\normsq{\ftt - \fopt} \leq \normsq{\bftt - \fopt} +2 \epstt\norm {\ftt - \fopt} \,.
\end{align}
\label{lem:itr_bnd_fth}
\end{lemma}
\begin{proof}
\begin{align*}
\normsq{\ftt - \fopt} 
&= \normsq{\ftt - \bftt + \bftt - \fopt} \\
&= \normsq{\ftt - \bftt} + \normsq{\bftt - \fopt} + 2 \langle \ftt - \bftt, \bftt - \fopt  \rangle \\
& =\normsq{\ftt - \bftt} + \normsq{\bftt - \fopt} +2 \langle \ftt - \bftt, \bftt-\ftt+\ftt - \fopt  \rangle\\
& =\normsq{\ftt - \bftt} + \normsq{\bftt - \fopt} +2 \langle \ftt - \bftt, \ftt - \fopt  \rangle-2\normsq{\ftt - \bftt}\\
& \leq \normsq{\bftt - \fopt} +2 \norm{\ftt - \bftt}\norm {\ftt - \fopt}\\
& \leq \normsq{\bftt - \fopt} +2 \epstt\norm {\ftt - \fopt}
\end{align*}
\end{proof}

Now we bound the exact sub-optimality at iteration $t+1$ by the inexact sub-optimality at iteration $t$ to get a recursion.  

\begin{lemma}
Assuming (i) each $\ell_t$ is $L$-smooth and (ii) $\ell$ is $\mu$-strongly convex and (iii) $\petat \leq \frac{1}{2L}$, we have  
\begin{align} 
\E \normsq{\bftt - \fopt} \leq (1- \mu \petat) \E \normsq{\ft - \fopt}+ 2 \petat^2 \sigma^2
\end{align}
\label{lem:exc_itr_bnd2}
where $\sigma^2 = \E \normsq{\grad{\fopt}-\gradt{\fopt}}$. 
\end{lemma}
\begin{proof}
Since $\fopt \in \cZ$ and optimal, $\fopt = \Pi_{\cZ} [\fopt - \petat \grad{\fopt}]$.
\begin{align*}
\normsq{\bftt - \fopt} &= \normsq{ \Pi_{\cZ}[\ft - \petat\gradt{\ft}] - \Pi_{\cZ} [\fopt - \petat \grad{\fopt}] } \\
& \leq \normsq{[\ft -\petat \gradt{\ft}] - [\fopt - \petat \grad{\fopt}] } \tag{Since projections are non-expansive} \\
& = \normsq{[\ft -\petat \gradt{\ft}] - [\fopt - \petat \gradt{\fopt}] + \petat \left[\grad{\fopt} - \gradt{\fopt}\right]} \\
& = \normsq{\ft - \fopt} + \petat^2 \normsq{\grad{\fopt} - \gradt{\fopt}} + \petat^2 \normsq{\gradt{\ft} - \gradt{\fopt}} +\underbrace{2\petat\langle \ft -\fopt , \grad{\fopt} - \gradt{\fopt}\rangle}_{A_t}\\
& - 2\petat\langle \ft -\fopt , \gradt{\ft} - \gradt{\fopt}\rangle + 2\petat^2 \underbrace{\langle \gradt{\fopt} -\grad{\fopt} , \gradt{\ft} - \gradt{\fopt}\rangle}_{B_t} 
 \\ 
&\leq \normsq{\ft - \fopt} + 2\petat^2 \normsq{\grad{\fopt} - \gradt{\fopt}} + 2\petat^2 \normsq{\gradt{\ft} - \gradt{\fopt}}\\
&+ A_t  - 2\petat\langle \ft -\fopt , \gradt{\ft} - \gradt{\fopt}\rangle \tag{Young inequality on $B_t$}
\end{align*}
Taking expectation w.r.t $i_t$, knowing that $\E A_t=0$ and using that $\sigma^2 = \E \normsq{\grad{\fopt}-\gradt{\fopt}}$.  
\begin{align}
    \E \normsq{\bftt - \fopt} &\leq \E \normsq{\ft - \fopt} + 2 \petat^2 \E \normsq{\gradt{\ft}-\gradt{\fopt}} -2\petat\langle \ft - \fopt, \grad{\ft} - \grad{\fopt} \rangle+ 2 \petat^2 \sigma^2 \nonumber\\ 
    & \leq \E \normsq{\ft - \fopt} + 4 \petat^2 L \, \E \left\{\ell_t(\ft)-\ell_t(\fopt) -\langle \gradt{\fopt}, \ft - \fopt \rangle \right\}-2\petat\langle \ft  \nonumber \\
        &  \quad\quad \quad - \fopt, \grad{\ft} - \grad{\fopt} \rangle+ 2\petatsq \sigma^2 \label{eq:sub_exct_up} \\
    & \leq \E \normsq{\ft - \fopt} + 2\petat \left\{\ell(\ft)-\ell(\fopt) -\langle \grad{\fopt}, \ft - \fopt \rangle \right\} -2\petat\langle \ft - \fopt, \grad{\ft} - \grad{\fopt} \rangle+ 2\petatsq \sigma^2\tag{$\petat < \frac{1}{2L}$}\nonumber \\
    & \leq \E \normsq{\ft - \fopt} + 2\petat \left\{\ell(\ft)-\ell(\fopt) -\langle \grad{\ft}, \ft - \fopt \rangle \right\}+ 2\petatsq \sigma^2\nonumber\\
    &\leq \E \normsq{\ft - \fopt} - \mu \petat\E \normsq{\ft - \fopt} + 2\petatsq \sigma^2\nonumber\tag{strong convexity of $\ell$}\\
    & \leq (1- \mu \petat) \E \normsq{\ft - \fopt}+ 2\petatsq \sigma^2 \nonumber
\end{align}
where in~\cref{eq:sub_exct_up} we use the smoothness of $\ell_t$. 
\end{proof}

\restatescpro*

\begin{proof}
Using~\cref{lem:itr_bnd_fth} and~\cref{lem:exc_itr_bnd2} we have
\begin{align*}
    \E \normsq{\ftt - \fopt} &\leq \E \normsq{\bftt - \fopt} +2 \E [\epstt\norm {\ftt - \fopt}]\\
    & \leq  \underbrace{(1- \mu \petat)}_{\rho_t} \E \normsq{\ft - \fopt}+ 2\petatsq \sigma^2+2 \E [\epstt\norm {\ftt - \fopt}] \\
\intertext{Recursing from $t = 1$ to $T$,}
    \E [\normsq{z_{T+1}- \fopt}] & \leq \left(\prod_{t=1}^{T}\rho_t\right) \normsq{z_1 - z^*} + 2\sigma^2 \sum_{t=1}^T \prod_{i=t+1}^{T}\rho_i \petat^2 + 2 \sum_{t=1}^T \prod_{i=t+1}^{T}\rho_i \E[\epsilon_{t+1} \, \norm{z_{t +1} - z^*}]
\intertext{Denote $u_t := \E \norm{\ft - \fopt}$. By applying Jensen's inequality we know that $u_{T}^2 \leq \E [\normsq{z_{T} - \fopt}]$.}
\implies u_{T+1}^2 & \leq \left(\prod_{t=1}^{T}\rho_t\right) u^2_1 + 2\sigma^2 \sum_{t=1}^T \prod_{i=t+1}^{T}\rho_i \petat^2 + 2 \epsilon \, \sum_{t=1}^T \prod_{i=t+1}^{T}\rho_i \, u_{t+1}   
\end{align*}
Dividing both sides in the previous inequality by $\prod_{i=1}^{T}\rho_i$ leads to:
\begin{align*}
    \left(\prod_{i=1}^{T}\rho_i\right)^{-1} u_{T+1}^2 \le 
    u_1^2 + 2\sigma^2 \left(\prod_{i=1}^{T}\rho_i\right)^{-1}
    \sum_{t=1}^T \petat^2 \left(\prod_{i=t+1}^{T}\rho_i\right) + 2\epsilon \left(\prod_{i=1}^{T}\rho_i\right)^{-1} \sum_{t=1}^T u_{t+1} \left(\prod_{i=t+1}^{T}\rho_i\right) 
\end{align*}
Simplify the above inequality for a generic $\tau \le T$, 
\begin{align*}    \left[\left(\prod_{i=1}^{\tau}\rho_i\right)^{-\frac{1}{2}} u_{\tau+1}\right]^2 \le 
    u_1^2 + 2\sigma^2 \sum_{t=1}^{\tau} \petat^2 \left(\prod_{i=1}^{t}\rho_i\right)^{-1} +
    \sum_{t=1}^{\tau} 2\epsilon \left(\prod_{i=1}^{t}\rho_i\right)^{-\frac{1}{2}}
    \left[\left(\prod_{i=1}^{t}\rho_i\right)^{-\frac{1}{2}} u_{t+1}\right]
\end{align*}
Let $v_{\tau} := \left(\prod_{i=1}^{\tau}\rho_i\right)^{-\frac{1}{2}} u_{\tau+1}$ and
$S_{\tau} := u_1^2 + 2\sigma^2 \sum_{t=1}^{\tau} \petat^2 \left(\prod_{i=1}^{t}\rho_i\right)^{-1}$. 
Let us also denote $\lambda_t := 2\epsilon \left(\prod_{i=1}^{t}\rho_i\right)^{-\frac{1}{2}}$. Observe that $S_0 = u_1^2 = v_0^2$ and $S_{\tau + 1} = S_{\tau} + 2\sigma^2 {\petat^{2}_{\tau+1}} \left(\prod_{i=1}^{\tau+1}\rho_i\right)^{-1}$. Therefore $S_{\tau}$ is an increasing sequence. Re-writing the previous inequality using the new variables leads to the following inequality:
\begin{align*}
    v_{\tau}^2 \le S_{\tau} + \sum_{t=1}^{\tau} \lambda_t v_t
\end{align*}
Using the result from Lemma~\ref{lem:slb_upperbound} we have:
\begin{align*}
    v_{\tau} &\le \frac{1}{2} \sum_{t=1}^{\tau} \lambda_t + 
    \left(S_{\tau} + \left(\frac{1}{2} \sum_{t=1}^{\tau} \lambda_t\right)^2\right)^{\frac{1}{2}}\\ &\le
    \sum_{t=1}^{\tau} \lambda_t + \sqrt{S_{\tau}} \tag{ using $\sqrt{a+b} \leq \sqrt{a} + \sqrt{b}$ for $a,b \geq 0$}
\end{align*}
Writing the inequality above using the original variables results in:
\begin{align*}
    \left(\prod_{i=1}^{\tau}\rho_i\right)^{-\frac{1}{2}} u_{\tau+1} \le \sum_{t=1}^{\tau} 2\epsilon \left(\prod_{i=1}^{t}\rho_i\right)^{-\frac{1}{2}} +
    \left(u_1^2 + 2\sigma^2 \sum_{t=1}^{\tau} \petat^2 \left(\prod_{i=1}^{t}\rho_i\right)^{-1}\right)^{\frac{1}{2}}\\
    u_{\tau+1} \le 2\epsilon \sum_{t=1}^{\tau} \left(\prod_{i=1}^{t}\rho_i\right)^{-\frac{1}{2}} \left(\prod_{i=1}^{\tau}\rho_i\right)^{\frac{1}{2}} + 
    u_1 \left(\prod_{i=1}^{\tau}\rho_i\right)^{\frac{1}{2}} +
    \sqrt{2}\sigma \left(\sum_{t=1}^{\tau} \petat^2 \left(\prod_{i=1}^{t}\rho_i\right)^{-1}\right)^{\frac{1}{2}} \left(\prod_{i=1}^{\tau}\rho_i\right)^{\frac{1}{2}}
\end{align*}
(a) \textbf{Constant step size}: Choosing a constant step size $\petat = \eta = \frac{1}{2L}$ implies $\rho_i = \rho = 1 - \mu\eta = 1 - \frac{\mu}{2L}$. Plugging this into the previous inequality leads to:
\begin{align*}
    u_{\tau + 1} &\le 2\epsilon \rho^{\frac{\tau}{2}} \sum_{t=1}^{\tau} \rho^{-\frac{t}{2}} + 
    u_1 \rho^{\frac{\tau}{2}} +
    \sqrt{2}\sigma \eta \, \rho^{\frac{\tau}{2}} \sqrt{\sum_{t=1}^{\tau}\rho^{-t}} \\
    & \le \frac{2\epsilon}{1 - \sqrt{\rho}} + 
    u_1 \rho^{\frac{\tau}{2}} + 
    \frac{\sqrt{2}\sigma \eta}{\sqrt{1 - \rho}} \tag{applying the formula for finite geometric series}  \\
\E \norm{z_{T+1} - \fopt} & \le \norm{z_{1} - \fopt} \left(1 - \frac{1}{2\kappa}\right)^{\frac{T}{2}} + \frac{\sigma}{\sqrt{\mu L}} + \frac{2\epsilon}{1 - \sqrt{1 - \frac{1}{2\kappa}}}\,.
\end{align*}

\newpage
(b) \textbf{Exponential step size}. 

    Starting from~\cref{eq:sub_exct_up} and using the proof from~\citet{vaswani2022towards}, we have 
    \begin{align*}
        \E \normsq{\bftt - \fopt}& \leq \E \normsq{\ft - \fopt} + 4\petatsq L\E \left\{\ell_t(\ft)-\ell_t(\fopt) -\langle \gradt{\fopt}, \ft - \fopt \rangle \right\}-2\petat\langle \ft - \fopt, \grad{\ft} \\
        &  \quad\quad \quad- \grad{\fopt} \rangle+ 2\petatsq \sigma^2\\
        & \leq \E \normsq{\ft - \fopt} + \frac{\alphatsq}{L} \E \left\{\ell_t(\ft)-\ell_t(\fopt) -\langle \gradt{\fopt}, \ft - \fopt \rangle \right\}-\frac{\alphat}{L}\langle \ft - \fopt, \grad{\ft}  \\
        &  \quad\quad \quad - \grad{\fopt} \rangle+ 2\petatsq \sigma^2 \tag{ using $\petat=\frac{\alphat}{2L}$}\\
        & \leq \E \normsq{\ft - \fopt} + \frac{\alphat}{L}  \left\{\ell(\ft)-\ell(\fopt) -\langle \grad{\fopt}, \ft - \fopt \rangle \right\}-\frac{\alphat}{L}\langle \ft - \fopt, \grad{\ft} - \grad{\fopt} \rangle+ 2\petatsq \sigma^2 \tag{ using $\alphat\leq 1$}\\
        &= \E \normsq{\ft - \fopt} + \frac{\alphat}{L}  \left\{\ell(\ft)-\ell(\fopt) -\langle \grad{\ft}, \ft - \fopt \rangle \right\}+ 2\petatsq \sigma^2\\
        & \leq \E \normsq{\ft - \fopt} - \frac{\mu \alphat}{2L}\E \normsq{\ft - \fopt}+ 2\petatsq \sigma^2 = \left(1 - \frac{1}{2\kappa}\alphat\right)\E \normsq{\ft - \fopt} + 2\petatsq \sigma^2 \tag{using strong convexity of $\ell$}
    \end{align*}
Combining the above with~\cref{lem:itr_bnd_fth} we get: 
\begin{align*}
    \E \normsq{\ftt -\fopt} &\leq \left(1 - \frac{1}{2\kappa}\alphat\right)\E \normsq{\ft - \fopt}+2\petatsq \sigma^2  + 2 \E [\epstt\norm {\ftt - \fopt}]  \\
    &\leq \exp{\left(- \frac{1}{2\kappa}\alphat\right)}\E \normsq{\ft - \fopt} +  \frac{\sigma^2}{2L^2}\alphatsq +2 \E [\epstt\norm {\ftt - \fopt}] \tag{$1 - x \leq \exp(-x)$}\\
    &\leq \exp{\left(- \frac{1}{2\kappa}\alphat\right)}\E \normsq{\ft - \fopt} + \frac{\sigma^2}{2L^2}\alphatsq +2 \epsilon \E [\norm {\ftt - \fopt}] \tag{$\epstt \leq \epsilon$}
\end{align*}

Unrolling the recursion starting from $t = 1$ to $T$, denoting $u_t := \E \norm{\ft - \fopt}$ and applying Jensen's inequality to deduce that that $u_{T+1}^2 \leq \E \normsq{\z_{T+1} - \fopt}$, we get that, 
\begin{align*}
    u_{T+1}^2 &\leq u_1^2 \, \exp\bigg( -\frac{1}{2\kappa}
    \, \sum_{t=1}^T \alphat \bigg) +
    \frac{ \sigma^2}{2L^2} \sum_{t=1}^T \alphatsq \exp\bigg( -\frac{1}{2\kappa} \sum_{i=t+1}^T \alpha_i\bigg) +2\epsilon \sum_{t=1}^T \exp\bigg( -\frac{1}{2\kappa} \sum_{i=t+1}^T \alpha_i\bigg) \, u_{t+1}
\end{align*}
By multiplying both sides by $\exp\bigg( \frac{1}{2\kappa} \sum_{t=1}^T \alphat \bigg)$ we have: 
\begin{align*}
\bigg(\exp\bigg( \frac{1}{4\kappa} \sum_{t=1}^T \alphat\bigg)   u_{T+1}\bigg)^2 &\leq u_1^2+
    \frac{ \sigma^2}{2L^2} \sum_{t=1}^T \alphatsq \exp\bigg( \frac{1}{2\kappa} \sum_{i=1}^t \alpha_i\bigg)+2\epsilon \sum_{t=1}^T \exp\bigg( \frac{1}{2\kappa} \sum_{i=1}^t \alpha_i\bigg)\, u_{t+1}. 
\end{align*}\\
Now let us define $v_\tau : =  \exp\bigg( \frac{1}{4\kappa} \sum_{t=1}^{\tau} \alphat\bigg)   u_{\tau+1}$, $S_{\tau} := u_1^2+
    \frac{ \sigma^2}{2L^2} \sum_{t=1}^{\tau} \alphatsq \exp\bigg( \frac{1}{2\kappa} \sum_{i=1}^t \alpha_i\bigg)$ and $\lambda_t := 2 \epsilon \exp\bigg( \frac{1}{4\kappa} \sum_{i=1}^{t} \alpha_i\bigg)$. Note that $S_{\tau}$ is increasing and $S_0= u_1^2 = v_0^2 \geq 0$ and $\lambda_t > 0$, $v_{\tau} > 0$. By applying Lemma~\ref{lem:slb_upperbound}, similar to the fixed step-size case, for $\tau=T$, we get: 
    \begin{align*}
        \exp\bigg( \frac{1}{4\kappa} \sum_{t=1}^T \alphat\bigg)   u_{T+1} \leq \bigg( u_1^2+
    \frac{ \sigma^2}{2L^2} \sum_{t=1}^{T} \alphatsq \exp\bigg( \frac{1}{2\kappa} \sum_{i=1}^t \alpha_i\bigg)\bigg)^{\nicefrac{1}{2}} + 2 \epsilon \sum_{t=1}^T \exp\bigg( \frac{1}{4\kappa} \sum_{i=1}^{t} \alpha_i\bigg)
    \end{align*}\\
    Multiplying both sides by $\exp\bigg( -\frac{1}{4\kappa} \sum_{t=1}^T \alphat\bigg)$ gives us 
    \begin{align*}
     u_{T+1} &\leq \bigg( \exp\bigg( -\frac{1}{2\kappa} \underbrace{\sum_{t=1}^T \alphat}_{:= A} \bigg)u_1^2+
    \frac{ \sigma^2}{2L^2} \underbrace{\sum_{t=1}^{T} \alphatsq \exp\bigg( \frac{1}{2\kappa} \sum_{i=t+1}^T \alpha_i\bigg) }_{:= B_T}\bigg)^{\nicefrac{1}{2}} + 2 \epsilon \underbrace{\sum_{t=1}^T \exp\bigg( -\frac{1}{4\kappa} \sum_{i=t+1}^T \alpha_i\bigg)}_{C_T}\\
    & =  \left (  u_1^2\exp\bigg( -\frac{1}{2\kappa} A \bigg) + \frac{ \sigma^2}{2L^2} B_T\right)^{1/2} + 2\epsilon \, C_T 
\end{align*}


To bound $A$, we use~\cref{lemma:A-bound} and get 
\begin{align*}  
    u_1^2 \, \exp\bigg( -\frac{1}{2\kappa} A\bigg)\leq \normsq{\z_1 -\fopt} \underbrace{\exp\left( \frac{1}{2\kappa} \, \frac{2\beta}{\ln(\nicefrac{T}{\beta})}\right)}_{: = c_1^2}\exp\left( - \frac{T}{2\kappa} \frac{\alpha}{\ln(\nicefrac{T}{\beta})}\right)
\end{align*}
To bound $B_T$ we use~\cref{lemma:B-bound}
\begin{align*}
    B_T \leq \frac{16 \kappa^2 c_1^2 (\ln(\nicefrac{T}{\beta}))^2}{e^2 \alpha^2 T}
\end{align*}
Finally using~\cref{lemma:C-bound} to bound $C_T$ we get 

\begin{align*}
    \E \norm{\z_{T+1} -\fopt} &\leq \left(c_1^2\exp\left( - \frac{T}{2\kappa} \frac{\alpha}{\ln(\nicefrac{T}{\beta})}\right)\normsq{\z_1 -\fopt} + \frac{16 \kappa^2 c_1^2 (\ln(\nicefrac{T}{\beta}))^2}{2L^2e^2 \alpha^2 T} \sigma^2\right)^{1/2}+ 2\epsilon\underbrace{\exp \left(\frac{\beta \ln (T)}{2\kappa\ln(\nicefrac{T}{\beta}) } \right)}_{c_2} \\
    \implies \E \norm{\z_{T+1} -\fopt} & \leq c_1 \, \exp\left( - \frac{T}{4\kappa} \frac{\alpha}{\ln(\nicefrac{T}{\beta})}\right)\norm{\z_1 -\fopt} + \frac{4 \kappa c_1 (\ln(\nicefrac{T}{\beta}))}{L e  \alpha \sqrt{T}} \sigma + 2 \epsilon \, c_2 \,.
\end{align*}

\end{proof}

\newpage
\subsection{Controlling the Projection Error}
\label{app:proj-err-cnt}
Let us recall the following definitions for the theoretical analysis:
\begin{align*}
\tthtt & := \argmin_{\theta} \tilde{\gz}_t(\theta) \text{;} \quad \tilde{\gz}_t(\theta) := \ell_{i_t}(\zt) + \frac{\partial \ell_{i_t}(\zt)}{\partial z^{i_t}} \left[f_{i_t}(\theta) - \zt^{i_t} \right] + \frac{1}{2 \etat} \left[f_{i_t}(\theta) - \zt^{i_t} \right]^{2} \quad \text{;} \quad \tftt = f(\thtt) \\
\quad \bthtt & := \argmin_{\theta} \tilde{q}_t(\theta) \quad \text{;} \quad \tilde{q}_t(\theta) := \ell_{i_t}(\zt) + \frac{\partial \ell_{i_t}(\zt)}{\partial z^{i_t}} \left[f_{i_t}(\theta) - \zt^{i_t} \right] +  \frac{1}{2 \petat} \normsq{f(\theta) - \zt}  \quad \text{;} \quad \bftt = f(\bthtt) \\
\quad \pthtt & := \argmin_{\theta} g_t(\theta) \quad \text{;} \quad g_t(\theta) := \frac{1}{n} \left[ \left[\sum_{i} \ell_i(\zt) + \langle \nabla \ell_i(\zt), f(\theta) - \zt  \rangle \right] + \frac{1}{2 \etat} \normsq{f(\theta) - \zt} \right] \\
\quad \ftt & = f(\thtt) \,,
\end{align*}
where $\thtt$ is obtained by running $m_t$ iterations of GD on $\tilde{\gz}_t(\theta)$. We will use these definitions to prove the following proposition to control the projection error in each iteration. 
\restateprojerr*
\begin{proof}
Since we obtain $\thtt$ by minimizing $\tilde{g}_t(\theta)$ using $m_t$ iterations of GD starting from $\tht$, using the convergence guarantees of gradient descent~\citep{nesterov2003introductory}, 
\begin{align*}
\normsq{\tthtt - \thtt} & \leq \exp\left(\nicefrac{-m_t}{\kappa_g} \right) \normsq{\tthtt - \tht} \\
\normsq{\thtt - \bthtt} & =  \normsq{\thtt - \tthtt + \tthtt - \bthtt} \leq 2 \normsq{\thtt - \tthtt} + 2 \normsq{\tthtt - \bthtt} \tag{$\normsq{a+b} \leq 2\normsq{a} + 2\normsq{b}$} \\
& \leq 2 \exp\left(\nicefrac{-m_t}{\kappa_g} \right) \normsq{\tthtt - \tht}  + 2 \normsq{\tthtt - \bthtt} \\
& \leq \frac{4}{\mu_g} \exp\left(\nicefrac{-m_t}{\kappa_g} \right) \, [\tilde{g}_t(\tht) - \tilde{g}_t(\tthtt)] + 2 \normsq{\tthtt - \bthtt}
\end{align*}
Taking expectation w.r.t $i_t$,
\begin{align*}
\E_{i_t} \normsq{\thtt - \bthtt} & \leq \frac{4}{\mu_g} \left[
\exp\left(\nicefrac{-m_t}{\kappa_g} \right) \, \E_{i_t} [\tilde{g}_t(\tht) - \tilde{g}_t(\tthtt)] \right] + 2 \E \normsq{\tthtt - \bthtt} 
\end{align*}
Let us first simplify $\E \normsq{\tthtt - \bthtt}$. 
\begin{align*}
\E \normsq{\tthtt - \bthtt} & = \E \normsq{\tthtt - \pthtt + \pthtt - \bthtt} \\
& \leq 2 \E \normsq{\tthtt - \pthtt} + 2 \E \normsq{\pthtt - \bthtt} \\
& \leq \frac{4}{\mu_g} \E [\tilde{g}_t(\pthtt) - \tilde{g}_t(\tthtt)] + \frac{4}{\mu_q} \E [\tilde{q}_t(\pthtt) - \tilde{q}_t(\bthtt)] \\
& = \frac{4}{\mu_g} \left[\min \left\{\E_{i_t} \left[\tilde{g}_t \right] \right\} - \E_{i_t} \left[\min \left\{\tilde{g}_t \right\} \right] \right]+ \frac{4}{\mu_q} \left[\min \left\{\E_{i_t} \left[\tilde{q}_t \right] \right\} - \E_{i_t} \left[\min \left\{\tilde{q}_t \right\} \right] \right] \tag{Since $\E[\tilde{q}] = \E[\tilde{g}] = g$} 
\end{align*}
\newpage
\begin{align*}
2 \E \normsq{\tthtt - \bthtt} & \leq \underbrace{\frac{8}{\min\{\mu_g, \mu_q\}} \, \left(\left[\min \left\{\E_{i_t} \left[\tilde{g}_t \right] \right\} - \E_{i_t} \left[\min \left\{\tilde{g}_t \right\} \right] \right] + \left[\min \left\{\E_{i_t} \left[\tilde{q}_t \right] \right\} - \E_{i_t} \left[\min \left\{\tilde{q}_t \right\} \right] \right] \right)}_{:= \zeta_t^2} \\
\implies 2 \E \normsq{\tthtt - \bthtt}  & \leq \zeta_t^2  
\end{align*}
Using the above relation, 
\begin{align*}
\E_{i_t} \normsq{\thtt - \bthtt} & \leq \frac{4}{\mu_g} \left[
\exp\left(\nicefrac{-m_t}{\kappa_g} \right) \, \E_{i_t} [\tilde{g}_t(\tht) - \tilde{g}_t(\tthtt)] \right] + \zeta^2_t \\
& \leq \frac{4}{\mu_g} \left[
\exp\left(\nicefrac{-m_t}{\kappa_g} \right) \, \E_{i_t} [h_t(\tht) - h_t(\tthtt)]   \right] + \zeta_t^2  \tag{Since $\tilde{g}_t(\tht)  = h_t(\tht)$ and $\tilde{g}_t(\theta) \geq h_t(\theta)$ for all $\theta$} \\
& = \frac{4}{\mu_g} \left[
\exp\left(\nicefrac{-m_t}{\kappa_g} \right) \, \E_{i_t} [h_t(\tht) - h_t^* + h_t^* - h_t(\tthtt)]  \right] + \zeta_t^2 \tag{$h_t^* := \min_{\theta} h_t(\theta)$} \\
& \leq \frac{4}{\mu_g} \left[
\exp\left(\nicefrac{-m_t}{\kappa_g} \right) \, \E_{i_t} [h_t(\tht) - h_t^*]  + \zeta_t^2 \right] \tag{Since $h_t^* \leq h_t(\theta)$ for all $\theta$} \\
& = \frac{4}{\mu_g} \left[
\exp\left(\nicefrac{-m_t}{\kappa_g} \right) \, \left[\E_{i_t} [h_t(\tht) - h_t(\theta^*)] + \E_{i_t}[h_t(\theta^*) - h_t^*] \right]  \right] + \zeta_t^2 \\
& =  \frac{4}{\mu_g} \left[
\exp\left(\nicefrac{-m_t}{\kappa_g} \right) \, \left[\E_{i_t}[\ell_t(\z_t) - \ell_t(\z^*)] + \E_{i_t}[\ell_t(\z^*) - \ell_t^*] \right] \right] + \zeta_t^2\tag{Since $h(\theta) = \ell(f(\theta)) = \ell(\z)$} \\
\E_{i_t}[\normsq{\thtt - \bthtt}] & \leq  \frac{4}{\mu_g} \left[
\exp\left(\nicefrac{-m_t}{\kappa_g} \right) \, \left[\E_{i_t}[\ell_t(\z_t) - \ell_t(\z^*)] + \E_{i_t}[\ell_t(\z^*) - \ell_t^*] \right]  \right] + \zeta_t^2 \\
& \leq \frac{4}{\mu_g} \left[
\exp\left(\nicefrac{-m_t}{\kappa_g} \right) \, \left[[\ell(\z_t) - \ell(\z^*)] + \underbrace{\E_{i_t}[\ell_t(\z^*) - \ell_t^*]}_{:= \sigma^2_z} \right]  \right] + \zeta_t^2
\tag{Since both $\z_t$ and $\z^*$ are independent of the randomness in $\ell_t$ and $\E_{i_t}[\ell_t] = \ell$}  \\
\implies \E[\normsq{\thtt - \bthtt}]  & \leq \frac{4}{\mu_g} \left[
\exp\left(\nicefrac{-m_t}{\kappa_g} \right) \, \left[\ell(\z_t) - \ell(\z^*) + \sigma^2_z \right] \right] + \zeta_t^2
\end{align*}
Now, we will bound $\E[\epsttsq]$ by using the above inequality and the Lipschitzness of $f$.  
\begin{align*}
\E[\epsttsq] &= \normsq{\ftt - \bftt} = \normsq{f(\thtt) - f(\bthtt)} \leq L_f^2 \normsq{\thtt - \bthtt} \tag{Since $f$ is $L_f$-Lipschitz} \\
\implies \E[\epsttsq] & \leq \frac{4 L_f^2}{\mu_g} \left[\exp\left(\nicefrac{-m_t}{\kappa_g} \right) \, \left[\ell(\z_t) - \ell(\z^*) + \sigma^2_{\z} \right] \right] + L_f^2 \, \zeta_t^2  \\
\intertext{Taking expectation w.r.t the randomness from iterations $k = 0$ to $t$,}
\E[\epsttsq] & \leq \frac{4 L_f^2}{\mu_g} \left[\exp\left(\nicefrac{-m_t}{\kappa_g} \right) \, \left[\E[\ell(\z_t) - \ell(\z^*)] + \sigma^2_{\z} \right]  \right] + L_f^2 \, \zeta_t^2  
\end{align*}
\end{proof}

\newpage
\subsection{Example to show the necessity of $\zeta^2$ term}
\label{app:counter-example}
\restatecounter*
\begin{proof}
Let us first compute $\theta^* := \argmin h(\theta)$. 
\begin{align*}
h(\theta) &= \frac{1}{4}(\theta - 1)^2 + \frac{1}{4}\left(2\theta + \frac{1}{2}\right)^2 = \frac{5}{4} \theta^2  + \frac{1}{4} + \frac{1}{16} \Rightarrow \theta^* = 0
\end{align*}
For $h_1$, 
\begin{align*}
\ell_1(\z) & = \frac{1}{2} (z - 1)^2 \quad \text{where } z = \theta \\
\tilde{\gz}_{t}(\theta) & := \frac{1}{2} (\tht - 1)^2 + (\tht - 1) \, (\tht - \theta) + \frac{1}{2 \etat} (\theta - \tht)^2 \\
\intertext{If $m_t = \infty$, $\SSO/$ will minimize $\tilde{\gz}_t$ exactly. Since $\nabla \tilde{\gz}_t(\thtt) = 0$,}
\implies \frac{1}{\etat} (\thtt - \tht) & = - (\tht - 1) \implies \thtt = \tht - \etat \, (\tht - 1) \implies \thtt = \tht - \etat \nabla h_1(\tht)
\end{align*}
Similarly, for $h_2$, 
\begin{align*}
\ell_2(\z) & = \frac{1}{2} (z + \nicefrac{1}{2})^2 \quad \text{where } z = 2 \theta \\
\tilde{\gz}_{t}(\theta) & := \frac{1}{2} (2 \tht + \nicefrac{1}{2})^2 + (2 \tht + \nicefrac{1}{2}) \, (2 \tht - 2 \theta) + \frac{1}{2 \etat} (2 \theta - 2 \tht)^2 \\
\intertext{If $m_t = \infty$, $\SSO/$ will minimize $\tilde{\gz}_t$ exactly. Since $\nabla \tilde{\gz}_t(\thtt) = 0$,}
\implies \frac{4}{\etat} (\thtt - \tht) & = - (4 \tht + 1) \implies \thtt = \tht - \etat \, (\tht + \nicefrac{1}{4}) \implies \thtt = \tht - \frac{\etat}{4} \, \nabla h_2(\tht)
\end{align*}
If $i_t = 1$ and $\eta = c \, \alphat$
\begin{align*}
\thtt = \tht - c \, \alphat \, (\tht - 1) = c \, \alphat + (1 - c \alphat) \tht
\end{align*}
If $i_k = 2$,
\begin{align*}
\thtt = \tht - c \, \alphat \frac{2}{4}(2 \tht + \frac{1}{2}) = (1 - c \, \alphat) \tht - \frac{1}{4} c \, \alphat
\end{align*}
Then
\begin{align*}
\E \thtt = (1- c \, \alphat) \tht + \frac{1}{2} c \, \alphat - \frac{1}{8} c \, \alphat = (1- c \, \alphat) \tht + \frac{3}{8} c \, \alphat
\end{align*}
and
\begin{align*}
\E \theta_T = \E (\theta_T - \theta^*) = (\theta_1 - \theta^*) \prod_{t=1}^T (1 - c \, \alphat) + \frac{3}{8} \sum_{t=1}^T 2(1-c)\alphat \prod_{i=t+1}^T (1 - c \, \alpha_i)
\end{align*}
Using~\cref{lem:sum-prod-eq} and the fact that $c \, \alphat \leq 1$ for all $t$, we have that if $\theta_1 - \theta^* = \theta_1 > 0$, then, 
\begin{align*}
\E (\theta_T - \theta^*) \geq \min\left(\theta_1, \frac{3}{8}\right).  
\end{align*}
\end{proof}

\subsection{Helper Lemmas}
The proofs of~\cref{lem:ineq1},~\cref{lemma:A-bound}, and~\cref{lemma:B-bound} can be found in~\citep{vaswani2022towards}.  

\begin{lemma}
\label{lem:ineq1}
For all $x>1$,
\begin{align*}
    \frac{1}{x - 1} \leq \frac{2}{\ln(x)}
\end{align*}
\end{lemma}
\begin{proof}
For $x > 1$, we have
\begin{align*}
    \frac{1}{x - 1} \leq \frac{2}{\ln(x)} &\iff \ln(x) < 2x - 2
\end{align*}
Define $f(x) = 2x - 2 - \ln(x)$. We have $f'(x) = 2 - \frac{1}{x}$. Thus for $x \geq 1$, we have $f'(x) > 0$ so $f$ is increasing on $[1, \infty)$. Moreover we have $f(1) = 2 - 2 - \ln(1) = 0$ which shows that $f(x) \geq 0$ for all $x > 1$ and ends the proof.
\end{proof}

\begin{lemma}
\label{lem:ineq2}
For all $x, \gamma >0$,
\begin{align*}
    \exp(-x) \leq \left( \frac{\gamma}{ex }\right)^\gamma
\end{align*}
\end{lemma}

\begin{proof}
Let $x > 0$. Define $f(\gamma) = \left( \frac{\gamma}{ex }\right)^\gamma - \exp(-x)$. We have
\begin{align*}
    f(\gamma) = \exp\left( \gamma\ln(\gamma) - \gamma\ln(ex)\right) - \exp(-x)
\end{align*}
and
\begin{align*}
    f'(\gamma) = \left( \gamma \cdot \frac{1}{\gamma} + \ln(\gamma) - \ln(ex)\right) \exp\left( \gamma\ln(\gamma) - \gamma\ln(ex)\right)
\end{align*}
Thus
\begin{align*}
    f'(\gamma) \geq 0 &\iff 1 + \ln(\gamma) - \ln(ex) \geq 0 \iff \gamma \geq \exp\left( \ln(ex) - 1\right) = x
\end{align*}
So $f$ is decreasing on $(0, x]$ and increasing on $[x, \infty)$. Moreover,
\begin{align*}
    f(x) = \left( \frac{x}{ex}\right)^x - \exp(-x) = \left(\frac{1}{e}\right)^x - \exp(-x) = 0
\end{align*}
and thus $f(\gamma) \geq 0$ for all $\gamma > 0$ which proves the lemma.
\end{proof}

\begin{lemma}
\label{lemma:A-bound}
Assuming $\alpha < 1$ we have 
\begin{align*}
A := \sum_{t=1}^T \alpha^t & \geq \frac{\alpha T}{\ln(\nicefrac{T}{\beta})} - \frac{2\beta}{\ln(\nicefrac{T}{\beta})}     
\end{align*}
\end{lemma}
\begin{proof}
\begin{align*}
    \sum_{t=1}^T \alpha^t = \frac{\alpha - \alpha^{T+1}}{ 1- \alpha} = \frac{\alpha}{1 - \alpha} - \frac{\alpha^{T+1}}{1 - \alpha}
\end{align*}
We have
\begin{align}
    \label{ineq:lemma-A-bound}
    \frac{\alpha^{T+1}}{1 - \alpha} &= \frac{\alpha \beta}{T(1 - \alpha)} = \frac{\beta}{T} \cdot \frac{1}{\nicefrac{1}{\alpha} - 1} \leq \frac{\beta}{T} \cdot \frac{2}{\ln(\nicefrac{1}{\alpha})} = \frac{\beta}{T} \cdot \frac{2}{\frac{1}{T}\ln(\nicefrac{T}{\beta})} = \frac{2\beta}{\ln(\nicefrac{T}{\beta})}
    \end{align}
where in the inequality we used Lemma \ref{lem:ineq1} and the fact that $\nicefrac{1}{\alpha} > 1$. Plugging back into $A$ we get, 
\begin{align*}
    A &\geq \frac{\alpha}{1- \alpha} - \frac{2\beta}{\ln(\nicefrac{T}{\beta})}\\
    &\geq \frac{\alpha }{\ln(\nicefrac{1}{\alpha})} - \frac{2\beta}{\ln(\nicefrac{T}{\beta})} \tag{$1-x \leq \ln(\frac{1}{x})$}\\
    &= \frac{\alpha T}{\ln(\nicefrac{T}{\beta})} - \frac{2\beta}{\ln(\nicefrac{T}{\beta})}
\end{align*}
\end{proof}

\begin{lemma}
For $\alpha = \left( \frac{\beta}{T}\right)^{1/T}$ and any $\kappa > 0$,
\begin{align*}
 \sum_{t=1}^T \alpha^{2t} \exp \left(-\frac{1}{2\kappa} \sum_{i=t+1}^T \alpha^i \right) & \leq \frac{16 \kappa^2 c_2 (\ln(\nicefrac{T}{\beta}))^2}{e^2 \alpha^2 T} 
\end{align*}
where $c_2= \exp\left( \frac{1}{2\kappa} \frac{2\beta}{\ln(\nicefrac{T}{\beta}}\right)$
\label{lemma:B-bound}
\end{lemma}

\begin{proof}
First, observe that, 
\begin{align*}
\sum_{i=t+1}^T \alpha^i = \frac{\alpha^{t+1} - \alpha^{T+1}}{1 - \alpha} 
\end{align*}
We have
\begin{align*}
    \frac{\alpha^{T+1}}{1 - \alpha} &= \frac{\alpha \beta}{T(1 - \alpha)} = \frac{\beta}{T} \cdot \frac{1}{\nicefrac{1}{\alpha} - 1} \leq \frac{\beta}{T} \cdot \frac{2}{\ln(\nicefrac{1}{\alpha})} = \frac{\beta}{T} \cdot \frac{2}{\frac{1}{T}\ln(\nicefrac{T}{\beta})} = \frac{2\beta}{\ln(\nicefrac{T}{\beta})}
    \end{align*}
where in the inequality we used~\ref{lem:ineq1} and the fact that $\nicefrac{1}{\alpha} > 1$. These relations imply that, 
\begin{align*}
&\sum_{i=t+1}^T \alpha^i \geq \frac{\alpha^{t+1}}{1 - \alpha} - \frac{2\beta}{\ln(\nicefrac{T}{\beta})} \\
&\implies 
\exp\left( - \frac{1}{2\kappa} \sum_{i=t+1}^T \alpha^i \right) \leq \exp\left( -\frac{1}{2\kappa} \frac{\alpha^{t+1}}{1- \alpha} + \frac{1}{2\kappa} \frac{2\beta}{\ln(\nicefrac{T}{\beta})}\right) = c_2 \exp\left(-\frac{1}{2\kappa} \frac{\alpha^{t+1}}{1- \alpha} \right) 
\end{align*}
We then have
\begin{align*}
\sum_{t=1}^T \alpha^{2t} \exp\left( - \frac{1}{2\kappa} \sum_{i=t+1}^T \alpha^i \right) & \leq c_2 \sum_{t=1}^T \alpha^{2t} \exp\left(-\frac{1}{2\kappa} \frac{\alpha^{t+1}}{1- \alpha} \right) \\
& \leq c_2 \sum_{t=1}^T \alpha^{2t} \left( \frac{2 (1-\alpha) 2\kappa}{e \alpha^{t+1}} \right)^2 & \tag{Lemma \ref{lem:ineq2}} \\
& = \frac{16 \kappa^2 c_2}{e^2 \alpha^2} \, T (1 - \alpha)^2 \\
& \leq \frac{16 \kappa^2 c_2}{e^2 \alpha^2} \, T (\ln(1/\alpha))^2 \\
& = \frac{16 \kappa^2 c_2 (\ln(\nicefrac{T}{\beta}))^2}{e^2 \alpha^2 T}
\end{align*}
\end{proof}

\begin{lemma}
For $\alpha = \left( \frac{\beta}{T}\right)^{1/T}$ and any $\zeta > 0$,
\begin{align*}
 \sum_{t=1}^T \exp \left(-\frac{1}{\zeta} \sum_{i=t}^T \alpha^i \right) & \leq \exp \left(\frac{2\beta \ln (T) }{\zeta\ln(\nicefrac{T}{\beta}) } \right)
\end{align*}
\label{lemma:C-bound}
\end{lemma}
\begin{proof}
First, observe that, 
\begin{align*}
\sum_{i=t}^T \alpha^i = \frac{\alpha^{t} - \alpha^{T+1}}{1 - \alpha} \geq -\frac{ \alpha^{T+1}}{1 - \alpha}
\end{align*}
We have
\begin{align*}
    \frac{\alpha^{T+1}}{1 - \alpha} &= \frac{\alpha \beta}{T(1 - \alpha)} = \frac{\beta}{T} \cdot \frac{1}{\nicefrac{1}{\alpha} - 1} \leq \frac{\beta}{T} \cdot \frac{2}{\ln(\nicefrac{1}{\alpha})} = \frac{\beta}{T} \cdot \frac{2}{\frac{1}{T}\ln(\nicefrac{T}{\beta})} = \frac{2\beta}{\ln(\nicefrac{T}{\beta})}
    \end{align*}
where in the inequality we used~\cref{lem:ineq1} and the fact that $\nicefrac{1}{\alpha} > 1$. Using the above bound we have 
\begin{align*}
    \sum_{t=1}^T \exp \left(-\frac{1}{\zeta} \sum_{i=t}^T \alpha^i \right) & \leq \sum_{t=1}^T \exp \left(\frac{2\beta }{\zeta\ln(\nicefrac{T}{\beta}) } \right) =  \exp \left(\frac{2\beta \ln (T) }{\zeta\ln(\nicefrac{T}{\beta}) } \right)
\end{align*}
\end{proof}

\begin{lemma}
\label{lem:sum-prod-eq}
For any sequence $\alpha_t$
\begin{align*}
    \prod_{t=1}^T (1 - \alphat) + \sum_{t=1}^T \alphat \prod_{i=t+1}^T (1 - \alpha_i)  = 1 
\end{align*}
\end{lemma}
\begin{proof}
We show this by induction on $T$. For $T = 1$,
\begin{align*}
    (1 - \alpha_1) + \alpha_1 = 1
\end{align*}
Induction step:
\begin{align*}
    \prod_{t=1}^{T+1}(1 - \alphat) + \sum_{t=1}^{T+1} \alphat \prod_{i=t+1}^{T+1} (1 - \alpha_i)  &= (1- \alpha_{T+1}) \prod_{t=1}^T (1- \alphat) + \left( \alpha_{T+1} + \sum_{t=1}^T\alphat \prod_{i=t+1}^{T+1} (1 - \alpha_i)\right) \\
    &= (1- \alpha_{T+1}) \prod_{t=1}^T (1- \alphat) + \left( \alpha_{T+1} + (1 - \alpha_{T+1})\sum_{t=1}^T\alphat \prod_{i=t+1}^{T} (1 - \alpha_i)\right)\\
    &= ( 1- \alpha_{T+1}) \left(\underbrace{ \prod_{t=1}^T (1- \alphat) + \sum_{t=1}^T\alphat \prod_{i=t+1}^{T} (1 - \alpha_i)}_{= 1}\right) + \alpha_{T+1} \tag{Induction hypothesis}\\
    &= (1 - \alpha_{T+1}) + \alpha_{T+1} = 1
\end{align*}
\end{proof}

\newpage

\begin{lemma}
\label{lem:slb_upperbound}
\citep[Lemma 1]{schmidt2011convergence}. Assume that the non-negative sequence $\{v_\tau\}$ for $\tau \geq 1$ satisfies the following recursion: 
$$v_\tau^2 \leq S_\tau + \sum_{t=1}^\tau \lambda_t v_t \,,$$
where $\{S_\tau\}$ is an increasing sequence, such that $S_0 \geq v_0^2$ and $\lambda_t \geq 0$. Then for all $\tau \geq 1$,
$$
v_\tau \leq \frac{1}{2} \sum_{t=1}^\tau \lambda_t+ \left ( S_\tau + \left(\frac{1}{2} \sum_{t=1}^\tau \lambda_t\right)^2\right)^{1/2}. 
$$
\end{lemma}
\begin{proof}
We prove this lemma by induction. For $\tau = 1$ we have 
\begin{align*}
    v_1^2 &\leq S_1 + \lambda_1 v_1 \implies (v_1 - \frac{\lambda_1}{2})^2 \leq S_1 + \frac{\lambda_1^2}{4} \implies v_1 \leq \left ( S_1 + \left(\frac{1}{2} \lambda_1\right)^2\right)^{1/2}+\frac{1}{2} \lambda_1 \\
    \intertext{\textbf{Inductive hypothesis}: Assume that the conclusion holds for $\tau \in \{1,\dots , k\}$. Specifically, for $\tau = k$,}
    v_k & \leq \frac{1}{2} \sum_{t=1}^k \lambda_t+ \left ( S_k + \left(\frac{1}{2} \sum_{t=1}^k \lambda_t\right)^2\right)^{1/2}.    
\end{align*}
Now we show that the conclusion holds for $\tau= k+1$. Define $\psi_k :=\sum_{t=1}^k \lambda_t$. Note that $\psi_{k+1} \geq \psi_k$ since $\lambda_t \geq 0$ for all $t$. Hence, $v_k \leq \frac{1}{2} \psi_k+ \left ( S_k + \left(\frac{1}{2} \psi_k \right)^2\right)^{1/2}$. Define $k^* := \argmax_{\tau \in \{0, \ldots, k\}} v_\tau$. Using the main assumption for $\tau= k+1$,
 \begin{align*}
      v_{k+1}^2 &\leq S_{k+1} + \sum_{t=1}^{k+1} \lambda_t v_t \\
       \implies v_{k+1}^2 - \lambda_{k+1} v_{k+1}&\leq S_{k+1} + \sum_{t=1}^{k} \lambda_t v_t\\
      \implies \left(v_{k+1} - \frac{\lambda_{k+1}}{2} \right)^2 &\leq S_{k+1} + \frac{\lambda_{k+1}^2}{4} + \sum_{t=1}^{k} \lambda_t v_t\\
      & \leq S_{k+1}+\frac{\lambda_{k+1}^2}{4}  + v_{k^*}\sum_{t=1}^{k} \lambda_t \tag{since $ v_{k^*}$ is the maximum} \\
      & = S_{k+1}+\frac{\lambda_{k+1}^2}{4}   + v_{k^*}\psi_k \tag{based on the definition for $\psi_k$}
\end{align*}      
Since $k^*\leq k$, by the inductive hypothesis, 
\begin{align*}
\left(v_{k+1} - \frac{\lambda_{k+1}}{2} \right)^2 & \leq S_{k+1}+\frac{\lambda_{k+1}^2}{4}   + \psi_k \left\{ \frac{1}{2} \psi_{k^*}+ \left ( S_{k^*} + \left(\frac{1}{2} \psi_{k^*}\right)^2\right)^{1/2}\right\} \\
      & = S_{k+1}+\frac{\lambda_{k+1}^2}{4} + \frac{1}{2}  \psi_k \psi_{k^*} + \psi_k \left ( S_{k^*} + \left(\frac{1}{2} \psi_{k^*}\right)^2\right)^{1/2}\\
      & \leq  S_{k+1}+\frac{\lambda_{k+1}^2}{4} + \frac{1}{2} \psi_{k}^2 +  \psi_k \left ( S_{k^*} + \left(\frac{1}{2} \psi_{k^*}\right)^2\right)^{1/2} \tag{Since $\{\psi_\tau\}$ is non-decreasing and $k^* \leq k$}\\
\end{align*}      
Furthermore, since $\{S_\tau\}$ is increasing and $k^* < k+1$, 
\begin{align*}
\left(v_{k+1} - \frac{\lambda_{k+1}}{2} \right)^2 & \leq  S_{k+1}+\frac{\lambda_{k+1}^2}{4} + \frac{1}{2} \psi_{k}^2 +  \psi_k \left ( S_{k+1} + \left(\frac{1}{2} \psi_{k^*}\right)^2\right)^{1/2} \\
        & \leq S_{k+1}+\frac{\lambda_{k+1}^2}{4} + \frac{1}{2} \psi_{k}^2 +  \psi_k \left ( S_{k+1} + \left(\frac{1}{2} \psi_{k+1}\right)^2\right)^{1/2}\tag{Since $\{\psi_\tau\}$ is non-decreasing and $k^* < k+1$}\\
        & = S_{k+1}+\frac{\lambda_{k+1}^2}{4} + \frac{1}{4} \psi_{k}^2 +  \psi_k \left ( S_{k+1} + \left(\frac{1}{2} \psi_{k+1}\right)^2\right)^{1/2}+ \frac{1}{4} \psi_{k}^2    \\
        &\leq  S_{k+1}+ \frac{1}{4} \psi_{k+1}^2 +  \psi_k \left ( S_{k+1} + \left(\frac{1}{2} \psi_{k+1}\right)^2\right)^{1/2}+ \frac{1}{4} \psi_{k}^2\tag{$a^2 + b^2 \leq (a+b)^2$ for $a = \psi_{k} > 0$ and $b = \lambda_{k+1} > 0$}\\
        & =  \underbrace{S_{k+1}+ \frac{1}{4} \psi_{k+1}^2 }_{:=x^2}+  \underbrace{\psi_k \left ( S_{k+1} + \left(\frac{1}{2} \psi_{k+1}\right)^2\right)^{1/2}}_{=2xy}+ \underbrace{\frac{1}{4} \psi_{k}^2}_{:=y^2}\\
         & = \left( \left ( S_{k+1} + \left(\frac{1}{2} \psi_{k+1}\right)^2\right)^{1/2} + \frac{1}{2} \psi_k \right)^2\\
\implies v_{k+1} - \frac{\lambda_{k+1}}{2} &\leq \left ( S_{k+1} + \left(\frac{1}{2} \psi_{k+1}\right)^2\right)^{1/2} + \frac{1}{2} \psi_k \\
\implies v_{k+1} &\leq \left ( S_{k+1} + \left(\frac{1}{2} \psi_{k+1}\right)^2\right)^{1/2} + \frac{1}{2} \psi_k + \frac{\lambda_{k+1}}{2} \\
&= \left ( S_{k+1} + \left(\frac{1}{2} \psi_{k+1}\right)^2\right)^{1/2} + \frac{1}{2} \psi_{k+1}
\end{align*}
where replacing $\psi_{k+1}$ with its definition gives us the required result.
\end{proof}


%% file: Appendix/Experiments.tex
\newpage
\section{Additional Experimental Results}
\label{app:experiments}
\textbf{Supervised Learning:} We evaluate our framework on the LibSVM benchmarks~\citep{chang2011libsvm}, a standard suite of convex-optimization problems. Here consider two datasets -- \texttt{mushrooms}, and \texttt{rcv1}, two losses -- squared loss and logistic loss, and four batch sizes -- \{25, 125, 625, full-batch\} for the linear parameterization ($f = X^\top \theta$). Each optimization algorithm is run for $500$ epochs (full passes over the data).  

\subsection{Stochastic Surrogate Optimization}
\label{app:experiments:sso}
Comparisons of \texttt{SGD}, \texttt{SLS}, \texttt{Adam}, \texttt{Adagrad}, and \texttt{SSO} evaluated on three SVMLib benchmarks \texttt{mushrooms}, \texttt{ijcnn}, and \texttt{rcv1}, two losses -- squared loss and logistic loss, and four batch sizes -- \{25, 125, 625, full-batch\}. Each run was evaluated over three random seeds following the same initialization scheme. All plots are in log-log space to make trends between optimization algorithms more apparent. All algorithms and batch sizes are evaluated for 500 epochs and performance is represented as a function of total optimization steps. We compare stochastic surrogate optimization (\texttt{SSO}), against \texttt{SGD} with the standard theoretical $\nicefrac{1}{2 L_\theta}$ step-size, \texttt{SGD} with the step-size set according to a stochastic line-search~\citep{vaswani2019painless} \texttt{SLS}, and finally \texttt{Adam}~\citep{kingma2014adam} using default hyper-parameters. Since \texttt{SSO} is equivalent to projected SGD in the target space, we set $\eta$ (in the surrogate definition) to $\nicefrac{1}{2 L}$ where $L$ is the smoothness of $\ell$ w.r.t $\z$. For squared loss, $L$ is therefore set to $1$, while for logistic it is set to $2$. These figures show (i) \texttt{SSO} improves over \texttt{SGD} when the step-sizes are set theoretically, (ii) \texttt{SSO} is competitive with \texttt{SLS} or \texttt{Adam}, and (iii) as $m$ increases, on average, the performance of \texttt{SSO} improves as projection error decreases. For further details see the attached code repository. Below we include three different step-size schedules: constant, $\frac{1}{\sqrt{t}}$~\cite{orabona2019modern}, and $(1/T)^{t/T}$~\citep{vaswani2022towards}. 
\begin{figure}[ht]
    \centering
    \includegraphics[width=0.9\textwidth]{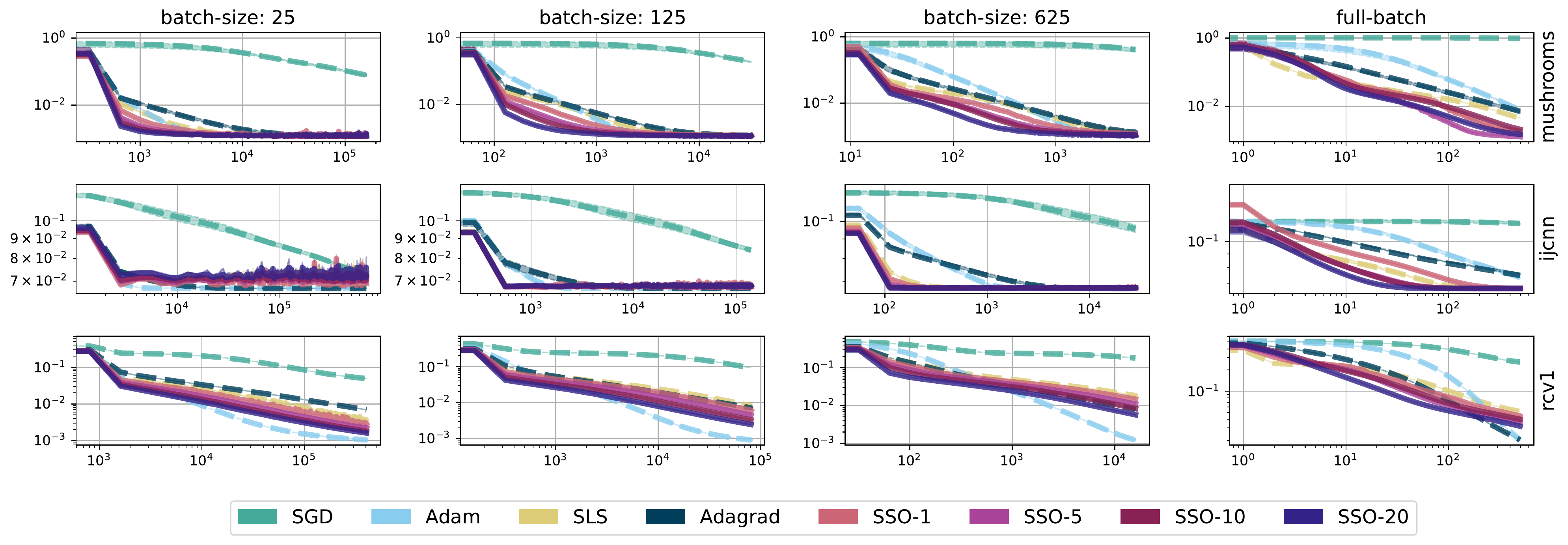}
    \caption{\textbf{Constant step-size:} comparison of optimization algorithms under a \textbf{mean squared error loss}. We note, \texttt{SSO} significantly outperforms its parametric counterpart, and maintains performance which is on par with both \texttt{SLS} and \texttt{Adam}. Additionally we note that taking additional steps in the surrogate generally improves performance.}
    \label{app:fig:sgd_const_mse}
\end{figure}
\begin{figure}[ht]
    \centering
    \includegraphics[width=0.9\textwidth]{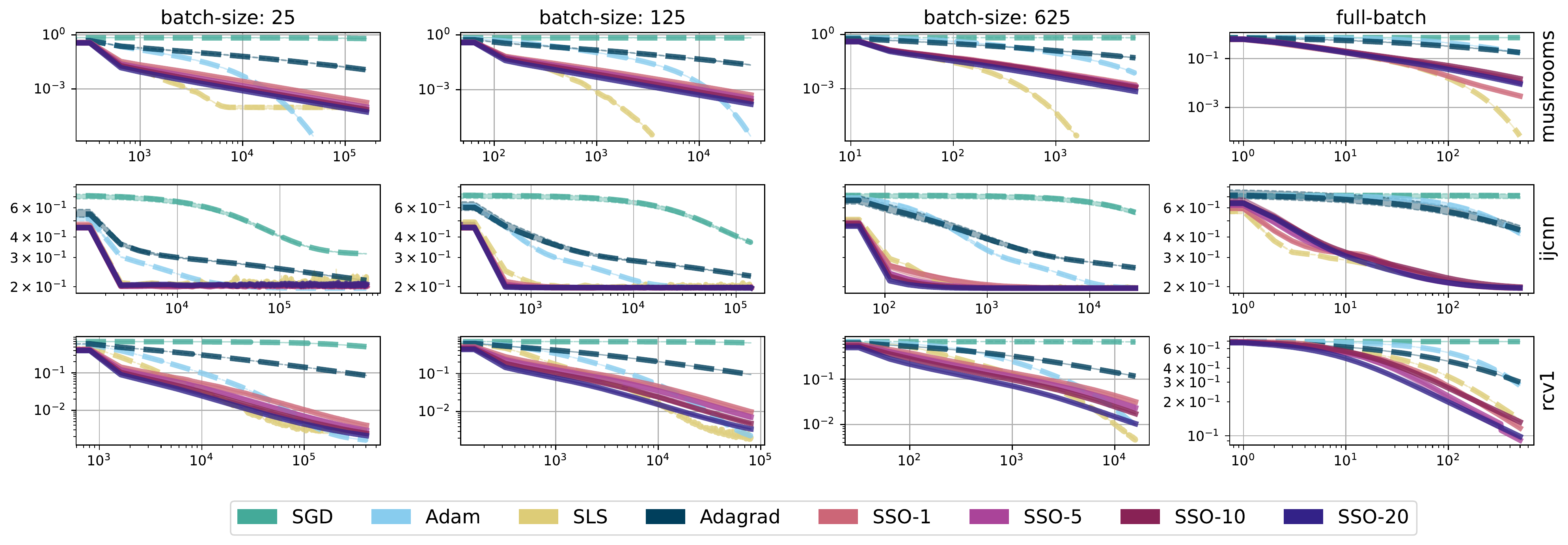}
    \caption{\textbf{Constant step-size:} comparison of optimization algorithms under a average \textbf{logistic loss}. We note, \texttt{SSO} significantly outperforms its parametric counterpart, and maintains performance which is on par with both \texttt{SLS} and \texttt{Adam}. Additionally we note that taking additional steps in the surrogate generally improves performance}
    \label{app:fig:sgd_const_logistic}
\end{figure}
\newpage 
\begin{figure}[ht]
    \centering
    \includegraphics[width=0.9\textwidth]{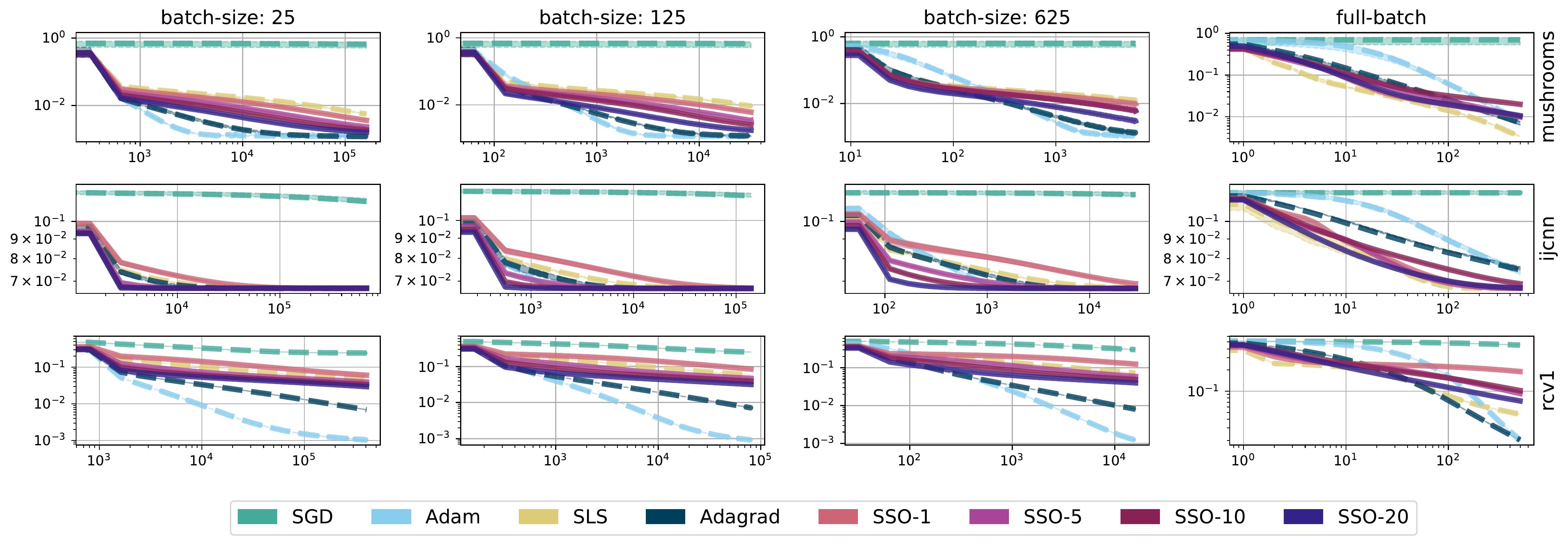}
    \caption{\textbf{Decreasing step-size:} comparison of optimization algorithms under a \textbf{mean squared error loss}. We compare examples which include a decaying step size of $\frac{1}{\sqrt{t}}$ alongside both \texttt{SSO} as well as \texttt{SGD} and \texttt{SLS}. \texttt{Adam} (and \texttt{Adagrad}). Again, we note that taking additional steps in the surrogate generally improves performance. Additionally the decreasing step-size seems to help maintain strict monotonic improvement.}
    \label{app:fig:sgd_stoch_mse}
\end{figure}
\begin{figure}[ht]
    \centering
    \includegraphics[width=0.9\textwidth]{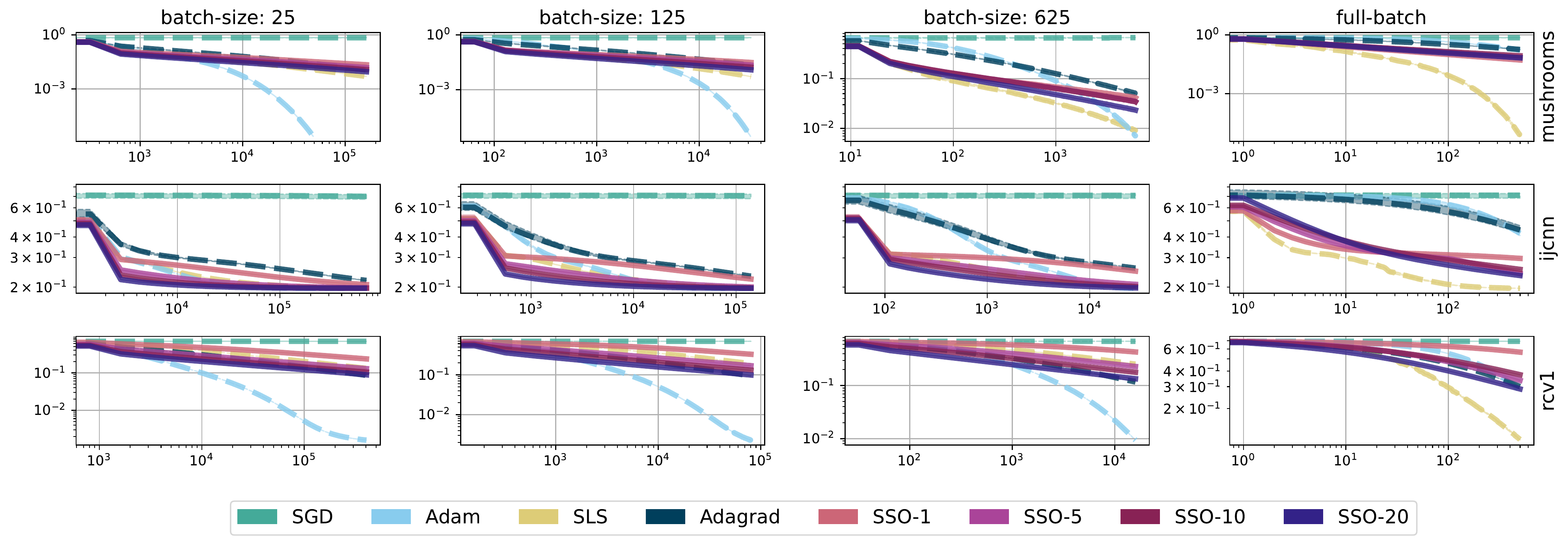}
     \caption{\textbf{Decreasing step-size:} comparison of optimization algorithms under a \textbf{logistic loss}. We compare examples which include a decaying step size of $\frac{1}{\sqrt{t}}$ alongside both \texttt{SSO} as well as \texttt{SGD} and \texttt{SLS}. \texttt{Adam} (and \texttt{Adagrad}). Again, we note that taking additional steps in the surrogate generally improves performance. Additionally the decreasing step-size seems to help maintain strict monotonic improvement.}
    \label{app:fig:sgd_stoch_logistic}
\end{figure}
\newpage 
\begin{figure}[ht]
    \centering
    \includegraphics[width=0.9\textwidth]{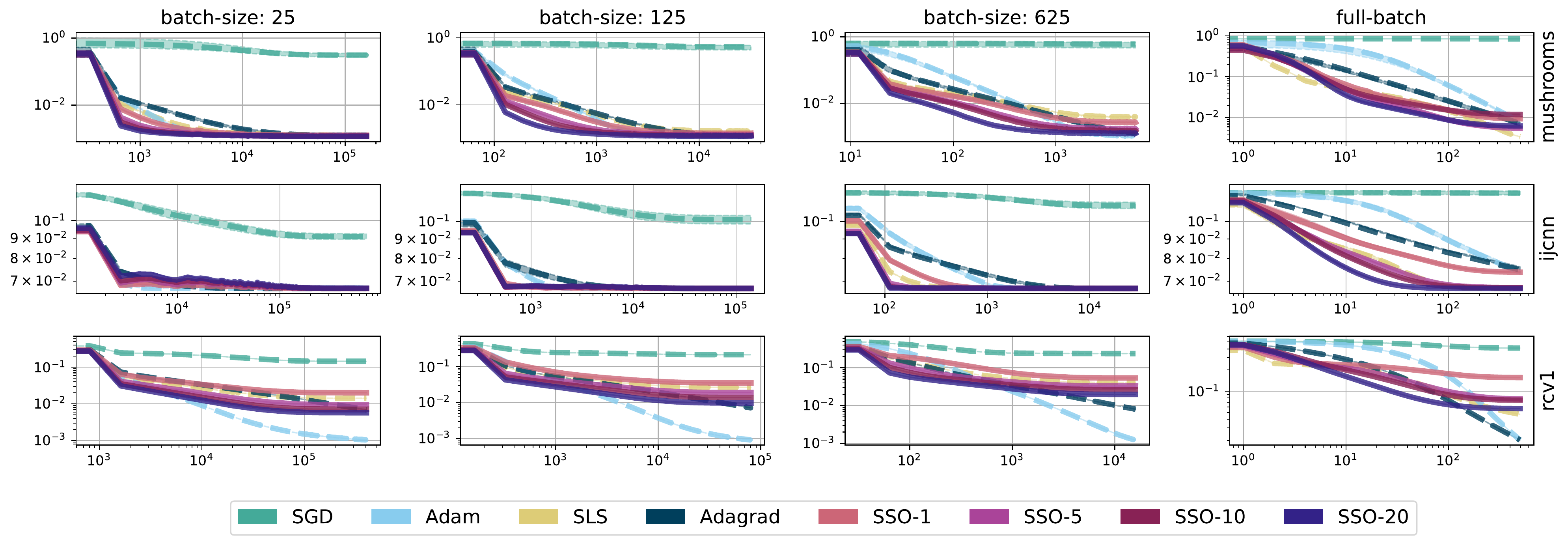}
    \caption{\textbf{Exponential step-size:} comparison of optimization algorithms under a \textbf{mean squared error loss}. We compare examples which include a decaying step size of $(\frac{1}{T})^{t/T}$ alongside both \texttt{SSO} as well as \texttt{SGD} and \texttt{SLS}. \texttt{Adam}. Again, we note that taking additional steps in the surrogate generally improves performance. Additionally the decreasing step-size seems to help maintain strict monotonic improvement. Lastly, because of a less aggressive step size decay, the optimization algorithms make more progress then their stochastic $\frac{1}{\sqrt{t}}$ counterparts.}
    \label{app:fig:sgd_expo_mse}
\end{figure}
\begin{figure}[ht]
    \centering
    \includegraphics[width=0.9\textwidth]{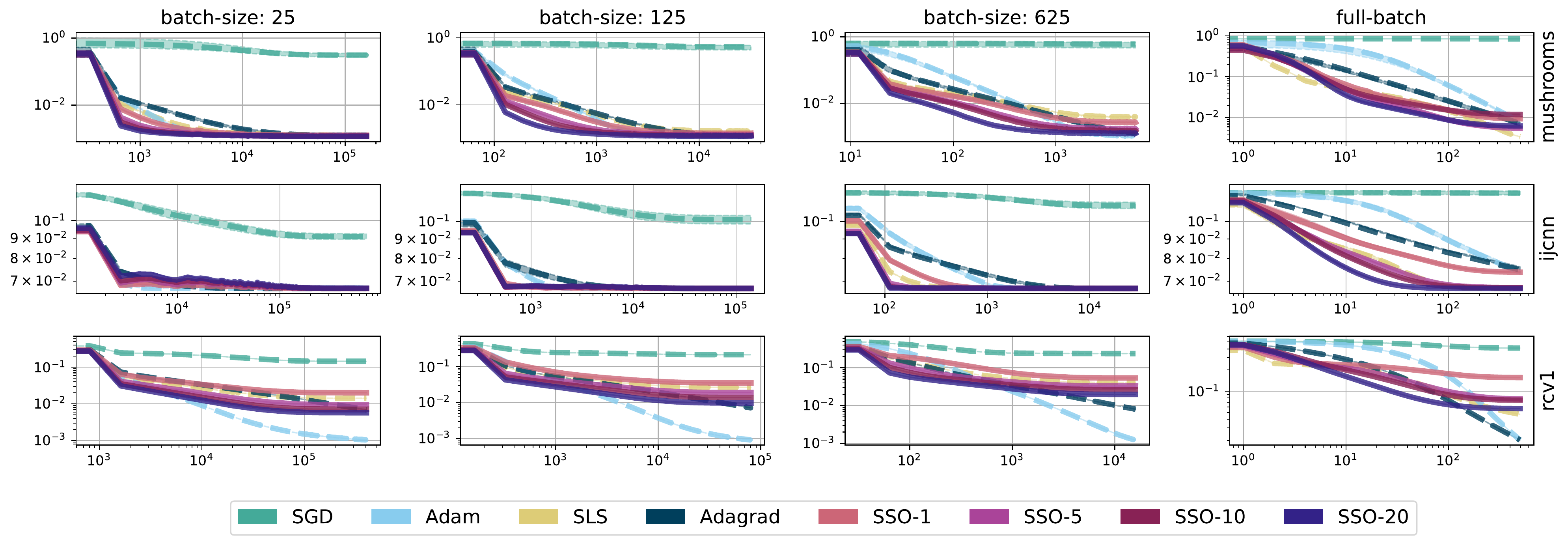}
    \caption{\textbf{Exponential step-size:} comparison of optimization algorithms under a \textbf{logistic loss}. We compare examples which include a decaying step size of $(\frac{1}{T})^{t/T}$ alongside both \texttt{SSO} as well as \texttt{SGD} and \texttt{SLS}. \texttt{Adam} remains the same as aboves. Again, we note that taking additional steps in the surrogate generally improves performance. Additionally the decreasing step-size seems to help maintain strict monotonic improvement. Lastly, because of a less aggressive step size decay, the optimization algorithms make more progress then their stochastic $\frac{1}{\sqrt{t}}$ counterparts.}
    \label{app:fig:sgd_expo_logistic}
\end{figure}
\newpage

\subsection{Stochastic Surrogate Optimization with a Line-search}
\label{app:experiments:sso:sls}
Comparisons of \texttt{SGD}, \texttt{SLS}, \texttt{Adam}, \texttt{Adagrad}, and \texttt{SSO-SLS} evaluated on three SVMLib benchmarks \texttt{mushrooms}, \texttt{ijcnn}, and \texttt{rcv1}. Each run was evaluated over three random seeds following the same initialization scheme. All plots are in log-log space to make trends between optimization algorithms more apparent. As before in all settings, algorithms use either their theoretical step-size when available, or the default as defined by \cite{paszke2017automatic}. The inner-optimization loop are set according to line-search parameters and heuristics following \citet{vaswani2019painless}. All algorithms and batch sizes are evaluated for 500 epochs and performance is represented as a function of total optimization steps. 
\begin{figure}[ht]
    \centering
    \includegraphics[width=0.9\textwidth]{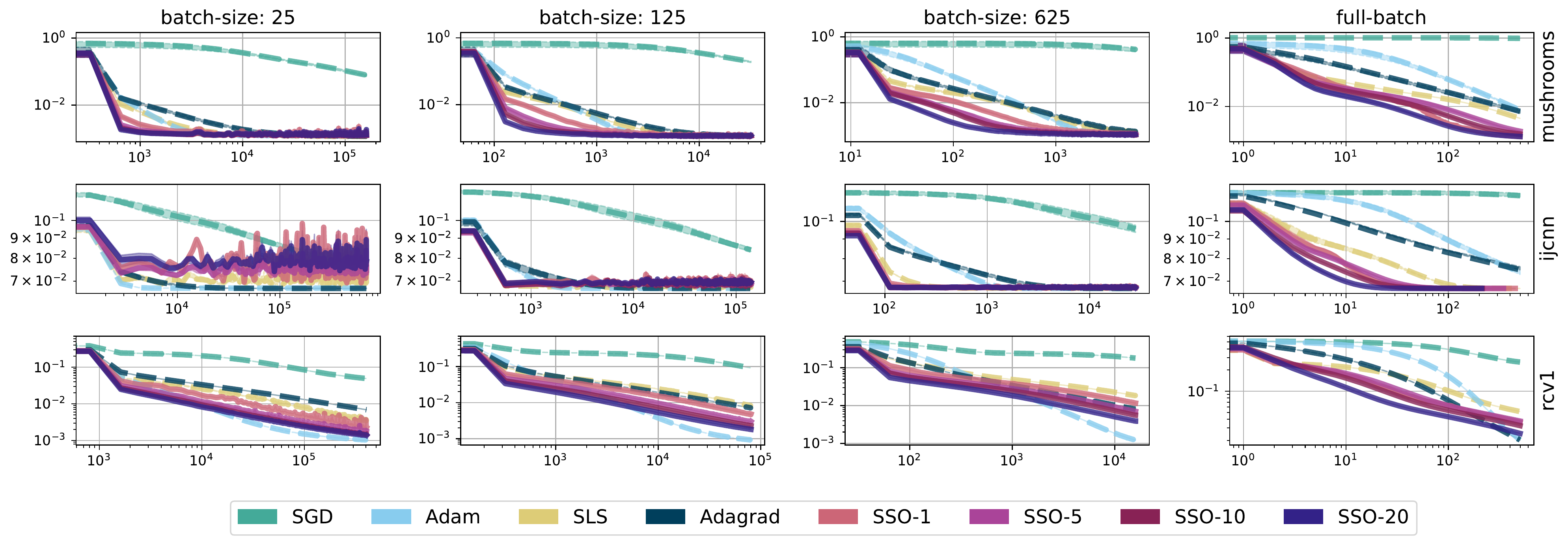}
    \caption{\textbf{Constant step-size:} comparison of optimization algorithms under a \textbf{mean squared error loss}. We note, \texttt{SSO-SLS} outperforms its parametric counterpart, and maintains performance which is on par with both \texttt{SLS} and \texttt{Adam}. Additionally we note that taking additional steps in the surrogate generally improves performance, especially in settings with less noise (full-batch and batch-size 625).}
    \label{app:fig:sls_const_mse}
\end{figure}
\begin{figure}[ht]
    \centering
    \includegraphics[width=0.9\textwidth]{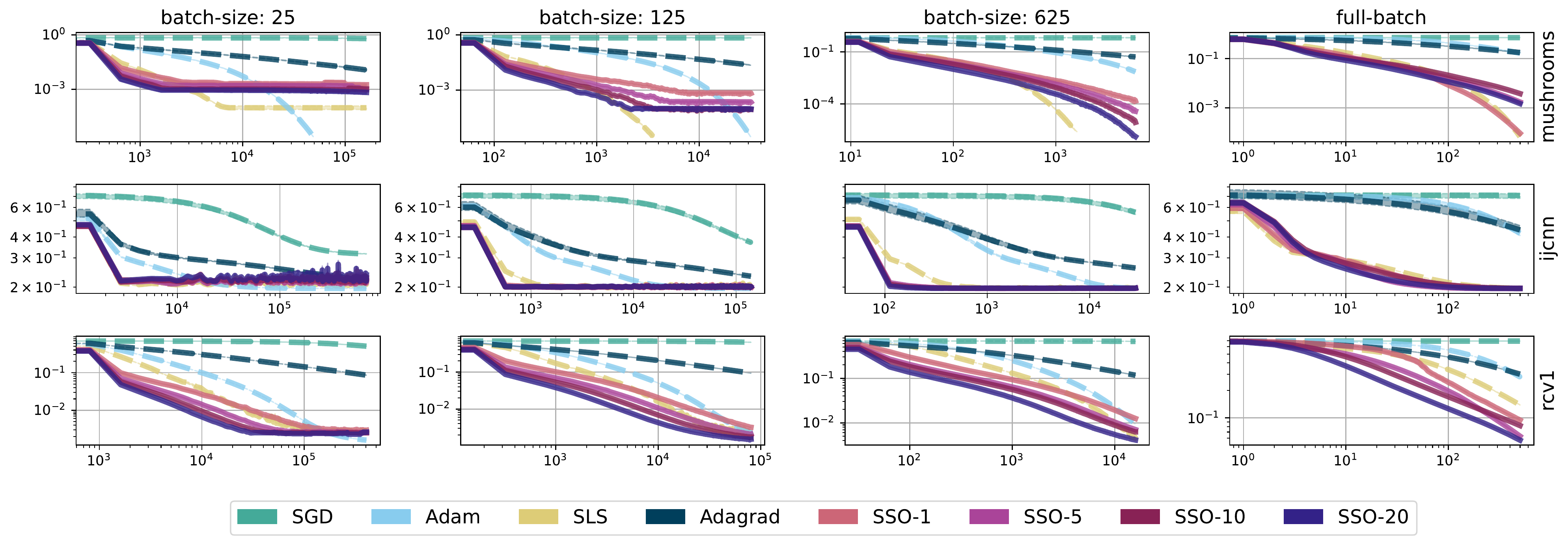}
    \caption{\textbf{Constant step-size:} comparison of optimization algorithms under a average \textbf{logistic loss}. We note, \texttt{SSO-SLS} outperforms its parametric counterpart, and maintains performance which is on par with both \texttt{SLS} and \texttt{Adam}. Additionally we note that taking additional steps in the surrogate generally improves performance.}
    \label{app:fig:sls_const_logistic}
\end{figure}

\newpage  
 
\subsection{Combining Stochastic Surrogate Optimization with Adaptive Gradient Methods}

Comparisons of \texttt{SGD}, \texttt{SLS}, \texttt{Adam}, \texttt{Adagrad}, and \texttt{SSO-Adagrad} evaluated on three SVMLib benchmarks \texttt{mushrooms}, \texttt{ijcnn}, and \texttt{rcv1}. Each run was evaluated over three random seeds following the same initialization scheme. All plots are in log-log space to make trends between optimization algorithms more apparent. As before in all settings, algorithms use either their theoretical step-size when available, or the default as defined by \cite{paszke2017automatic}. The inner-optimization loop are set according to line-search parameters and heuristics following \citet{vaswani2019painless}. All algorithms and batch sizes are evaluated for 500 epochs and performance is represented as a function of total optimization steps. Here we update the $\eta$ according to the same schedule as scalar \texttt{Adagrad} (termed AdaGrad-Norm in~\citet{ward2020adagrad}). Because Adagrad does not have an easy to compute optimal theoretical step size, for our setting we set the log learning rate (the negative of log $\eta$) to be $2.$. For further details see the attached coding repository.

\begin{figure}[ht]
    \centering
    \includegraphics[width=0.9\textwidth]{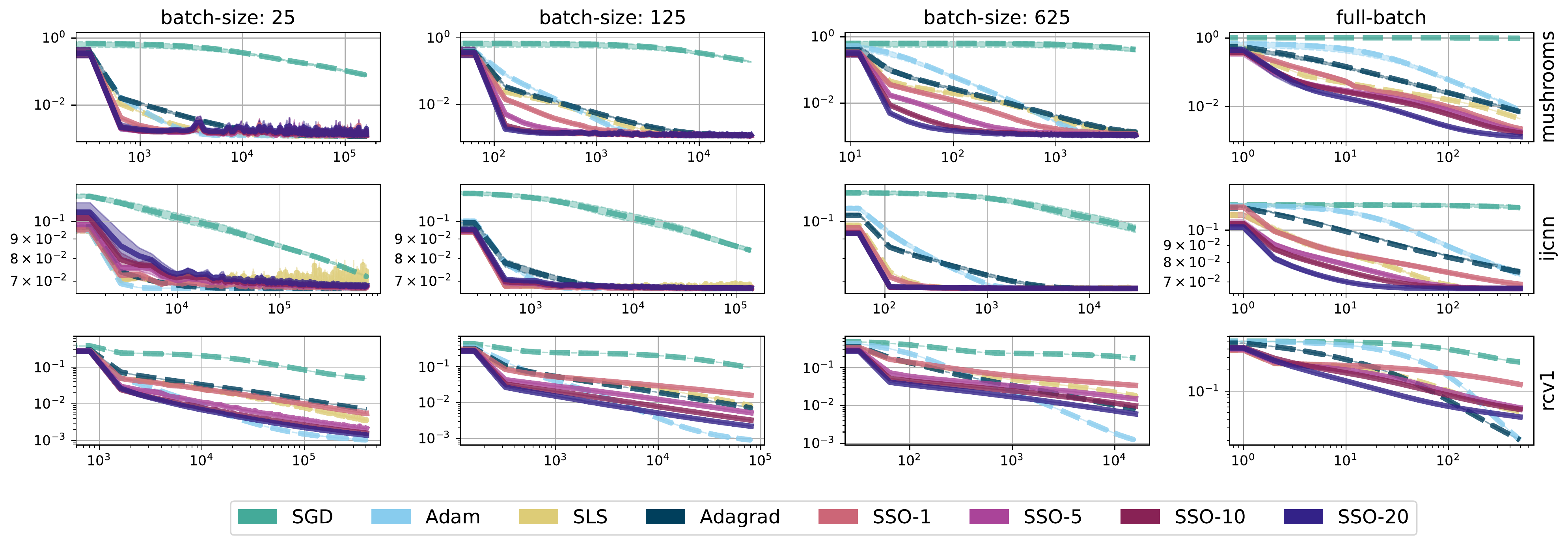}
    \caption{Comparison in terms of average MSE loss of \texttt{SGD}, \texttt{SLS}, \texttt{Adam}, and \texttt{SSO-Adagrad} evaluated under a \textbf{mean squared error loss}. These plots show that \texttt{SSO-Adagrad} outperforms its parametric counterpart, and maintains performance which is on par with both \texttt{SLS} and \texttt{Adam}. Additionally, we again find that taking additional steps in the surrogate generally improves performance.}
    \label{app:fig:ada_const_mse}
\end{figure}
\begin{figure}[ht]
    \centering
    \includegraphics[width=0.9\textwidth]{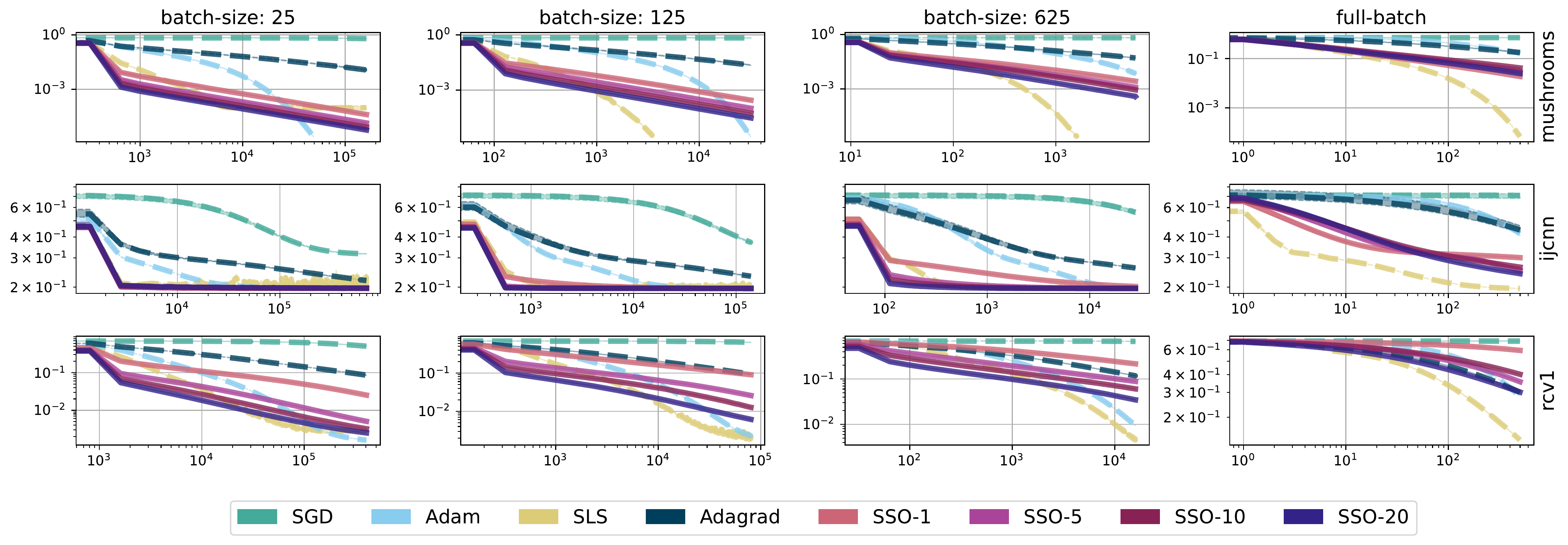}
    \caption{Comparison in terms of average MSE loss of \texttt{SGD}, \texttt{SLS}, \texttt{Adam}, and \texttt{SSO-Adagrad} evaluated under a \textbf{logistic loss}. These plots show that \texttt{SSO-Adagrad} outperforms its parametric counterpart, and maintains performance which is on par with both \texttt{SLS} and \texttt{Adam}. Additionally, we again find that taking additional steps in the surrogate generally improves performance. }
    \label{app:fig:ada_const_logistic}
\end{figure}

\newpage 
 
\subsection{Combining Stochastic Surrogate Optimization With Online Newton Steps}

Comparisons of \texttt{SGD}, \texttt{SLS}, \texttt{Adam}, \texttt{Adagrad}, and \texttt{SSO-Newton} evaluated on three SVMLib benchmarks \texttt{mushrooms}, \texttt{ijcnn}, and \texttt{rcv1}. Each run was evaluated over three random seeds following the same initialization scheme. All plots are in log-log space to make trends between optimization algorithms more apparent. As before, in all settings, algorithms use either their theoretical step-size when available, or the default as defined by \cite{paszke2017automatic}. The inner-optimization loop are set according to line-search parameters and heuristics following \citet{vaswani2019painless}. All algorithms and batch sizes are evaluated for 500 epochs and performance is represented as a function of total optimization steps. Here we update the $\eta$ according to the same schedule as \texttt{Online Newton}. We omit the MSE example as \texttt{SSO-Newton} in this setting is equivalent to \texttt{SSO}. In the logistic loss setting however, which is displayed below, we re-scale the regularization term by $(1-p)p$ where $p = \sigma{(f(x))}$ where $\sigma$ is this sigmoid function, and $f$ is the target space. This operation is done per-data point, and as can be seen below, often leads to extremely good performance, even in the stochastic setting. In the plots below, if the line vanishes before the maximum number of optimization steps have occurred, this indicates that the algorithm has converged to the minimum and is no longer executed. Notably, \texttt{SSO-Newton} achieves this in for multiple data-sets and batch sizes. 

\begin{figure}[ht]
    \centering
    \includegraphics[width=0.9\textwidth]{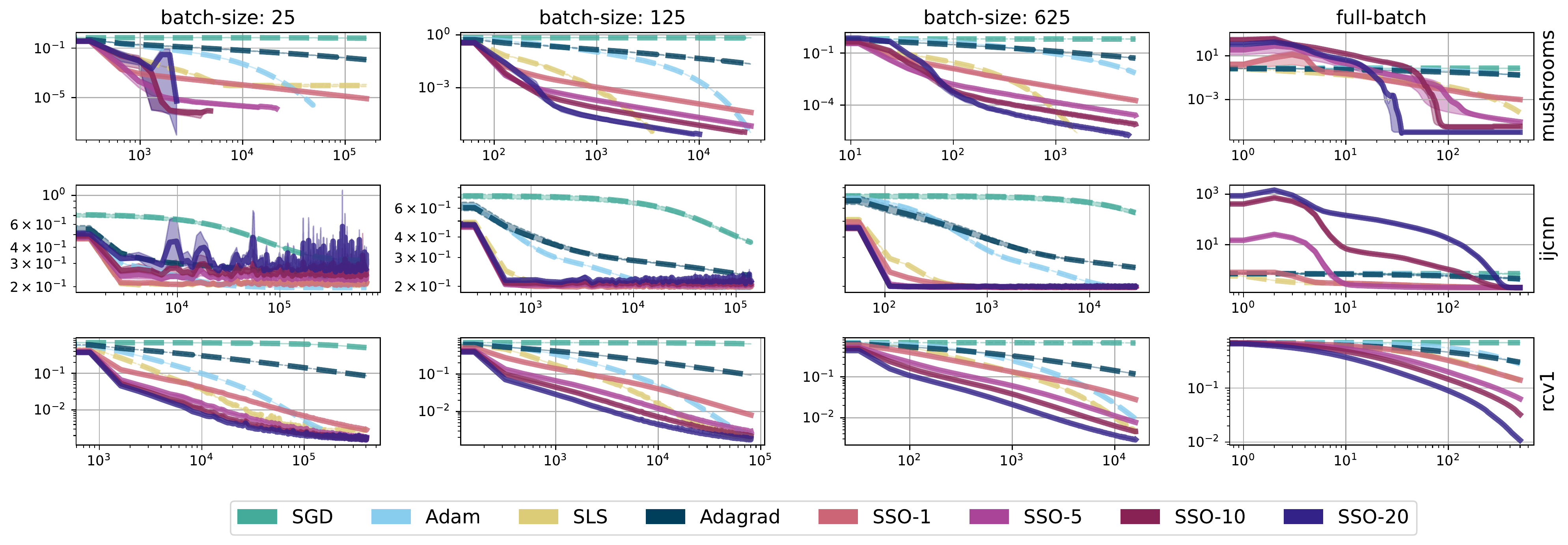}
    \caption{Comparison in terms of average logistic loss of \texttt{SGD}, \texttt{SLS}, \texttt{Adam}, and \texttt{SSO-Newton} evaluated on the \textbf{logistic loss}. This plot displays that significant improvement can be made at no additional cost by re-scaling the regularization term correctly. Note that in the case of mushrooms, \texttt{SSO-Newton} in all cases for $m=20$ reaches the stopping criteria before the 500th epoch. Second, even in many stochastic settings, \texttt{SSO-Newton} outperforms both \texttt{SLS} and \texttt{Adam}.}
    \label{app:fig:newton_const_logistic}
\end{figure}

\newpage
\subsection{Comparison with SVRG}
Below we include a simple supervised learning example in which we compare the standard variance reduction algorithm SVRG~\cite{johnson2013accelerating} to the $SSO$ algorithm in the supervised learning setting. Like the examples above, we test the optimization algorithms using the \citet{chang2011libsvm} rcv1 dataset, under an MSE loss. This plot shows that SSO can outperform SVRG for even small batch-size settings. For additional details on the implementation and hyper-parameters used, see \url{https://github.com/WilderLavington/Target-Based-Surrogates-For-Stochastic-Optimization}. We note that it is likely, that for small enough batch sizes SVRG may become competitive with SSO, however we leave such a comparison for future work.

\begin{figure}[h]
    \centering
    \includegraphics[width=0.995\textwidth]{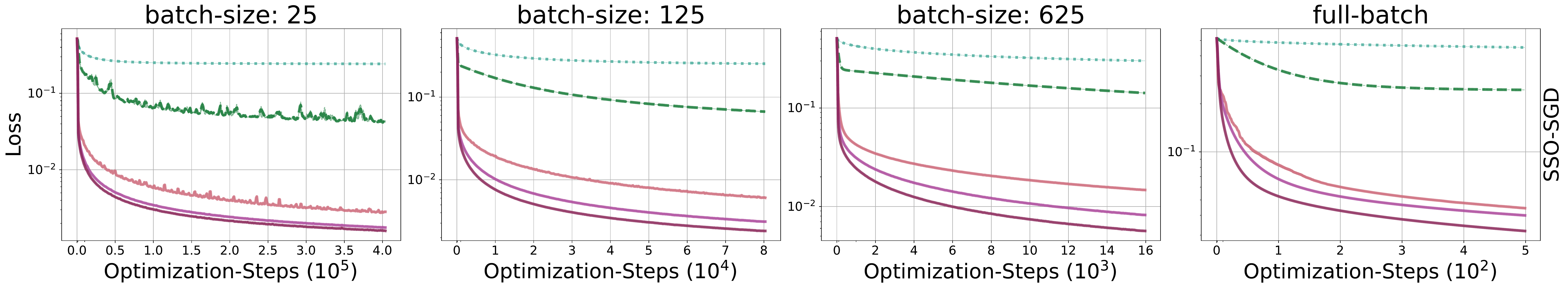}
    \includegraphics[width=0.715\textwidth]{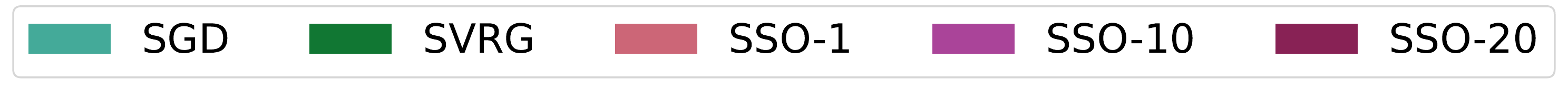}
    \caption{Comparison of the average mean-squared error between \texttt{SGD}, \texttt{SVRG}, and \texttt{SSO-SGD}. This plot displays that even in settings where the batch-size is small, \texttt{SSO} can even improve over variance reduction techniques like SVRG.}
    \label{app:fig:svrg}
\end{figure}

\clearpage

\subsection{Imitation Learning}
\begin{figure}[!ht]
\centering 
    \begin{subfigure}{.4325\linewidth}
    \centering
         \includegraphics[width=\textwidth]{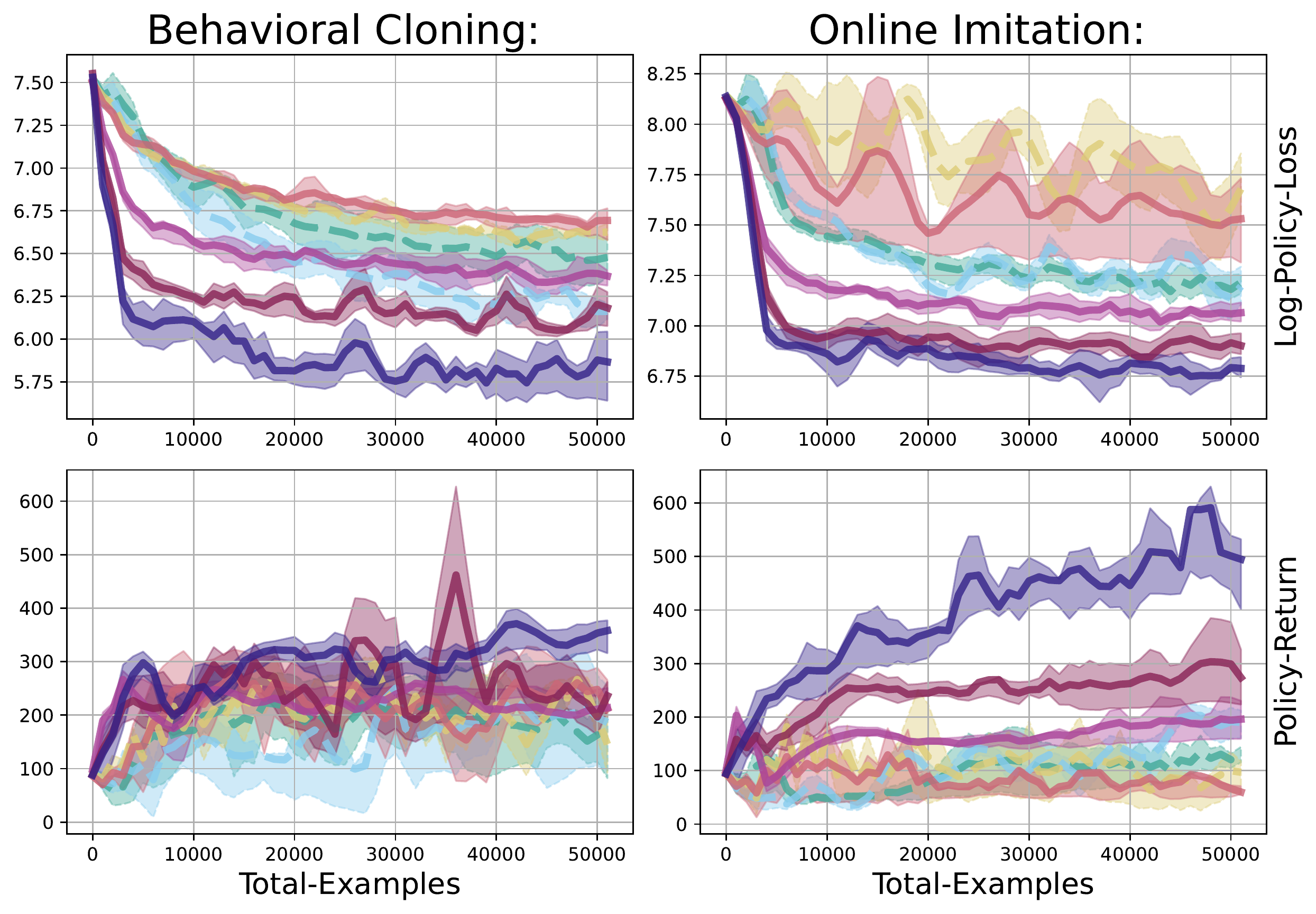}
        \caption{Walker2d-v2 Environment}
    \end{subfigure}
    \begin{subfigure}{.56\linewidth}
    \centering
        \includegraphics[width=\textwidth]{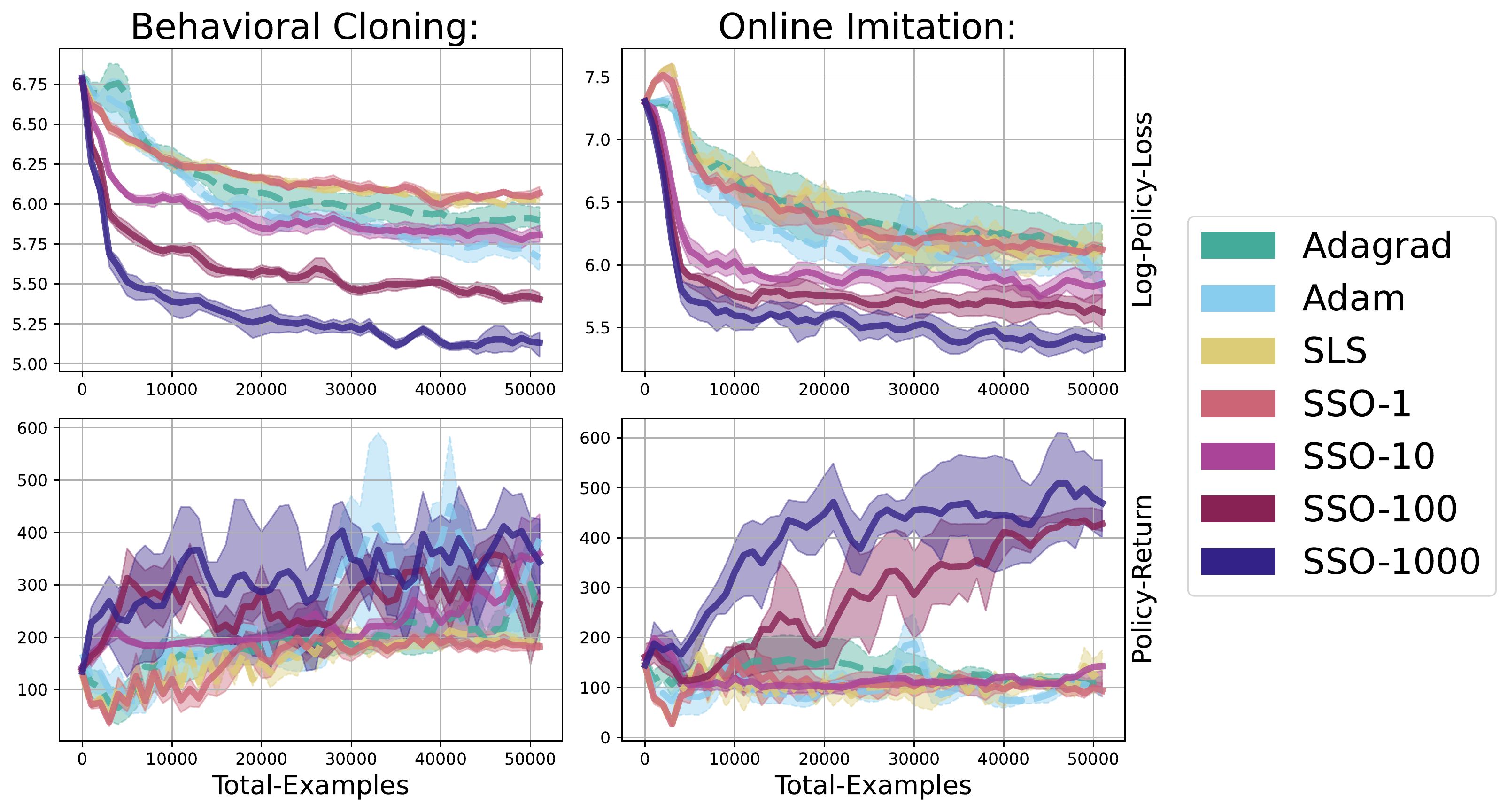}
        \caption{Hopper-v2 Environment}
     \end{subfigure} 
    \caption{Comparison of policy return, and log policy loss incurred by \texttt{SGD}, \texttt{SLS}, \texttt{Adam}, \texttt{Adagrad}, and \texttt{SSO} as a function of the total interactions. Unlike \cref{sec:experiments}, the mean of the policy is parameterized by a \textbf{neural network model}. In both environments, for both behavioral policies, \texttt{SSO} outperforms all other online-optimization algorithms. Additionally, as m in increases, so to does the performance of \texttt{SSO} in terms of both the return as well as the loss.}  
    \label{fig:oil_fig-app-nn}
\end{figure}

\begin{figure}[!ht]
\centering 
    \begin{subfigure}{.4325\linewidth}
    \centering
         \includegraphics[width=\textwidth]{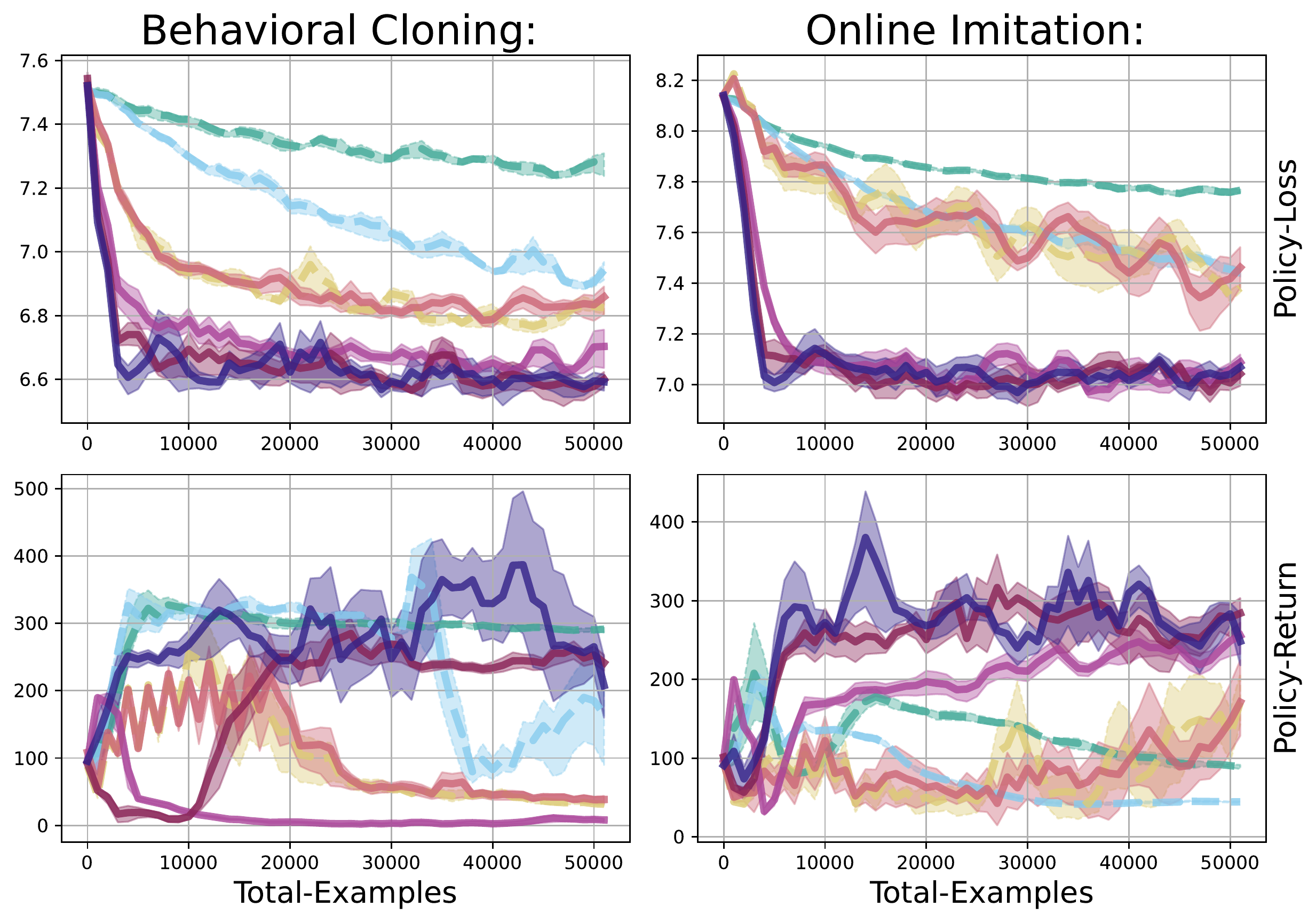}
        \caption{Walker2d-v2 Environment}
    \end{subfigure}
    \begin{subfigure}{.56\linewidth}
    \centering
        \includegraphics[width=\textwidth]{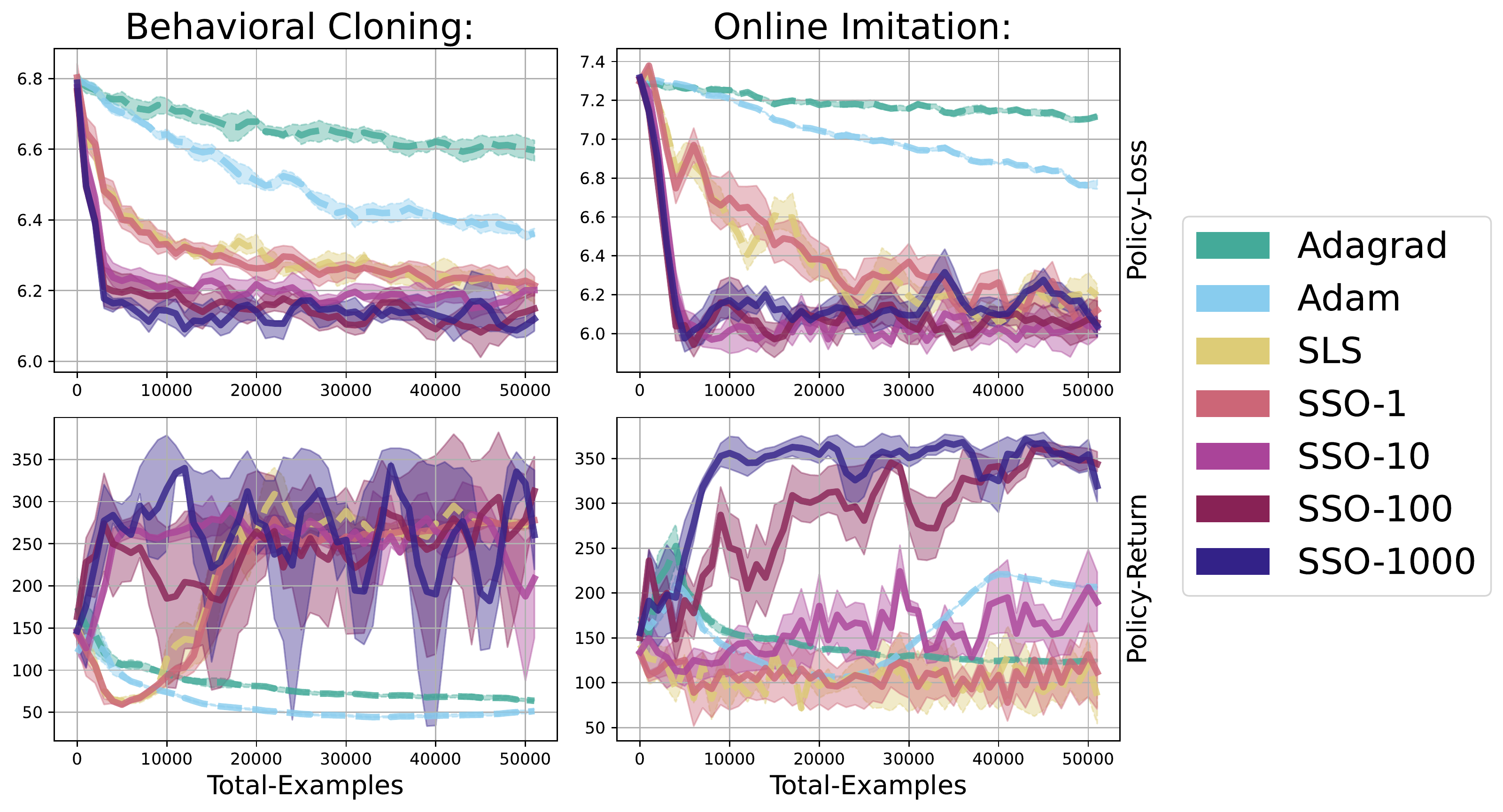}
        \caption{Hopper-v2 Environment}
     \end{subfigure} 
    \caption{Comparison of policy return, and log policy loss incurred by \texttt{SGD}, \texttt{SLS}, \texttt{Adam}, \texttt{Adagrad}, and \texttt{SSO} as a function of the total interactions. Unlike \cref{sec:experiments}, the mean of the policy is parameterized by a \textbf{linear model}. In both environments, for both behavioral policies, \texttt{SSO} outperforms all other online-optimization algorithms. Additionally, as m in increases, so to does the performance of \texttt{SSO} in terms of both the return as well as the loss.}  
    \label{fig:oil_fig-app-linear}
\end{figure}

\begin{figure}[!ht]
\centering 
    \includegraphics[width=0.85\textwidth]{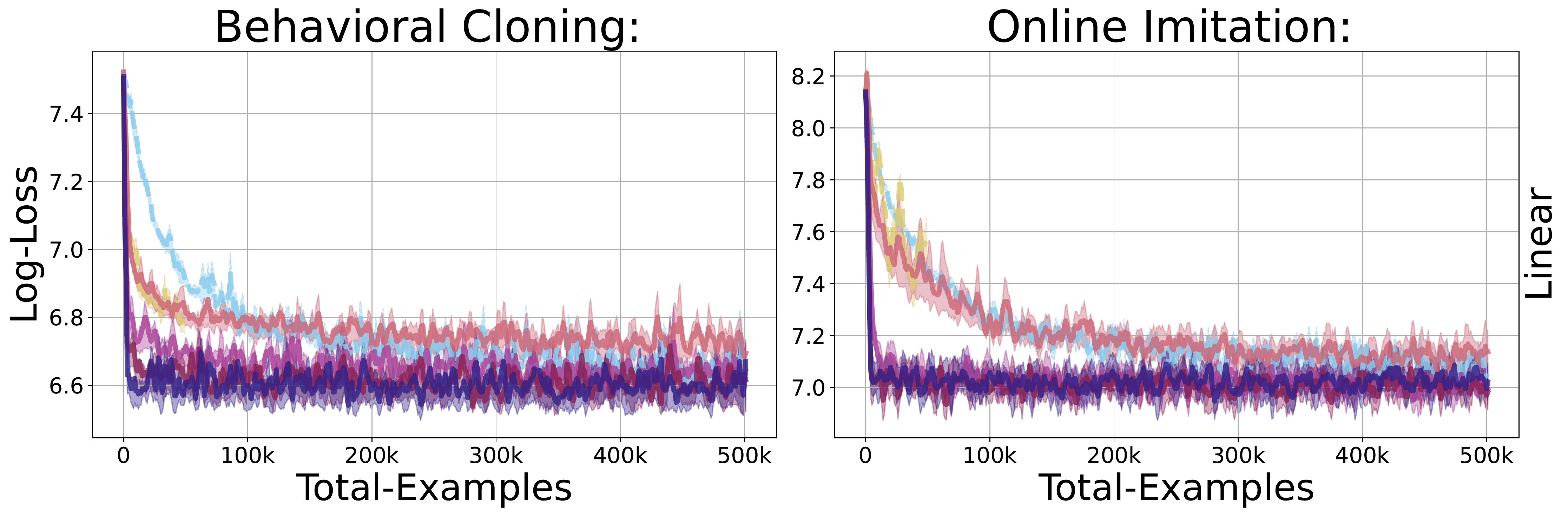} 
    \includegraphics[width=\textwidth]{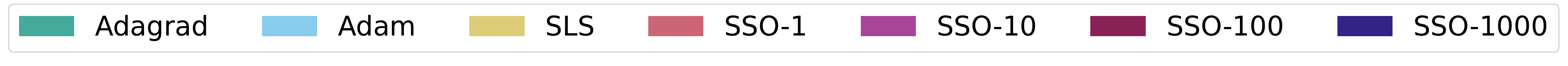}
    \caption{Comparison of log policy loss incurred by \texttt{SGD}, \texttt{SLS}, \texttt{Adam}, \texttt{Adagrad}, and \texttt{SSO} as a function of the total interactions for the Walker2d gym environment. Unlike \cref{sec:experiments}, the mean of the policy is parameterized by a \textbf{linear model}. Unlike \cref{fig:oil_fig-app-linear} and \cref{fig:oil_fig-app-nn},this plot displays 500 thousand iterations instead of only 50. We again see that \texttt{SSO} outperforms all other online-optimization algorithms, \textit{even} for a very large number of iterations. Additionally, as m in increases, we again see the performance of \texttt{SSO} log loss improves as well.}  
    \label{fig:oil_fig-app-linear-cr}
\end{figure}

Comparisons of \texttt{Adagrad}, \texttt{SLS}, \texttt{Adam}, and \texttt{SSO} evaluated on two Mujoco~\citep{todorov2012mujoco} imitation learning benchmarks~\citep{lavington2022improved}, \texttt{Hopper-v2}, and \texttt{Walker-v2}. In this setting training and evaluation proceed in rounds. At every round, a behavioral policy samples data from the environment, and an expert labels that data. The goal is guess at the next stage (conditioned on the sampled states) what the expert will label the examples which are gathered. Here, unlike the supervised learning setting, we receive a stream of new data points which can be correlated and drawn from following different distributions through time. Theoretically this makes the optimization problem significantly more difficult, and because we must interact with a simulator, querying the stochastic gradient can be expensive. Like the example in \cref{app:experiments}, in this setting we will interact under both the experts policy distribution (behavioral cloning), as well as the policy distribution induced by the agent (online imitation). We parameterize a standard normal distribution whose mean is learned through a mean squared error loss between the expert labels and the mean of the agent policy. Again, the expert is trained following soft-actor-critic. All experiments were run using an NVIDIA GeForce RTX 2070 graphics card. with a AMD Ryzen 9 3900 12-Core Processor.

In this this setting we evaluate two measures: the per-round log-policy loss, and the policy return. The log policy loss is as described above, while the return is a measure of how well the imitation learning policy actually solves the task. In the Mujoco benchmarks, this reward is defined as a function of how quickly the agent can move in space, as well as the power exerted to move. Imitation learning generally functions by taking a policy which has a high reward (e.g. can move through space with very little effort in terms of torque), and directly imitating it instead of attempting to learn a cost to go function as is done in RL~\cite{sutton2000policy}. 

Each algorithm is evaluated over three random seeds following the same initialization scheme. As before, in all settings, algorithms use either their theoretical step-size when available, or the default as defined by \cite{paszke2017automatic}. The inner-optimization loop is set according to line-search parameters and heuristics following \citet{vaswani2019painless}. All algorithms and batch sizes are evaluated for 50 rounds of interaction (accept in the case of \cref{fig:oil_fig-app-linear-cr}, which is evaluated for 500) and performance is represented as a function of total interactions with the environment. Below we learn a policy which is parameterized by a two layer perception with 256 hidden units and relu activations (as was done in the main paper). For further details, please see the attached code repository.

%% file: main.bbl
\begin{thebibliography}{}

\bibitem[Abernethy et~al., 2009]{abernethy2009competing}
Abernethy, J.~D., Hazan, E., and Rakhlin, A. (2009).
\newblock Competing in the dark: An efficient algorithm for bandit linear
  optimization.
\newblock {\em COLT}.

\bibitem[Agarwal et~al., 2020]{agarwal2020stochastic}
Agarwal, N., Anil, R., Koren, T., Talwar, K., and Zhang, C. (2020).
\newblock Stochastic optimization with laggard data pipelines.
\newblock {\em Advances in Neural Information Processing Systems},
  33:10282--10293.

\bibitem[Amari, 1998]{Amari1998}
Amari, S. (1998).
\newblock Natural gradient works efficiently in learning.
\newblock {\em Neural Computation}.

\bibitem[Amid et~al., 2022]{pmlr-v151-amid22a}
Amid, E., Anil, R., and Warmuth, M. (2022).
\newblock Locoprop: Enhancing backprop via local loss optimization.
\newblock In Camps-Valls, G., Ruiz, F. J.~R., and Valera, I., editors, {\em
  Proceedings of The 25th International Conference on Artificial Intelligence
  and Statistics}, volume 151 of {\em Proceedings of Machine Learning
  Research}, pages 9626--9642. PMLR.

\bibitem[Armijo, 1966]{armijo1966minimization}
Armijo, L. (1966).
\newblock Minimization of functions having lipschitz continuous first partial
  derivatives.
\newblock {\em Pacific Journal of mathematics}, 16(1):1--3.

\bibitem[Belkin et~al., 2019]{belkin2019does}
Belkin, M., Rakhlin, A., and Tsybakov, A.~B. (2019).
\newblock Does data interpolation contradict statistical optimality?
\newblock In {\em AISTATS}.

\bibitem[Bollapragada et~al., 2019]{bollapragada2019exact}
Bollapragada, R., Byrd, R.~H., and Nocedal, J. (2019).
\newblock Exact and inexact subsampled newton methods for optimization.
\newblock {\em IMA Journal of Numerical Analysis}, 39(2):545--578.

\bibitem[Bottou et~al., 2018]{bottou2018optimization}
Bottou, L., Curtis, F.~E., and Nocedal, J. (2018).
\newblock Optimization methods for large-scale machine learning.
\newblock {\em Siam Review}, 60(2):223--311.

\bibitem[Bubeck et~al., 2015]{bubeck2015convex}
Bubeck, S. et~al. (2015).
\newblock Convex optimization: Algorithms and complexity.
\newblock {\em Foundations and Trends{\textregistered} in Machine Learning},
  8(3-4):231--357.

\bibitem[Chang and Lin, 2011]{chang2011libsvm}
Chang, C.-C. and Lin, C.-J. (2011).
\newblock Libsvm: A library for support vector machines.
\newblock {\em ACM transactions on intelligent systems and technology (TIST)},
  2(3):1--27.

\bibitem[Choi et~al., 2019]{choi2019faster}
Choi, D., Passos, A., Shallue, C.~J., and Dahl, G.~E. (2019).
\newblock Faster neural network training with data echoing.
\newblock {\em arXiv preprint arXiv:1907.05550}.

\bibitem[Dempster et~al., 1977]{dempster1977maximum}
Dempster, A.~P., Laird, N.~M., and Rubin, D.~B. (1977).
\newblock Maximum likelihood from incomplete data via the em algorithm.
\newblock {\em Journal of the Royal Statistical Society: Series B
  (Methodological)}, 39(1):1--22.

\bibitem[Dosovitskiy et~al., 2017]{Dosovitskiy17}
Dosovitskiy, A., Ros, G., Codevilla, F., Lopez, A., and Koltun, V. (2017).
\newblock {CARLA}: {An} open urban driving simulator.
\newblock In {\em Proceedings of the 1st Annual Conference on Robot Learning},
  pages 1--16.

\bibitem[Drusvyatskiy, 2017]{drusvyatskiy2017proximal}
Drusvyatskiy, D. (2017).
\newblock The proximal point method revisited.
\newblock {\em arXiv preprint arXiv:1712.06038}.

\bibitem[Duchi et~al., 2011]{duchi2011adaptive}
Duchi, J., Hazan, E., and Singer, Y. (2011).
\newblock Adaptive subgradient methods for online learning and stochastic
  optimization.
\newblock {\em JMLR}.

\bibitem[Florence et~al., 2022]{pmlr-v164-florence22a}
Florence, P., Lynch, C., Zeng, A., Ramirez, O.~A., Wahid, A., Downs, L., Wong,
  A., Lee, J., Mordatch, I., and Tompson, J. (2022).
\newblock Implicit behavioral cloning.
\newblock In Faust, A., Hsu, D., and Neumann, G., editors, {\em Proceedings of
  the 5th Conference on Robot Learning}, volume 164 of {\em Proceedings of
  Machine Learning Research}, pages 158--168. PMLR.

\bibitem[Gower et~al., 2019]{gower2019sgd}
Gower, R.~M., Loizou, N., Qian, X., Sailanbayev, A., Shulgin, E., and
  Richt{\'a}rik, P. (2019).
\newblock Sgd: General analysis and improved rates.
\newblock In {\em International Conference on Machine Learning}, pages
  5200--5209. PMLR.

\bibitem[Haarnoja et~al., 2018]{haarnoja2018soft}
Haarnoja, T., Zhou, A., Abbeel, P., and Levine, S. (2018).
\newblock Soft actor-critic: Off-policy maximum entropy deep reinforcement
  learning with a stochastic actor.
\newblock In {\em International conference on machine learning}, pages
  1861--1870. PMLR.

\bibitem[Hazan et~al., 2007]{hazan2007logarithmic}
Hazan, E., Agarwal, A., and Kale, S. (2007).
\newblock Logarithmic regret algorithms for online convex optimization.
\newblock {\em Machine Learning}, 69(2):169--192.

\bibitem[Johnson and Zhang, 2013]{johnson2013accelerating}
Johnson, R. and Zhang, T. (2013).
\newblock Accelerating stochastic gradient descent using predictive variance
  reduction.
\newblock In {\em Advances in Neural Information Processing Systems,
  {N}eur{IPS}}.

\bibitem[Johnson and Zhang, 2020]{johnson2020guided}
Johnson, R. and Zhang, T. (2020).
\newblock Guided learning of nonconvex models through successive functional
  gradient optimization.
\newblock In {\em International Conference on Machine Learning}, pages
  4921--4930. PMLR.

\bibitem[Kakade, 2001]{kakade2001natural}
Kakade, S.~M. (2001).
\newblock A natural policy gradient.
\newblock In {\em Advances in Neural Information Processing Systems 14 [Neural
  Information Processing Systems: Natural and Synthetic, {NIPS} 2001, December
  3-8, 2001, Vancouver, British Columbia, Canada]}.

\bibitem[Kingma and Ba, 2015]{kingma2014adam}
Kingma, D. and Ba, J. (2015).
\newblock Adam: {A} method for stochastic optimization.
\newblock In {\em {ICLR}}.

\bibitem[Lavington et~al., 2022]{lavington2022improved}
Lavington, J.~W., Vaswani, S., and Schmidt, M. (2022).
\newblock Improved policy optimization for online imitation learning.
\newblock {\em arXiv preprint arXiv:2208.00088}.

\bibitem[Li et~al., 2021]{li2021second}
Li, X., Zhuang, Z., and Orabona, F. (2021).
\newblock A second look at exponential and cosine step sizes: Simplicity,
  adaptivity, and performance.
\newblock In {\em International Conference on Machine Learning}, pages
  6553--6564. PMLR.

\bibitem[Liang and Rakhlin, 2018]{liang2018just}
Liang, T. and Rakhlin, A. (2018).
\newblock Just interpolate: Kernel" ridgeless" regression can generalize.
\newblock {\em arXiv preprint arXiv:1808.00387}.

\bibitem[Lohr, 2019]{lohr2019sampling}
Lohr, S.~L. (2019).
\newblock {\em {Sampling: Design and Analysis: Design and Analysis}}.
\newblock Chapman and Hall/CRC.

\bibitem[Loizou et~al., 2021]{loizou2021stochastic}
Loizou, N., Vaswani, S., Laradji, I.~H., and Lacoste-Julien, S. (2021).
\newblock Stochastic polyak step-size for sgd: An adaptive learning rate for
  fast convergence.
\newblock In {\em International Conference on Artificial Intelligence and
  Statistics}, pages 1306--1314. PMLR.

\bibitem[Ma et~al., 2018]{ma2018power}
Ma, S., Bassily, R., and Belkin, M. (2018).
\newblock The power of interpolation: Understanding the effectiveness of {SGD}
  in modern over-parametrized learning.
\newblock In {\em ICML}.

\bibitem[Mairal, 2013]{mairal2013stochastic}
Mairal, J. (2013).
\newblock Stochastic majorization-minimization algorithms for large-scale
  optimization.
\newblock {\em Advances in Neural Information Processing Systems}, 26.

\bibitem[Mairal, 2015]{mairal2015incremental}
Mairal, J. (2015).
\newblock Incremental majorization-minimization optimization with application
  to large-scale machine learning.
\newblock {\em SIAM Journal on Optimization}, 25(2):829--855.

\bibitem[Meng et~al., 2020]{meng2020fast}
Meng, S.~Y., Vaswani, S., Laradji, I.~H., Schmidt, M., and Lacoste-Julien, S.
  (2020).
\newblock Fast and furious convergence: Stochastic second order methods under
  interpolation.
\newblock In {\em International Conference on Artificial Intelligence and
  Statistics}, pages 1375--1386. PMLR.

\bibitem[Nesterov, 2003]{nesterov2003introductory}
Nesterov, Y. (2003).
\newblock {\em Introductory lectures on convex optimization: A basic course},
  volume~87.
\newblock Springer Science \& Business Media.

\bibitem[Newton et~al., 2021]{newton2021retrospective}
Newton, D., Bollapragada, R., Pasupathy, R., and Yip, N.~K. (2021).
\newblock Retrospective approximation for smooth stochastic optimization.
\newblock {\em arXiv preprint arXiv:2103.04392}.

\bibitem[Nguyen et~al., 2022]{nguyen2022finite}
Nguyen, L.~M., Tran, T.~H., and van Dijk, M. (2022).
\newblock Finite-sum optimization: A new perspective for convergence to a
  global solution.
\newblock {\em arXiv preprint arXiv:2202.03524}.

\bibitem[Nocedal and Wright, 1999]{nocedal1999numerical}
Nocedal, J. and Wright, S.~J. (1999).
\newblock {\em Numerical optimization}.
\newblock Springer.

\bibitem[Orabona, 2019]{orabona2019modern}
Orabona, F. (2019).
\newblock A modern introduction to online learning.
\newblock {\em arXiv preprint arXiv:1912.13213}.

\bibitem[Paszke et~al., 2019]{paszke2017automatic}
Paszke, A., Gross, S., Massa, F., Lerer, A., Bradbury, J., Chanan, G., Killeen,
  T., Lin, Z., Gimelshein, N., Antiga, L., Desmaison, A., Kopf, A., Yang, E.,
  DeVito, Z., Raison, M., Tejani, A., Chilamkurthy, S., Steiner, B., Fang, L.,
  Bai, J., and Chintala, S. (2019).
\newblock Pytorch: An imperative style, high-performance deep learning library.
\newblock In {\em Advances in Neural Information Processing Systems 32}, pages
  8024--8035. Curran Associates, Inc.

\bibitem[Robbins and Monro, 1951]{robbins1951stochastic}
Robbins, H. and Monro, S. (1951).
\newblock A stochastic approximation method.
\newblock {\em The annals of mathematical statistics}, pages 400--407.

\bibitem[Ross et~al., 2011]{ross2011reduction}
Ross, S., Gordon, G., and Bagnell, D. (2011).
\newblock A reduction of imitation learning and structured prediction to
  no-regret online learning.
\newblock {\em Journal of machine learning research}, pages 627--635.

\bibitem[Schmidt and Le~Roux, 2013]{schmidt2013fast}
Schmidt, M. and Le~Roux, N. (2013).
\newblock Fast convergence of stochastic gradient descent under a strong growth
  condition.
\newblock {\em arXiv preprint arXiv:1308.6370}.

\bibitem[Schmidt et~al., 2011]{schmidt2011convergence}
Schmidt, M., Roux, N., and Bach, F. (2011).
\newblock Convergence rates of inexact proximal-gradient methods for convex
  optimization.
\newblock {\em Advances in neural information processing systems}, 24.

\bibitem[Schulman et~al., 2015]{schulman2015trust}
Schulman, J., Levine, S., Abbeel, P., Jordan, M., and Moritz, P. (2015).
\newblock Trust region policy optimization.
\newblock In {\em International Conference on Machine Learning (ICML)}, pages
  1889--1897.

\bibitem[Schulman et~al., 2017]{schulman2017proximal}
Schulman, J., Wolski, F., Dhariwal, P., Radford, A., and Klimov, O. (2017).
\newblock Proximal policy optimization algorithms.
\newblock {\em CoRR}, abs/1707.06347.

\bibitem[Sutton et~al., 2000]{sutton2000policy}
Sutton, R.~S., McAllester, D.~A., Singh, S.~P., and Mansour, Y. (2000).
\newblock Policy gradient methods for reinforcement learning with function
  approximation.
\newblock In {\em Advances in Neural Information Processing Systems (NeurIPS)},
  pages 1057--1063.

\bibitem[Taylor et~al., 2016]{pmlr-v48-taylor16}
Taylor, G., Burmeister, R., Xu, Z., Singh, B., Patel, A., and Goldstein, T.
  (2016).
\newblock Training neural networks without gradients: A scalable admm approach.
\newblock In Balcan, M.~F. and Weinberger, K.~Q., editors, {\em Proceedings of
  The 33rd International Conference on Machine Learning}, volume~48 of {\em
  Proceedings of Machine Learning Research}, pages 2722--2731, New York, New
  York, USA. PMLR.

\bibitem[Todorov et~al., 2012]{todorov2012mujoco}
Todorov, E., Erez, T., and Tassa, Y. (2012).
\newblock Mujoco: A physics engine for model-based control.
\newblock In {\em 2012 IEEE/RSJ International Conference on Intelligent Robots
  and Systems}.

\bibitem[Vaswani et~al., 2019a]{vaswani2019fast}
Vaswani, S., Bach, F., and Schmidt, M. (2019a).
\newblock Fast and faster convergence of sgd for over-parameterized models and
  an accelerated perceptron.
\newblock In {\em The 22nd international conference on artificial intelligence
  and statistics}, pages 1195--1204. PMLR.

\bibitem[Vaswani et~al., 2021]{vaswani2021general}
Vaswani, S., Bachem, O., Totaro, S., M{\"u}ller, R., Garg, S., Geist, M.,
  Machado, M.~C., Castro, P.~S., and Roux, N.~L. (2021).
\newblock A general class of surrogate functions for stable and efficient
  reinforcement learning.
\newblock {\em arXiv preprint arXiv:2108.05828}.

\bibitem[Vaswani et~al., 2022]{vaswani2022towards}
Vaswani, S., Dubois-Taine, B., and Babanezhad, R. (2022).
\newblock Towards noise-adaptive, problem-adaptive (accelerated) stochastic
  gradient descent.
\newblock In {\em International Conference on Machine Learning}, pages
  22015--22059. PMLR.

\bibitem[Vaswani et~al., 2020]{vaswani2020adaptive}
Vaswani, S., Laradji, I., Kunstner, F., Meng, S.~Y., Schmidt, M., and
  Lacoste-Julien, S. (2020).
\newblock Adaptive gradient methods converge faster with over-parameterization
  (but you should do a line-search).
\newblock {\em arXiv preprint arXiv:2006.06835}.

\bibitem[Vaswani et~al., 2019b]{vaswani2019painless}
Vaswani, S., Mishkin, A., Laradji, I., Schmidt, M., Gidel, G., and
  Lacoste-Julien, S. (2019b).
\newblock Painless stochastic gradient: Interpolation, line-search, and
  convergence rates.
\newblock In {\em Advances in Neural Information Processing Systems}, pages
  3727--3740.

\bibitem[Ward et~al., 2020]{ward2020adagrad}
Ward, R., Wu, X., and Bottou, L. (2020).
\newblock Adagrad stepsizes: Sharp convergence over nonconvex landscapes.
\newblock {\em The Journal of Machine Learning Research}, 21(1):9047--9076.

\bibitem[Williams, 1992]{williams1992simple}
Williams, R.~J. (1992).
\newblock Simple statistical gradient-following algorithms for connectionist
  reinforcement learning.
\newblock {\em Machine learning}, 8(3-4):229--256.

\bibitem[Woodworth et~al., 2023]{woodworth2023two}
Woodworth, B., Mishchenko, K., and Bach, F. (2023).
\newblock Two losses are better than one: Faster optimization using a cheaper
  proxy.
\newblock {\em arXiv preprint arXiv:2302.03542}.

\bibitem[Zhang et~al., 2017]{zhang2016understanding}
Zhang, C., Bengio, S., Hardt, M., Recht, B., and Vinyals, O. (2017).
\newblock Understanding deep learning requires rethinking generalization.
\newblock In {\em ICLR}.

\end{thebibliography}
